\documentclass[twoside,11pt]{article}

\usepackage{blindtext}
\usepackage[preprint]{jmlr2e}

\usepackage{lastpage}
\jmlrheading{??}{2026}{1-\pageref{LastPage}}{??/??; Revised ??/??}{??/??}{??-??}{Anya Fries, Markus Reichstein, David Blei, and Jonas Peters}

\ShortHeadings{Worst-case low-rank approximations}{Fries, Reichstein, Blei, and Peters}
\firstpageno{1}

\usepackage{makecell, graphicx}
\usepackage{titletoc}
\usepackage{amsmath, bbm}
\usepackage{enumitem}
\usepackage{float}
\usepackage{multirow}
\usepackage{subfig}
\usepackage[dvipsnames]{xcolor}
\usepackage{xspace}

\newtheorem{setting}{Setting}
\newtheorem{assumption}{Assumption}

\newcommand{\mstar}{m^*}

\newcommand{\Vmaxrcs}{V^{\textrm{maxRCS}}}

\newcommand{\Vminpca}{V^{\textrm{minPCA}}}

\newcommand{\Vnormminpca}{V^{\textrm{norm-minPCA}}}

\newcommand{\Vpool}{V^{\textrm{pool}}}
\newcommand{\Vpoolk}{V^{\textrm{pool}}_k}
\newcommand{\Vsep}{V^{\textrm{sep}}}
\newcommand{\Vsepk}{V^{\textrm{sep}}_k}
\newcommand{\Vseqminpcak}{V^{\textrm{seq-minPCA}}_k}

\newcommand{\VpoolHAT}{\hat V^{\textrm{pool}}}
\newcommand{\VpoolkHAT}{\hat V^{\textrm{pool}}_k}
\newcommand{\VminpcakHAT}{\hat V^{\textrm{minPCA}}_k}
\newcommand{\VmaxrcsHAT}{\hat V^{\textrm{maxRCS}}}
\newcommand{\VmaxrcskHAT}{\hat V^{\textrm{maxRCS}}_k}

\newcommand{\wcPCA}{wcPCA\xspace}

\newcommand{\minPCA}{\ensuremath{\mathtt{minPCA}}\xspace}
\newcommand{\normminPCA}{\ensuremath{\mathtt{norm}\textrm{-}\mathtt{minPCA}}\xspace}

\newcommand{\maxRCS}{\ensuremath{\mathtt{maxRCS}}\xspace}
\newcommand{\normmaxRCS}{\ensuremath{\mathtt{norm}\textrm{-}\mathtt{maxRCS}}\xspace}
\newcommand{\maxMC}{\ensuremath{\mathtt{maxMC}}\xspace}
\newcommand{\poolPCA}{\ensuremath{\mathtt{poolPCA}}\xspace}
\newcommand{\sepPCA}{\ensuremath{\mathtt{sepPCA}}\xspace}
\newcommand{\avgcovPCA}{\ensuremath{\mathtt{avgcovPCA}}\xspace}
\newcommand{\poolMC}
{\ensuremath{\mathtt{poolMC}}\xspace}
\newcommand{\maxregret}{\ensuremath{\mathtt{maxRegret}}\xspace}
\newcommand{\normmaxregret}{\ensuremath{\mathtt{norm}\textrm{-}\mathtt{maxRegret}}\xspace}

\newcommand{\conv}{\xrightarrow{p}}

\newcommand{\E}{\mathcal{E}}
\newcommand{\Ex}{\mathbb{E}} 

\newcommand{\iid}{\overset{\text{iid}}{\sim}}
 
\newcommand{\Ok}{\mathcal{O}_{p\times k}}
\newcommand{\R}{\mathbb{R}}
\newcommand{\Rnp}{\mathbb{R}^{n\times p}}

\newcommand{\Rpk}{\mathbb{R}^{p\times k}}
 
\DeclareMathOperator{\Tr}{Tr}
\DeclareMathOperator*{\argmax}{arg\,max}
\DeclareMathOperator*{\argmin}{arg\,min}

\DeclareMathOperator{\diag}{diag}

\begin{document}
\title{Worst-case low-rank approximations}
\author{\name Anya Fries \email anya.fries@stat.math.ethz.ch \\
       \addr Seminar for Statistics\\
       ETH Z\"urich\\
       Z\"urich, Switzerland
       \AND
\name Markus Reichstein \email mreichstein@bgc-jena.mpg.de \\
\addr Max Planck Institute for Biogeochemistry \\
        Jena, Germany    
\AND
\name David Blei \email david.blei@columbia.edu \\
\addr Department of Computer Science and Department of Statistics \\
    Columbia University \\
    New York, USA
\AND
\name Jonas Peters \email  jonas.peters@stat.math.ethz.ch \\
\addr Seminar for Statistics\\
       ETH Z\"urich\\
       Z\"urich, Switzerland
}
\editor{My editor}
\maketitle
\begin{abstract}%
    Real-world data in health, economics, and environmental sciences are often collected across heterogeneous domains (such as  hospitals, regions, or time periods). 
    In such settings, 
    distributional shifts can make standard PCA unreliable, in that, for example, the leading principal components may explain substantially less variance in unseen domains than in the training domains.
    Existing approaches (such as FairPCA) have proposed to consider 
    worst-case (rather than average) performance across multiple domains.
    This work develops
    a unified framework,  
    called \wcPCA,   
    applies it to other objectives (resulting in the novel estimators 
    such as \normminPCA
    and \normmaxregret, which are better suited for applications with heterogeneous total variance)
    and analyzes their relationship.
    We prove that for all objectives, the estimators 
    are worst-case optimal not only over the observed source domains but also over all target domains whose covariance lies in the convex hull of the (possibly normalized) source covariances. We establish consistency and asymptotic worst-case guarantees of empirical estimators. 
    We extend our methodology to matrix completion, another problem that makes use of low-rank approximations, and prove approximate worst-case optimality for inductive matrix completion.
    Simulations and 
    two
    real-world applications on ecosystem-atmosphere fluxes 
    demonstrate marked improvements in worst-case performance, with only minor losses 
    in average performance.
\end{abstract}

\section{Introduction}\label{sec:intro}
Real-world data frequently originate from multiple heterogeneous domains, characterized by distinct statistical properties or underlying structures. For example, medical data collected from different hospitals or environmental data spanning different ecosystems and time periods often exhibit distributional shifts. Traditional dimensionality reduction methods, such as principal components analysis (PCA), generally 
implicitly 
assume homogeneity across domains. When this assumption is violated, pooled representations can fail to generalize, for example, explaining substantially less variance in unseen domains.

PCA 
summarizes data in a lower-dimensional subspace 
chosen to maximize
explained variance (or, equivalently, to minimize $\ell_2$-reconstruction error). Given a random (row) vector $\mathbf{x} \in\R^{1\times p}$ with zero mean and 
covariance matrix $\Sigma := \Ex[\mathbf{x}^\top \mathbf{x}] \in \R^{p \times p}$, 
we say that an orthonormal matrix $V^*_k \in \Rpk$ solves \emph{rank-$k$ population PCA} if
\begin{equation}\label{eqn:pca}
     V^*_k \in \argmax_{V \in \Ok}  \Tr(V^\top \Sigma V),
\end{equation}   
where $\Ok$ denotes the space of all $p\times k$-dimensional orthonormal matrices 
(in practice, one considers sample covariance matrices instead). In the multi-domain setting, 
where we observe data in different domains $e$ 
(corresponding to different covariance matrices $\Sigma_e$)
no single covariance characterizes 
the data 
and the question arises how to take the different domains into consideration. 
Intuitively,
pooling all samples 
to form a single sample covariance matrix
ignores domain-specific variability, while treating each domain independently discards common structure. 
This work investigates the implications of 
replacing 
$\Tr(V^\top \Sigma V)$ in~\eqref{eqn:pca} 
by a term aggregating over the source domains; for example,
$\Tr(V^\top \sum_e w_e \Sigma_e V)$
(corresponding to a pooling approach)
or 
$\min_e \Tr(V^\top \Sigma_e V)$
(corresponding to a worst-case approach).
These objectives generally yield different answers (see, for example, 
Figures~\ref{fig:appl_fn} and~\ref{fig:eg-minpca}).

Moreover, this work shows  that the different objectives yield different performances on test domains, too: for example, some, but not all, satisfy robustness guarantees that extend beyond the observed source domains (see Section~\ref{sec:robustness} for details).
To illustrate this 
point
empirically, we consider daily averages of FLUXNET data \citep{pastorello2020fluxnet2015}, a global network of eddy covariance towers that measure biosphere–atmosphere exchanges of CO$_2$ (e.g., net ecosystem exchange, gross primary production), water vapor, and energy. 
We treat the TransCom 3 regions (a standard partition of the globe into large-scale climate zones \citep{gurney2002towards, gurney2004transcom}) as domains, and select five of them as source domains. 
We solve PCA on the pooled source data (\poolPCA) and compare it with one of the methods 
proposed
in this paper (\normmaxregret), a worst-case low-rank approximation. 
We consider two PCs 
(see Figure~\ref{fig:scree_plots} for a justification).
The proportion of explained variance is then evaluated for each source and target regions. 
\begin{figure}[ht]
    \centering
    \includegraphics[scale=1]{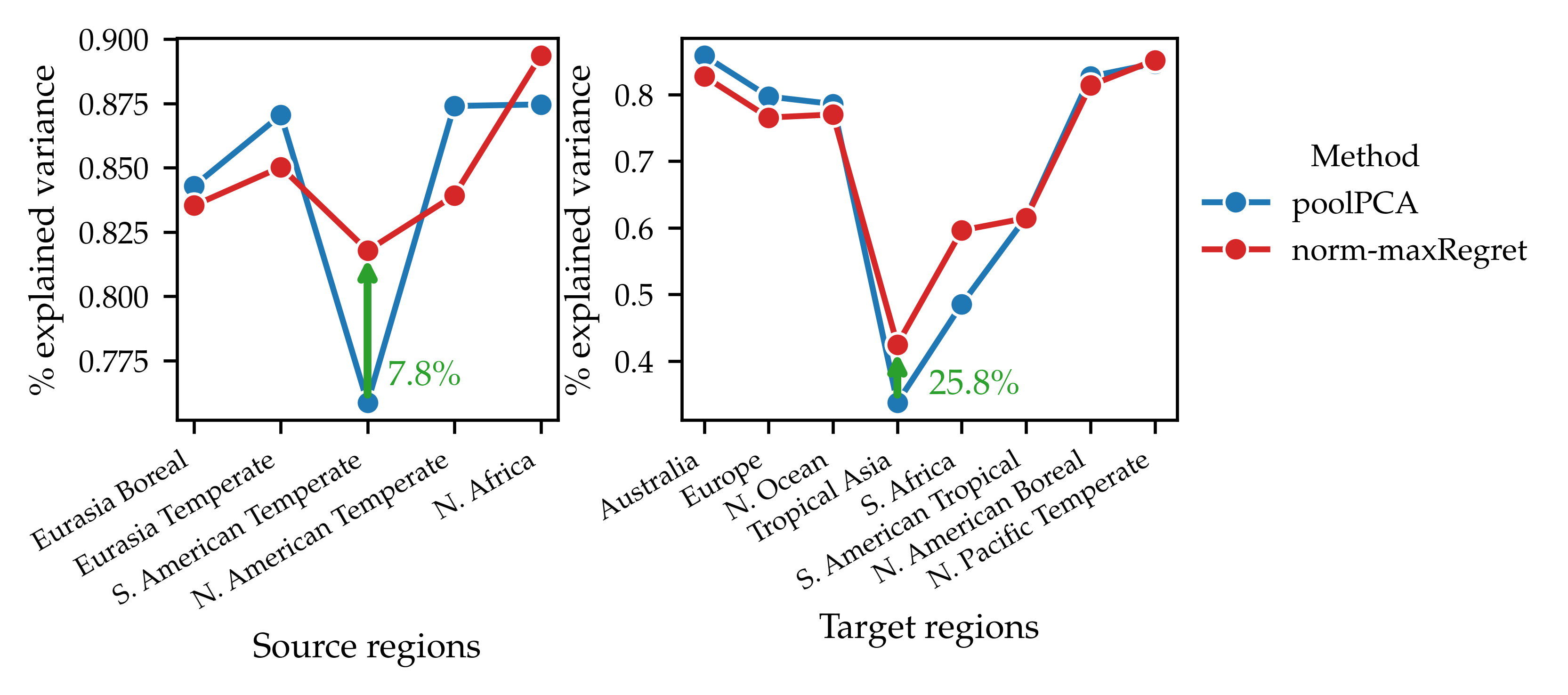}
    \caption{
        \textit{Proportion of explained variance by \poolPCA and \normmaxregret in the FLUXNET example on source (left) and target (right) regions 
                (for one specific split).}
                Unlike \poolPCA, which optimizes a pooled explained variance, \normmaxregret optimizes a worst-case criterion over the source domains, which, here, yields a worst-case improvement of 7.8\%. 
        Section~\ref{sec:robustness} proves that this choice comes with a worst-case improvement over a larger class of target domains, too. Indeed, in this example, the worst-case explained variance over the target domains improves by 25.8\%.  
                (Over 20 different splits, the median relative increase in worst-case improvements equal 25.6\% and 16.6\% for source and target domains, respectively; the median decrease in average performance for source domains equals 7.5\%; see Section~\ref{sec:appl:fluxnet} for more details.)
                                                                                 }
    \label{fig:appl_fn}
\end{figure}
Figure~\ref{fig:appl_fn} shows that
\normmaxregret improves upon \poolPCA in worst-case error 
(with only a moderate reduction in terms of average performance).
Section~\ref{sec:appl:fluxnet} 
further
shows 
that similar 
results hold for other splits into source and target domains, too; it also provides further details on the experimental setup.

We extend the worst-case framework to matrix completion, motivated by the following connection to PCA.
For a fixed $k < p$, PCA is an instance of low-rank approximation: 
in the finite-sample setting,
it
solves
$\min_{V\in\Ok} \|X - X VV^\top\|^2_F$,
where $X\in\Rnp$ is a data matrix (so that $X V$ and $V^\top$ act as left and right factors, respectively).
Low-rank approximations are also used to tackle the problem of matrix completion 
\citep{candes2010mc, candes2012exactmc}: here, a partially observed matrix is approximated by a low-rank factorization in order to predict the missing entries \citep[e.g.,][]{jain2013low}. 
However, traditional matrix completion approaches---much like PCA---implicitly assume that the latent structure is homogeneous across all samples.
We thus apply the worst-case principle to matrix completion and propose to learn a shared representation that minimizes worst-case reconstruction error (on the observed entries) across the source domains, and then use the shared factor for inductive matrix completion in a new, partially observed target domain. Section~\ref{sec:maxmc} proves that, when the source domains are fully observed,
we obtain analogous robustness guarantees to those of the worst-case PCA formulations, up to an approximation factor.

\subsection{Related work}\label{sec:related}
\paragraph{Domain generalization in prediction.}
Domain generalization in prediction has been studied extensively. One approach considers adversarial robustness, where performance is optimized under worst-case training scenarios.
Examples include maximin effects in regression \citep{meinshausen2015maximin, guo2024maxmin}, Magging \citep{buhlmann2016magging}, Distributionally Robust Optimization (DRO) \citep{ben2013robust, duchi2019variance,sinha2017certifying}, Group DRO \citep{sagawa2019distributionally, hu2018does}, and maximum risk minimization (maxRM) \citep{freni2025maximumriskminimizationrandom}.
These approaches focus on predictive risk, whereas our work studies worst-case robustness for unsupervised dimensionality reduction.

\paragraph{Fair PCA.}
Several extensions of PCA have been proposed to address heterogeneous domains, especially in the context of fairness. Fair principal component analysis (Fair PCA),
first proposed by \citet{samadi2018price},
aims to reduce disparities in reconstruction error across groups, which can also be interpreted as domains. 
\citet{samadi2018price} consider two domains, and minimize the maximum increase in reconstruction error incurred by using a shared projection instead of the domain-specific optimal projection.
For multiple domains, this objective is often referred to as the disparity error or regret (see Section~\ref{sec:regret}).
\citet{kamani2022efficient} study trade-offs between reconstruction accuracy and fairness constraints, while \citet{zalcberg2021fairpca} introduce fair robust PCA and fair sparse PCA. Within a broader multi-criteria dimensionality reduction framework, \citet{tantipongpipat2019multi} define Fair PCA as maximizing worst-case explained variance (discussed in Section~\ref{sec:explvar}). These objectives have motivated algorithmic developments, including those of \citet{babu2023fair}.
However, all of these approaches target fairness (i.e., in-sample guarantees), whereas our framework addresses out-of-sample guarantees.

\paragraph{Stable and common PCA.}
Beyond fairness, \cite{wang2025stablepca} seek robustness guarantees across domains and study a Group DRO formulation of PCA with the Fantope relaxation, termed StablePCA. 
While they focus on the optimization algorithm, 
we focus on
explicit out-of-sample guarantees 
(including 
stronger results for the population case
and both consistency and robustness guarantees for the estimators).
We also
consider a broader class of worst-case objectives, 
analyze their relationship,
and extend the framework to matrix completion.
Other multi-domain PCA frameworks, such as the one proposed by \citet{puchhammer-2024}, accommodate heterogeneous data but yield domain-specific projections rather than a unified robust representation. Similarly, common PCA \citep{flury1986asymptotic, flury1987two} 
and common empirical orthogonal functions \citep{hannachi2023common} allow for heterogeneous data, but neither provide worst-case nor out-of-sample guarantees.

\paragraph{Matrix completion} 
Matrix completion is well studied in the single-domain setting, where recovery guarantees are available under an incoherence assumption and assumptions on the sampling 
process
of the observed entries \citep[e.g.,][]{candes2010mc, candes2012exactmc, jain2013low}.
Inductive matrix 
enables
completion for previously unseen rows or columns, for example by incorporating side information \citep{jain2013provableinductivematrixcompletion} or by estimating latent factors from partially observed entries in a learned low-rank model \citep{mf_koren2009}.
Multi-domain extensions include fairness-aware non-negative matrix factorization \citep{kassab2026fairernonnegativematrixfactorization}
and fairness-aware federated matrix factorization \citep{fair_federated_mc_2022}.
These works, however, focus on group fairness rather than robustness across heterogeneous domains.
To our knowledge, 
the results we present in Section~\ref{sec:robustness-mc} are the first
explicit worst-case guarantees for (inductive) matrix completion.

\subsection{Contributions and outline}
We develop a unified framework for worst-case PCA and matrix completion.
While related objectives have appeared in Fair PCA and in relaxed formulations such as StablePCA, these works either do not address out-of-sample robustness or focus on the optimization problem.
Our framework clarifies the relationships between worst-case variance-based, reconstruction-based, and regret-based objectives, which—unlike in classical PCA—generally yield distinct solutions. 
We prove that the resulting worst-case solutions generalize beyond the observed domains, 
more specifically, that
they are worst-case optimal over all distributions whose covariances lie in the convex hull of the source covariances. The guarantees extend ideas from distributionally robust optimisation while requiring only second-moment information.
We further provide finite-sample theory, proving consistency and asymptotic worst-case optimality of the empirical estimators.
Simulations and a FLUXNET case study demonstrate  improvements in worst-case performance with only minor losses on average.

The remainder of this work is organized as follows.
Section~\ref{sec:motivation} motivates worst-case PCA (\wcPCA). 
Section~\ref{sec:variants} introduces and discusses
the variants of the 
objective
and links them to reconstruction error. 
Section~\ref{sec:robustness} establishes the main robustness guarantees, which we extend to the finite-sample setting in Section~\ref{sec:finite-sample}. 
Section~\ref{sec:maxmc} broadens the framework to matrix completion with missing data. 
Section~\ref{sec:simulations} illustrates the theory through simulations
and Section~\ref{sec:applications} presents two real-world applications on ecosystem–atmosphere flux data. 
All proofs can be found in Appendix~\ref{appendix:proofs}.

\subsection{Notation}
Throughout the article, we make use of the following notation. 
For all $n\in\mathbb{N}_+$, we write $[n] := \{1, \ldots, n\}$. 
Generally, 
non-bold
lowercase letters (e.g., $x$) are used for deterministic row vectors, bold lowercase letters (e.g., $\mathbf{x}$) are used for random row vectors, and uppercase letters (e.g., $X$) are used for random matrices. 
For all matrices $X \in \Rnp$  
(random or deterministic),
the Frobenius norm of $X$ is denoted by 
$\|X\|_F = \sqrt{\sum_{i=1}^n \sum_{j=1}^pX_{i j}^2} = \sqrt{\Tr\left(X^\top X\right)} = \sqrt{\sum_{i=1}^{\min \{n, p\}} \sigma_i^2(X)}$, where $\sigma_i(X)$ denotes singular values of $X$ and $\Tr(\cdot)$ the trace.
For all $x\in\R^p$, $\diag(x)\in\R^{p\times p}$ is the matrix formed with $x$ on the diagonal and 0s elsewhere.
The set of $p\times k$ matrices with 
orthonormal
columns is denoted $\Ok := \{V \in\Rpk \mid V^\top V = I_k \}$. 
The convex hull of a set $S$ contained in a real vector space, denoted by $\mathrm{conv}(S)$, is the set of all convex combinations of points in $S$, that is, 
$\mathrm{conv}(S) := 
\{ \sum_{i=1}^k \alpha_i x_i \ | \ x_i \in S, \ \alpha_i \ge 0, \ \sum_{i=1}^k \alpha_i = 1, \ k \in \mathbb{N} \}$.

\section{Worst-case PCA and its variants}
\subsection{Setting and motivation}\label{sec:motivation}
Given data from 
source domains, the goal is to learn low-dimensional representations that explain variance across these domains and on new (possibly unseen) target domains. 
We first consider a population setting, i.e., in each source domain, we are given the whole distribution; we  discuss the finite-sample case in Section~\ref{sec:finite-sample}.
\begin{setting}[Source domains]\label{setting:set-up} 
    Let $\E := \{1,\ldots,E\}$ denote a set of $E$ \emph{source domains}.
    For all $e\in\E$, let $\mathbf{x}_e\in \R^{1\times p}$ be a random row vector drawn
    from a distribution $P_e$ such that $\Ex_{P_e}[\mathbf{x}_e] = 0$ and
    the covariance $\Sigma_e :=  \Ex_{P_e}[\mathbf{x}_e^\top \mathbf{x}_e] \in\R^{p\times p}$ 
    exists and has strictly positive trace.
    For each $e \in \E$, let $w_e > 0$ denote the probability\footnote{    All methods and theoretical results developed in this paper apply both when
    domains are sampled randomly according to $(w_e)_{e\in\E}$ and when the domain sequence is deterministic.} that an observation comes from domain $e$, 
    where the domain assignment is independent of all other randomness and
    $\sum_{e\in\E} w_e = 1$.
    The domain label associated with each observation is assumed to be observed.
\end{setting}
A standard approach, which we refer to as \poolPCA, is to pool all domains. 
This is often done implicitly, for example, when the domain structure is ignored.
More concretely, we say $\Vpoolk \in \Ok$ solves \emph{rank-$k$ \poolPCA} if it maximizes the average variance, i.e., it solves PCA for $\Sigma_\mathrm{pool} := \sum_{e\in\E}w_e \Sigma_e$.
Throughout the paper, we omit the subscript $k$ whenever the rank is clear from context.
One can also solve PCA in each domain.
For each $e \in \E$, let 
$V^{*,e}_k \in \Ok$ 
denote a rank-$k$ PCA solution.
There is no canonical way to combine these into a single representation, but one can consider a conservative option and  select the solution from the domain that achieves the minimal explained variance.
We say $\Vsepk$ solves \emph{rank-$k$ \sepPCA} (separate PCA) if
\begin{equation*}
    \Vsepk := V^{*,e_0}_k
    \quad\textrm{with}\quad
    e_0 := \argmin_{e\in\E} 
        \Tr\big((V^{*,e}_k)^\top \Sigma_e V^{*,e}_k\big) \big\},
\end{equation*}
where
we choose $e_0$ to be the smallest index minimizing the objective.

Neither of these two baselines is necessarily optimal for maximizing explained variance across domains. 
When evaluated on
the observed source domains, the pooled and separate solutions can perform poorly in the worst case
and
Section~\ref{sec:minpca-sim} demonstrates that the same may hold when these solutions are applied in test domains. 
The following example illustrates this 
point
and motivates \wcPCA, which explicitly optimizes for worst-case performance across domains. 
\begin{example}\label{example1}
    Consider two domains $\E := \{1,2\}$,
    where observations are drawn from each domain with equal probability, that is, $w_1 = w_2$.
    The distribution in each
    domain is a zero-mean Gaussian with population covariances
    \begin{equation*}
        \Sigma_1 := \diag(0.9, 0.1, 0) \
        \mathrm{ and } \
        \Sigma_2 := \diag(0, 0.4, 0.6),
    \end{equation*}    
    respectively;
    the example is illustrated in Figure~\ref{fig:eg-minpca}.
    \begin{figure}[t]
        \centering        
        \includegraphics[trim={0 5.15cm 0 4.19cm},clip]{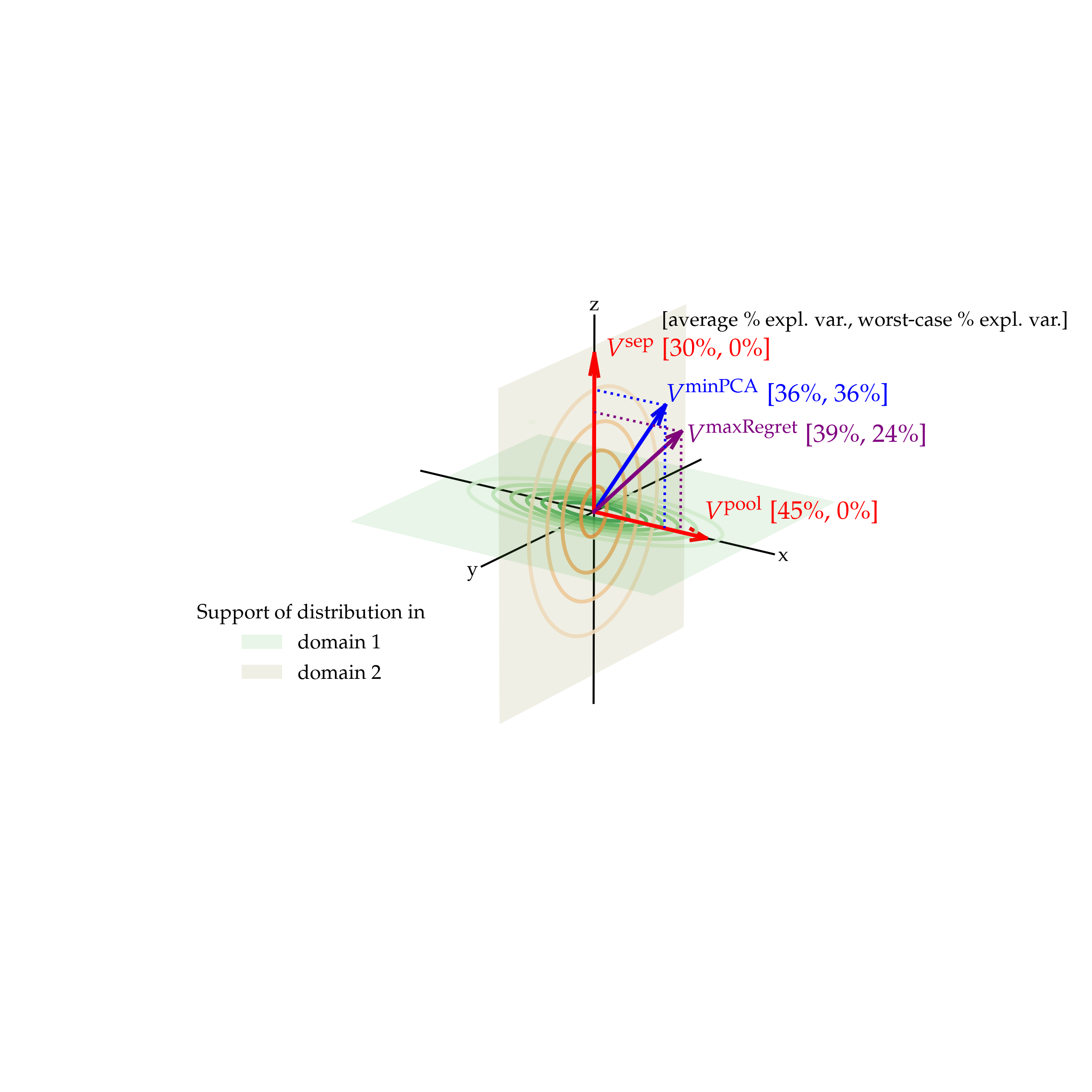}
        \caption{
        \textit{Visualization of possible solutions to the rank-1 approximation problem in a specific example.} 
        The figure 
        shows contour plots of two distributions $\mathcal{N}(0, \Sigma_1)$ (green) and $\mathcal{N}(0, \Sigma_2)$ (orange), together with their supports (shaded planes). The methods
            \poolPCA, \sepPCA, \minPCA, and \maxregret 
            describe different ways to aggregate over different source domains in the objective function (see Section~\ref{sec:variants} for details); their solutions are denoted by $\Vpool$, $\Vsep$, $\Vminpca$, and $V^\textrm{maxRegret}$, respectively (the values in parentheses indicate the pooled
            and worst-case explained variance). 
            For example,
            $\Vpool$ and $\Vsep$ explain $0\%$ of variance in the worst-case domain, as they are orthogonal to the support of that domain. In contrast, $\Vminpca$ maximizes the worst case explained variance (resulting in 36\%) and $V^\textrm{maxRegret}$ explains at least 24\% of variance in each domain. Details are provided in Example~\ref{example1}.
        }
        \label{fig:eg-minpca}
    \end{figure}
   
    Rank-1 \poolPCA considers $\Sigma_\mathrm{pool} = \frac{1}{2}(\Sigma_1 + \Sigma_2)$ and yields $\Vpool = (1,0,0)^\top$, while solving PCA separately in each 
    (source)
    domain gives $(1,0,0)^\top$ for domain 1 and $(0,0,1)^\top$ for domain 2; selecting the solution from the domain with minimal explained variance yields $\Vsep = (0,0,1)^\top$. 
    $\Vpool$ and $\Vsep$, respectively, explain 45\% and 30\% 
    of the average covariance $\Sigma_\mathrm{pool}$. However,  both perform poorly (0\% explained variance)
    in the respective worst-case 
    source domain. 
    This occurs because $\Vpool$ and $\Vsep$ are both orthogonal to the support of one of the domains, thus capturing no variance in that domain.

    In contrast, the \wcPCA variants explicitly optimize worst-case performance across domains. In particular, \minPCA (Definition~\ref{def:minPCA}, below)
    is solved by
    $\Vminpca = \frac{1}{\sqrt{5}} (\sqrt{2}, 0, \sqrt{3})^\top$,
    which explains 36\% of variance in both domains and \maxregret (Definition~\ref{def:maxRegret}, below)
    is solved by $\Vminpca = \frac{1}{\sqrt{5}} (\sqrt{3}, 0, \sqrt{2})^\top$,
    which explains at least 24\% of variance. 
    Thus, while the explained variance with respect to the pooled covariance $\Sigma_{\mathrm{pool}}$ is comparable to that of $\Vsep$ and $\Vpool$, they achieve strictly positive percentage of explained variance in the worst-case source domain.

    We will see in Section~\ref{sec:robustness} that this improvement is not restricted to the source domains: we prove that several variants of \wcPCA (including \minPCA and \maxregret) are worst-case optimal over all domains whose covariances lie in the convex hull of the source covariances.
\end{example}

\subsection{Variants of \wcPCA}\label{sec:variants}
We now formally introduce some variants of \wcPCA 
More concretely, the 
objective 
may focus on 
maximizing variance or minimizing reconstruction error 
and the source covariances may or may not be normalized prior to optimization. 
This leads to four variants: \minPCA, \normminPCA, \maxRCS, and \normmaxRCS. 
In addition, we introduce 
the
regret-based formulations
\maxregret and \normmaxregret, which evaluate performance relative to each domain’s own optimal subspace. 
The objectives \minPCA and \maxregret have been proposed previously; only \minPCA has been considered as a 
``distributionally robust'' PCA by \citet{wang2025stablepca} 
-- but no 
explicit
theoretical guarantees have been provided.
Table~\ref{tab:worstcase-pca} provides an overview of these variants.

For classical PCA, the solution spaces of many of these variants coincide\footnote{E.g., for a covariance matrix $\Sigma$, the sets of solutions to rank-$k$ PCA on $\Sigma$ and $\Sigma / \Tr(\Sigma)$ are the same,
as 
$\argmax_{V\in\Ok} \left\{  \Tr(V^\top \Sigma V) \right\} = \argmax_{V\in\Ok} \left\{  {\Tr(V^\top \Sigma V) }/{\Tr(\Sigma)} \right\}$.}  but we see in Section~\ref{sec:comparing-solutions} that this is generally not the case in the multi-domain setting (Section~\ref{sec:discussion-of-objectives} discusses intuition about when to use which objective).
Section~\ref{sec:robustness} proves that all of these variants satisfy worst-case optimality guarantees.
\begin{table}[t]
    \centering
    \renewcommand{\arraystretch}{1.1}
    \setlength{\tabcolsep}{4pt}
    \begin{tabular}{ | l r | c c c c c c| }
    \hline
    & & \multicolumn{6}{c|}{\textbf{Objective function} \ \ }  \\
    \multirow{5}{*}{\rotatebox{90}{\textbf{Normalized \ }}}
    & & {Explained variance} & & {Reconstruction error} & & {Regret} & {} \\
    & & {(Def.~\ref{def:minPCA})} & & {(Def.~\ref{def:maxRCS})} & & {(Def.~\ref{def:maxRegret})} & {} \\
    \cline{2-8} 
            & No 
            & \minPCA  
            & $\ne$     
            & \maxRCS
            & $\ne$ 
            & \maxregret
            & $\ne$ \\
            &       
            & $\ne$         
            &           
            & $\ne$ 
            &           
            & $\ne$
            &  \\
            & Yes
            & \normminPCA
            & $=$       
            & \normmaxRCS
            & $\ne$ 
            & \normmaxregret
            & $\ne$ \\
    \hline 
\end{tabular} 
\caption{
    The
    six worst-case PCA formulations 
    considered in this work: three objectives (explained variance, reconstruction error, and regret) combined with two normalization choices (normalized vs.\ unnormalized). 
    The symbols 
    $=$ and $\ne$
    indicate whether the solutions coincide (see Theorem~\ref{thm:comparing-solutions}).
        The equality holds if the methods are applied to the full (and usually unknown) distribution or to its empirical counterpart.
}
\label{tab:worstcase-pca}
\end{table}

\subsubsection{Explained variance} \label{sec:explvar}
Let the explained variance and normalized 
(or proportion of)
explained variance of a subspace $V$ for a distribution $P$ with covariance $\Sigma$ be
\begin{equation}\label{eqn:losses-variance}
    \mathcal{L}_\mathrm{var}(V;P) := \Tr(V^\top \Sigma V), \quad \mathcal{L}_\mathrm{normVar}(V;P) := \frac{\Tr(V^\top \Sigma V)}{\Tr(\Sigma)}.
\end{equation}
\begin{definition}[minPCA]\label{def:minPCA}
    Consider Setting~\ref{setting:set-up}.
    An orthonormal matrix $V^* \in \Ok$ solves \emph{rank-$k$ \minPCA} (resp.\ \emph{\normminPCA}) 
    if
    \begin{equation*}
        V^* \in \argmax_{V \in \Ok}
            \left\{
                \min_{e \in \E} \mathcal{L}(V;P_e)
            \right\},
    \end{equation*}
    where $\mathcal{L} = \mathcal{L}_\mathrm{var}$ (resp.\ $\mathcal{L} = \mathcal{L}_\mathrm{normVar}$).
\end{definition}
With respect to the standard metric induced by the Frobenius norm, 
$\Ok$ is compact as a closed and bounded subset of $\R^{p\times k}$ (by Heine-Borel), 
and the objective function is continuous. 
Hence, the maximum is attained.

\subsubsection{Reconstruction error} \label{sec:maxrcs}
PCA does not only maximize explained variance, it also minimizes the expected reconstruction error among all rank-$k$ approximations. Specifically, for a random vector $\mathbf{x}\in\R^{1\times p}$ drawn from distribution $P$ with covariance matrix $\Sigma:=\Ex_P[\mathbf{x}^\top \mathbf{x}]$,
it holds that
    $\argmax_{V \in \Ok} \Tr(V^\top \Sigma V)
    =
    \argmin_{V \in \Ok} \Ex_{\mathbf{x}\sim P}\left[\|\mathbf{x}-\mathbf{x}VV^\top\|_2^2\right]$.
We 
now define a worst-case objective corresponding to the right-hand side of this equation. To do so, we use the following notation for
the expected unnormalized and normalized reconstruction losses: 
\begin{equation}\label{eqn:losses-rcs}
    \mathcal{L}_\mathrm{RCS}(V;P) :=  
        \Ex_{\mathbf{x}\sim P} \left[\| \mathbf{x} - \mathbf{x}VV^\top \|^2_2 \right], \quad 
     \mathcal{L}_\mathrm{normRCS}(V;P) :=  
        \frac{\Ex_{\mathbf{x}\sim P} \left[\| \mathbf{x} - \mathbf{x}VV^\top \|^2_2 \right]}{\Ex_{\mathbf{x}\sim P} \left[\| \mathbf{x}\|^2_2\right]}.
\end{equation}
\begin{definition}[maxRCS]\label{def:maxRCS}
    Consider Setting~\ref{setting:set-up}.
    An orthonormal matrix $V^* \in \Ok $
    solves rank-$k$ \emph{\maxRCS} (resp.\ \emph{\normmaxRCS}) if
    \begin{equation*}
        V^* \in \argmin_{V \in \Ok}
        \left\{
            \max_{e \in \E} \mathcal{L}(V;P_e)
        \right\},
    \end{equation*}
    where $\mathcal{L}=\mathcal{L}_\mathrm{RCS}$ (resp.\ $\mathcal{L}=\mathcal{L}_\mathrm{normRCS}$).
\end{definition}

\subsubsection{Regret}\label{sec:regret}
The 
two
formulations introduced 
above
optimize `absolute performance' in the worst-case domain. 
Alternatively, we can assess performance with respect to each domain's own optimal subspace. 
This leads to a regret-based formulation, which quantifies the increase in reconstruction error incurred by enforcing a common subspace across domains. Let the expected unnormalized and normalized regret be
\begin{align*}
    \mathcal{L}_\mathrm{reg}(V;P) &:=  
        \mathcal{L}_\mathrm{RCS}(V;P_e)   - \min_{W\in\Ok}\mathcal{L}_\mathrm{RCS}(W;P_e),
    \\
    \mathcal{L}_\mathrm{normReg}(V;P) &:=  
        \mathcal{L}_\mathrm{normRCS}(V;P_e)   - \min_{W\in\Ok}\mathcal{L}_\mathrm{normRCS}(W;P_e).
\end{align*}
We can now define a regret-based\footnote{Theorem~\ref{thm:comparing-solutions} below proves that Definition~\ref{def:maxRegret} can equivalently be formulated using explained variance rather than reconstruction error.} variant of \wcPCA.
\begin{definition}[maxRegret]\label{def:maxRegret}
    Consider Setting \ref{setting:set-up}.
    An orthonormal matrix $V^* \in \Ok$ solves rank-$k$ \emph{\maxregret} (resp.\ \emph{\normmaxregret}) if
            \begin{equation*}
                 V^* \in \argmin_{V \in \Ok} 
                    \left\{ \max_{e \in \E} \mathcal{L}(V;P_e)\right\},
            \end{equation*}
            where $\mathcal{L}=\mathcal{L}_\mathrm{reg}$ (resp.\ $\mathcal{L}=\mathcal{L}_\mathrm{normReg}$). 
\end{definition}

\subsubsection{Iterative optimization}\label{sec:iterative}
The PCA solution in~\eqref{eqn:pca} can be computed iteratively by sequentially selecting orthonormal directions that successively maximize explained variance in the remaining space.
For PCA, this sequential procedure is equivalent to solving the joint optimization problem~\eqref{eqn:pca} and therefore yields an optimal solution 
\citep[Property A1, p.11]{Jolliffe2002}. 
Sequential analogues can also be defined for \minPCA and \normminPCA by choosing orthonormal components iteratively to optimize worst-case explained variance. 
In contrast to classical PCA, however, this greedy construction generally fails to recover the optimal joint solution.
\begin{proposition}[Sequential vs.\ joint optimization, informal]\label{prop:seq-ne-joint-informal}
In general, the subspace obtained by sequentially selecting directions to optimize the worst-case explained variance (resp.\ worst-case proportion of explained variance) does not coincide with the solution of the joint \minPCA (resp.\ \normminPCA) problem. 
\end{proposition}
A formal statement of the proposition and a counterexample are given in Appendix~\ref{app:seq-joint}.

When an ordered basis is required, we first compute the optimal subspace and then, 
starting from the full subspace, we iteratively remove the direction whose removal maximizes the worst-case explained variance of the remaining subspace.
This construction yields an ordered basis consistent with the joint solution (Appendix~\ref{app:ordering-basis} provides further details).

\subsubsection{Relations between the \wcPCA variants}\label{sec:comparing-solutions}
For classical PCA the solution set does not depend on whether one optimizes 
for
explained variance, reconstruction error, or their normalized versions. 
For the worst-case setting, 
Theorem~\ref{thm:comparing-solutions} below proves that
normalizing the objective can alter the solution set 
(statement~(i)), 
equivalence between worst-case explained-variance and reconstruction-error formulations holds only for the normalized objectives (statement~(ii)), and
for regret-based objectives the solution set does not depend on whether regret is formulated using explained variance or reconstruction error (statement~(iii)).
\begin{theorem}[Relations between \wcPCA variants]\label{thm:comparing-solutions}
    Consider Setting~\ref{setting:set-up}.
    Then, the symbols $=$ and $\ne$ 
    in Table~\ref{tab:worstcase-pca} 
    represent
    whether the solutions coincide. More precisely, the following statements hold.
    \begin{enumerate}[label=\roman*)]
        \item \textbf{Effect of normalization.}
        Let $\Vminpca$ and $\Vnormminpca$ solve rank-$k$ \minPCA and \normminPCA, respectively. 
        In general, $\Vminpca$ does not solve \normminPCA and $\Vnormminpca$ does not solve \minPCA.
        Analogous statements hold for \maxRCS versus \normmaxRCS and for \maxregret versus \normmaxregret.
        \label{thm:comparing-solutions-normalization}

        \item \textbf{Explained variance vs.\ reconstruction error.}
        Let $\Vminpca$ and $\Vmaxrcs$ solve rank-$k$ \minPCA and \maxRCS, respectively. 
        In general, $\Vminpca$ does not solve \maxRCS and $\Vmaxrcs$ does not solve \minPCA.      
        In contrast, the set of solutions to rank-$k$ \normminPCA coincides with the set of solutions to rank-$k$ \normmaxRCS.
        \label{thm:comparing-solutions-var-vs-rcs}

        \item \textbf{Regret.}
        The set of solutions to the regret problem formulated using reconstruction error (as in Definition~\ref{def:maxRegret}) coincides 
        with that of the variance-based formulation, that is, with the solutions obtained by replacing $\mathcal{L}_\mathrm{RCS}$ with $-\mathcal{L}_\mathrm{var}$ (or $\mathcal{L}_\mathrm{normRCS}$ with $-\mathcal{L}_\mathrm{normVar}$).
        \label{thm:comparing-solutions-regret}
    \end{enumerate}
\end{theorem}

\subsubsection{Practical considerations for choosing the objective}\label{sec:discussion-of-objectives}
As summarized in Table~\ref{tab:worstcase-pca}, the various objectives lead to different sets of solutions. We 
now
discuss some of the differences
between the objectives and how factors such as scale heterogeneity and noise influence their behavior.
We also discuss how worst-case objectives can be used diagnostically to assess domain heterogeneity.

\paragraph{Sensitivity to varying variances.}
The unnormalized objectives optimize worst-case explained variance or reconstruction error in the original units. 
As a consequence, domains with particularly small (resp.\ large) total variance $\Tr(\Sigma_e)$ can dominate the objective of \minPCA (resp.\ \maxRCS). 
In particular, \minPCA maximizes worst-case explained variance in
absolute terms, so its solution may be 
entirely
determined by a domain whose variance in every direction is smaller than that of 
any other
domain.
In contrast, the normalized variants \normminPCA and \normmaxRCS optimize proportional explained variance or reconstruction error, 
which 
may result in a different
worst-case domain and lead to a different solution. 
These normalized objectives are therefore less sensitive to differences in total variance and align with the interpretation of how much of each domain is explained, 
but unit comparability across domains is lost. 
Similarly, 
\maxregret 
is 
not as
sensitive to particularly large or small $\Tr(\Sigma_e)$ 
either
as it compares performance relative to domain-optimal subspaces.

\paragraph{Choosing between explained variance and reconstruction error.}
For the unnormalized objectives, \minPCA and \maxRCS 
can lead to different solutions.
In contrast, the normalized objectives (\normminPCA and \normmaxRCS), as well as the regret, produce identical solution sets under either formulation, eliminating the need for this decision and aligning more closely with standard PCA.

\paragraph{Effect of heterogeneous noise.}
When noise levels differ across domains, the regret provides a robust way to learn a common subspace even if the ultimate goal is to optimize explained variance or reconstruction error. Indeed, 
let us
consider observations 
with different noise variances,
that is, for all $e\in\E$ we observe
\begin{equation*}
    \mathbf{z}_e := \mathbf{x}_e + \sigma_e \boldsymbol{\varepsilon}_e,
    \qquad \boldsymbol{\varepsilon}_e \sim N(0,I_p),
\end{equation*}
where $\mathbf{x}_e$ is as in Setting~\ref{setting:set-up} 
and for all $e, e'\in\E$, 
$\sigma_e>0$ and
$\boldsymbol{\varepsilon}_e$ is independent 
of
$\boldsymbol{\varepsilon}_{e'}$ and $\mathbf{x}_e$. 
Then, $\mathbf{z}_e$ has covariance $\Sigma_e+\sigma_e^2 I_p$. 
By slight abuse of notation, we say the problem is  \emph{noiseless} if $\sigma_e = 0$, $\forall e \in \mathcal{E}$; otherwise, we say that there is \emph{heterogeneous noise}.
Let us first consider the proposed methods applied to the population covariance matrices.
For any $V\in\Ok$ and $e\in\E$, the explained variance is 
$\Tr(V^\top\Sigma_e V) + k\sigma_e^2$. 
Thus, domains with particularly high noise do not
attain the minimum in \minPCA
but can influence the \maxRCS solution.
In contrast, the regret compares a common subspace $V$ to the
rank-$k$ domain-optimal subspace $V_k^{e,*}$, so the noise terms cancel and the objective reduces to $\Tr((V_k^{e,*})^\top\Sigma_e V_k^{e,*}) - \Tr(V^\top\Sigma_e V)$. Consequently, heterogeneous noise affects the solutions of \minPCA and \maxRCS but not of \maxregret.
Moreover, if each $\Sigma_e$ is exactly rank-$k$, then
$\Tr((V_k^{e,*})^\top\Sigma_e V_k^{e,*})=\Tr(\Sigma_e)$ and minimizing the maximum regret is
equivalent to minimizing the maximum reconstruction error in the noiseless
setting 
(see proof of Theorem~\ref{thm:comparing-solutions}).
For 
finitely many
data, these statements 
may not
hold 
precisely
but 
in our simulations,
the same qualitative
patterns persist 
(see Section~\ref{sec:sims-noise}).

\paragraph{Worst-case losses as diagnostics tools.}
While Section~\ref{sec:robustness} proves that 
the different \wcPCA objectives come with 
out-of-distribution guarantees, the multi-domain losses can also be used diagnostically. If domain-wise losses under \poolPCA vary substantially, this flags heterogeneity across domains that may warrant further investigation. If a worst-case solution differs substantially from \poolPCA, it provides a principled alternative representation that explicitly trades-off a small reduction in average performance for improved worst-case behavior.

\section{Robustness guarantees}
\subsection{Population guarantees}\label{sec:robustness}
Optimizing any of the worst-case objectives 
described in Section~\ref{sec:variants}
the following
strong robustness guarantee: the solution
remains worst-case
optimal not only for the source domains but also for all distributions whose covariances lie in the convex hull of the 
(possibly normalized)  source covariances. In contrast, 
standard approaches such as pooled or separate PCA do not offer such guarantees. 
The following two theorems summarize the results for all worst-case objectives from Section~\ref{sec:variants}.

\begin{theorem}[Robustness of \wcPCA]\label{thm:maxrcs-convex-hull}  
    Consider Setting~\ref{setting:set-up}. Let
    \begin{equation}\label{eqn:mathcalP}
        \mathcal{C} := \mathrm{conv}(\{\Sigma_e\}_{e \in \E})
        \quad \text{and} \quad
        \mathcal{P} := \left\{ P \in \mathcal{P}(\R^p) \mid \Ex[\mathbf{x}] = 0,\ \Ex[\mathbf{x}^\top \mathbf{x}] \in \mathcal{C} \right\}
    \end{equation}
    be the convex hull of the source covariances and the corresponding uncertainty set of distributions with zero mean and covariance in $\mathcal{C}$, 
    respectively. Let $V^*_k \in \Ok$ solve rank-$k$ \maxRCS,
    \minPCA, or \maxregret.
    Let $\mathcal{L}$ denote the corresponding loss function, i.e., $\mathcal{L}=\mathcal{L}_\mathrm{RCS}$ for \maxRCS, $\mathcal{L}=-\mathcal{L}_\mathrm{var}$ for \minPCA, and $\mathcal{L}=\mathcal{L}_\mathrm{reg}$ for \maxregret.
    Then, the following statements hold.
    \begin{enumerate}[label=\roman*)]
        \item 
        For all $V_k\in\Ok,$ the worst-case loss over $\mathcal{P}$ equals the worst-case loss over the source domains, i.e., 
            \begin{equation*}
                \sup_{P \in \mathcal{P}}
                    \mathcal{L}(V_k;P)
                = \max_{e \in \E} \mathcal{L}(V_k;P_e).
            \end{equation*}
        \label{bullet:wc-optimality-equality-sup-max}
        \item Consequently, the ordering induced by the worst-case loss over the source domains is preserved over $\mathcal{P}$: for all $V_k, W_k \in \Ok$, if
            $
            \max_{e \in \E} \mathcal{L}(V_k;P_e)
            <
            \max_{e \in \E} \mathcal{L}(W_k;P_e),
            $
            then
            $
            \sup_{P \in \mathcal{P}} \mathcal{L}(V_k;P)
            <
            \sup_{P \in \mathcal{P}} \mathcal{L}(W_k;P).
            $
        \label{bullet:wc-optimality-ordering}
        \item Hence, $V^*_k$ is worst-case optimal over $\mathcal{P}$, i.e.,
            $V^*_k \in \displaystyle 
                \argmin_{V \in \Ok} 
                    \sup_{P \in \mathcal{P}} 
                        \mathcal{L}(V;P).$
        \label{bullet:wc-optimality-optimality}
        \item Statement~\ref{bullet:wc-optimality-optimality} does not generally hold for the solutions to \poolPCA and \sepPCA.
        \label{bullet:wc-optimality-poolpca}
    \end{enumerate}
\end{theorem}
With a slight modification of the proof, one can show that
Theorem~\ref{thm:maxrcs-convex-hull} extends to the normalized objectives \normmaxRCS, \normminPCA, and \normmaxregret after modifying the uncertainty set.
\begin{theorem}[Robustness of \wcPCA, normalized]\label{thm:normmaxrcs}
Let $\mathcal{C}_{\mathrm{norm}} := \mathrm{conv}\left( \left\{ \Sigma_e / \Tr(\Sigma_e) \right\}_{e \in \E} \right)$ and
    \begin{equation}\label{eqn:mathcalP-norm}
        \mathcal{P}_{\mathrm{norm}} := 
            \left\{ P \in \mathcal{P}(\R^p) \middle| \Ex_{\mathbf{x}\sim P}[\mathbf{x}] = 0, \frac{\Ex_{\mathbf{x}\sim P}[\mathbf{x}^\top\mathbf{x}]}{\Ex_{\mathbf{x}\sim P}[\|\mathbf{x}\|_2^2]} \in \mathcal{C}_{\mathrm{norm}} \right\}.
    \end{equation}
    Then the same statements 
    as in Theorem~\ref{thm:comparing-solutions}
    hold when replacing 
    $\mathcal{L}$ with 
    $\mathcal{L}_\mathrm{normRCS}$ (for reconstruction), 
    $-\mathcal{L}_\mathrm{normVar}$ (for variance), 
    or $\mathcal{L}_\mathrm{normReg}$ (for regret), 
    and replacing $\mathcal{P}$ with $\mathcal{P}_{\mathrm{norm}}$.
\end{theorem}
Although similar in spirit, the results do not follow from Group DRO \citep{sagawa2019distributionally}, as we outline in Appendix~\ref{app:groupdro}.

\subsection{Finite-sample guarantees and consistency}\label{sec:finite-sample}
So far, we have focused on population-level objectives. 
We now prove analogous results in a finite-sample setting. Moreover, we prove
that the empirical estimators are consistent for the population solutions and asymptotically worst-case optimal.
\begin{setting}[Source domains, finite sample]\label{setting:set-up-finite-sample}
    Consider Setting~\ref{setting:set-up}. 
    For all $e\in\E$, let $X_e\in\R^{n_e\times p}$ be a data matrix of $n_e$ i.i.d.\ observations 
        from $P_e$. 
    Define the empirical covariance as $\hat\Sigma_e := \frac{1}{n_e}X_e^\top X_e$. 
    Let  $n:=\sum_{e\in\E}n_e$, $X_{\mathrm{pool}} := [X_1^\top,\ldots,X_E^\top]^\top \in\R^{n \times p}$
    be the pooled data and $\hat\Sigma_{\mathrm{pool}} := \frac{1}{n} X_{\mathrm{pool}}^\top X_{\mathrm{pool}}$ its covariance. We assume that for all $e\in\E$, $\Tr(\hat\Sigma_e) > 0$.
\end{setting}
Given Setting~\ref{setting:set-up-finite-sample}, we define the empirical analogues of \poolPCA and the worst-case PCA objectives. We say:
\begin{gather*}
    \VpoolkHAT \textrm{ solves \emph{rank-$k$ empirical \poolPCA} if }\VpoolkHAT \in \argmax_{V\in\Ok} \Tr(V^\top \hat {\Sigma}_\mathrm{pool} V), \\
    \VminpcakHAT \textrm{ solves \emph{rank-$k$ empirical \minPCA} if } \VminpcakHAT\in \argmax_{V\in\Ok} \min_{e\in\E} \Tr(V^\top \hat \Sigma_e V), \\
    \VmaxrcskHAT \textrm{ solves \emph{rank-$k$ empirical \maxRCS} if } \VmaxrcskHAT\in \argmin_{V\in\Ok} \max_{e\in\E} \frac{1}{n_e} \|X_e - X_e VV^\top\|^2_F.
\end{gather*}
Analogous definitions of empirical \maxregret, \normminPCA, \normmaxRCS, and \normmaxregret are given in Appendix~\ref{app:finite-sample}.
\begin{remark}\label{rem:fsrobust}
    The empirical formulations exhibit optimality guarantees similar to their population counterparts 
    when considering the convex hull of empirical covariance matrices rather than population covariances. More concretely, we provide
    a representative result in Appendix~\ref{app:finite-sample},  Proposition~\ref{prop:maxrcs-convex-hull-emp}). 
\end{remark}
We now prove that the worst-case estimators are asymptotically consistent for the population subspace (Proposition~\ref{prop:consistency}) and the worst-case guarantees of Theorem~\ref{thm:maxrcs-convex-hull} hold asymptotically (Proposition~\ref{prop:consistency-of-guarantees}). To establish these results,
we require a uniqueness condition. 
\begin{assumption}\label{ass:uniqueness}
    Each population problem \maxRCS, \normmaxRCS, \minPCA, \normminPCA, \maxregret, and \normmaxregret admits a unique solution up to right multiplication by an orthonormal 
    matrix.
    That is, if $V_k^*$ and $\tilde V_k$ are both solutions, then there exists $Q\in \R^{k\times k}$ with $Q^\top Q = QQ^\top = I_k$ such that 
    $V_k^*=\tilde V_k Q$, or, equivalently, $\mathrm{span}(V_k^*)=\mathrm{span}(\tilde V_k)$.
\end{assumption}
Assumption~\ref{ass:uniqueness}
avoids
set-valued limits (up to rotation).
It holds under several conditions. 
For example,
if a single domain strictly determines the worst-case loss at the optimum, the problem reduces to standard PCA on that domain, and 
Assumption~\ref{ass:uniqueness}
follows from a strict eigengap 
between the $k$-th and $(k+1)$-th eigenvalue
in the corresponding covariance matrix (normalized where applicable).
In addition, if multiple domains are 
`active' at the optimum, Assumption~\ref{ass:uniqueness}
holds 
if
their (normalized) covariances share a common top-$k$ eigenspace and each has a strict eigengap
between its $k$-th and $(k+1)$-th eigenvalue.
In our simulations, Assumption~\ref{ass:uniqueness} seems to hold.
For example, 
repeating the experiment of Section~\ref{sec:sims-avg-vs-wc} and solving the worst-case low-rank problem twice over 100 runs, the median projection distance (see Proposition~\ref{prop:consistency}) between the two solutions in each run is approximately $0.04$, corresponding (if concentrated in one principal angle) to about $1.66^\circ$.
Finally, we hypothesize that Propositions~\ref{prop:consistency} and~\ref{prop:consistency-of-guarantees}  hold more generally -- even in (possibly non-generic) settings, where Assumption~\ref{ass:uniqueness} is violated.
\begin{proposition}[Consistency of the estimators]\label{prop:consistency}    
    Consider Setting~\ref{setting:set-up-finite-sample}. 
    Let $V_k^*$ and $\hat V_k$ denote the population and empirical solutions, respectively, to 
    \maxRCS, \normmaxRCS, \minPCA, \normminPCA, \maxregret, or \normmaxregret, and let $n_{\min}:=\min_{e\in\E} n_e$. 
    Under Assumption~\ref{ass:uniqueness}, as $n_{\min}\to\infty$, we have
    \begin{equation*}
            d(\hat V_k, V^*_k) \conv 0
        \end{equation*}
    where $d(V_1,V_2):=\|V_1V_1^\top-V_2V_2^\top\|_F$ 
    and $\conv$ denotes convergence in probability.
\end{proposition}
We further show that the empirical solutions are asymptotically worst-case optimal, in the sense of Theorem~\ref{thm:maxrcs-convex-hull}. 
First, the empirical estimator achieves the optimal worst-case population loss in the limit. 
Second, the empirical objective provides an asymptotic upper bound on the population loss over all distributions in the uncertainty set.
\begin{proposition}[Asymptotic worst-case optimality]\label{prop:consistency-of-guarantees}    
    Consider Setting~\ref{setting:set-up-finite-sample}. 
    Let $\hat V_k$ and $V^*_k$ be the empirical and population solutions, respectively, to 
    \maxRCS, \normmaxRCS, \minPCA, \normminPCA, \maxregret, or \normmaxregret. 
    Let $\mathcal{L}$ denote the corresponding loss\footnote{ 
        So $\mathcal{L} \in \{-\mathcal{L}_{\mathrm{var}},\ -\mathcal{L}_{\mathrm{normVar}},\ \mathcal{L}_{\mathrm{RCS}},\ \mathcal{L}_{\mathrm{normRCS}},\ \mathcal{L}_{\mathrm{reg}},\ \mathcal{L}_{\mathrm{normReg}}\}$, as in Theorems~\ref{thm:maxrcs-convex-hull} and~\ref{thm:normmaxrcs}.
    }
    and let $\mathcal{Q}$ denote the associated distributional class ($\mathcal{P}$ for the unnormalized losses, $\mathcal{P}_{\mathrm{norm}}$ for the normalized losses). 
    Under Assumption~\ref{ass:uniqueness}, as $n_{\min}:=\min_{e\in\E} n_e \to\infty$, 
    the following two statements hold.
    \begin{enumerate}[label=\roman*), ref=\roman*)]
        \item The worst-case population loss of $\hat V_k$ converges in probability to the optimal value,
        $$
        \sup_{P\in\mathcal{Q}} \mathcal{L}(\hat V_k;P)
        \xrightarrow{p}
        \sup_{P\in\mathcal{Q}} \mathcal{L}(V^*_k;P). 
        $$
        \label{prop:consistency-of-guarantees-i}
        \item For every $\epsilon>0$,
        $$
        \lim_{n_{\min}\to\infty} 
        \Pr\left(
        \sup_{P\in\mathcal{Q}} \mathcal{L}(\hat V_k;P) > \hat m_k + \epsilon
        \right) = 0,
        $$
        where $\hat m_k := \max_{e\in\E} \mathcal{L}(\hat V_k;\hat\Sigma_e)$ is the empirical loss\footnote{\label{fn:loss-sigma}
            In Equation~\eqref{eqn:losses-variance} and~\eqref{eqn:losses-rcs} we define the losses $\mathcal{L}$ as a function of the distribution $P$. Each loss depends on the distribution $P$ only through its covariance matrix $\Sigma$, so, by slight abuse of notation, we can write the losses in terms of $\Sigma$. For example, 
            for all $V \in \Ok$, 
            $\mathcal{L}_\mathrm{RCS}(V;P) = \mathcal{L}_\mathrm{RCS}(V;\Sigma) := \Tr(\Sigma) - \Tr(V^\top \Sigma V)$.
            }.
        \label{prop:consistency-of-guarantees-ii}
    \end{enumerate}
\end{proposition}
    Here, ii) does not follow directly from finite-sample robustness, see Remark~\ref{rem:fsrobust}, because the distribution class $\mathcal{Q}$ is formed from the population not the empirical covariance matrices.

\section{Extension: worst-case matrix completion}\label{sec:maxmc}
When data 
are
only partially observed, low rank approximation extends naturally to matrix completion.
In the single-domain setting, a common formulation \citep[e.g.,][]{jain2013low} approximates $X \in \R^{n\times p}$ by a rank-$k$ factorization $X \approx LR^\top$, where $L \in \R^{n\times k}$ and $R \in \R^{p\times k}$ (with $k \ll p$), by minimizing reconstruction error over the observed entries.
A classic example is movie recommendation 
\citep[e.g.,][]{harper2015movielens}, where rows correspond to users and columns to movies.
Only a subset of ratings is observed, and the goal is to predict the missing entries under a low-rank assumption. 
The factorization can be interpreted as modeling latent features (e.g., genre or style) that explain preferences: 
the left factor represents user-specific weights on these features, while the right factor encodes the weights of the movies.

As before,
the data may arise from multiple domains.
In the movie example, this occurs when users belong to heterogeneous subpopulations (e.g., regions or age groups) with differing rating patterns.
We 
now
extend the worst-case framework developed for PCA to 
matrix completion.
In particular, we introduce a worst-case objective for multi-domain matrix completion and formalize inductive matrix completion in target domains. We then show 
the following worst-case optimality guarantee:
when the source domains are fully observed but the target domains incur missingness, subspaces that are worst-case optimal for reconstruction error remain $\epsilon$-worst-case optimal over the convex hull of the source covariances.
\begin{setting}[Multiple domains with missingness]\label{setting:mc}
    Consider Setting~\ref{setting:set-up}.
    For each domain $e \in \E$, 
    the
    entries of $\mathbf{x}_e \in \R^p$ are observed according to a mask $\omega_e \in \{0,1\}^p$, 
    where $\omega_e$ is
    drawn independently of $\mathbf{x}_e$ from a distribution $Q$:
    $(\omega_e)_i =1$ if the $i$-th entry of $\mathbf{x}_e$ is observed and $(\omega_e)_i =0$ otherwise
\end{setting}
\begin{setting}[Finite-sample version]\label{setting:mc-empirical}
    Consider Setting~\ref{setting:mc}. For each $e \in \E$, let $X_e \in \R^{n_e\times p}$ be a data matrix of $n_e$ i.i.d.\ rows from $P_e$, and let $\Omega_e \in \{0,1\}^{n_e \times p}$ contain $n_e$ 
    masks 
    drawn i.i.d.\
    from $Q$; 
    only entries with $(\Omega_e)_{ij}=1$ are observed.
    Let $n :=\sum_{e\in\E} n_e$ and $X_{\mathrm{pool}} \in \Rnp$ and $\Omega_{\mathrm{pool}} \in \{0,1\}^{n\times p}$ 
    denote the row-wise concatenation of the domain-specific matrices.
\end{setting}
In the multi-domain setting, a natural baseline is to ignore domain labels and solve matrix completion on
the pooled data $(X_{\mathrm{pool}}, \Omega_{\mathrm{pool}})$, minimizing average reconstruction error across domains;  
we term this solving \emph{\poolMC}.
Similarly as for \wcPCA, we can also optimize the worst-case reconstruction error.
\begin{definition}[maxMC]
    Consider Setting~\ref{setting:mc-empirical}. A 
    (shared)
    right factor $R^* \in \Rpk$ and 
    (domain-specific)
    left factors $L_e^*\in \R^{n_e\times k}$ for $e\in\E$ solve \emph{\maxMC} if
    \begin{equation}\label{eqn:maxMC}
        (R^*, L_1^*,\ldots,L_E^*)
        \in
        \argmin_{\substack{R \in \Ok\\ \forall e\in\E,\ L_e \in \R^{n_e\times k}}}
        \max_{e \in \E} 
        \frac{1}{n_e}
        \|P_{\Omega_e}(X_e) - P_{\Omega_e}(L_e R^\top)\|_F^2,
    \end{equation}
    where, for all $e\in\E$, $P_{\Omega_e}: \R^{n_e \times p} \to \R^{n_e \times p}$ is 
    defined as $[P_{\Omega_e}(X_e)]_{ij} = [X_e]_{ij}$ if $[\Omega_e]_{ij} = 1$ and 
    $[P_{\Omega_e}(X_e)]_{ij} = 0$
    otherwise.
\end{definition}
This objective learns a shared right factor that minimizes worst-case reconstruction error across domains.
In the single-domain setting, this coincides with the standard matrix completion objective. 
Additionally, in absence of missingness in the source domains, the shared right factor that solves \maxMC coincides with the solution to empirical \maxRCS, 
since $P_{\Omega_e}$ reduces to the identity and, for all $e\in\E$, the optimal choice of $L_e$ is $L_e = X_e R$.

\subsection{Inductive matrix completion}\label{sec:inductive-mc}
After learning a shared right factor $R^*$ on the source domains, we 
can also
use it to reconstruct
new, partially observed 
observations
in a target domain 
(in the movie recommendation example, this could correspond to new users in a new domain) -- a problem that is sometimes referred to as 
inductive matrix completion.
We 
can, for example,
reconstruct a row $\mathbf{x}\in\R^{1\times p}$ observed on coordinates $\omega$, by estimating coefficients $\ell\in\R^k$ from the observed entries and forming $\hat{\mathbf{x}}=\ell (R^{*})^\top$.
In this paper, 
we focus on the 
ordinary least squares estimator,
\begin{equation}\label{eqn:ell-ols}
   \ell_\mathrm{OLS} = \ell_\mathrm{OLS}(\mathbf{x}, \omega, R^*) 
    \in
    \argmin_{\ell \in \R^k}
    \sum_{\substack{i \in [p]:\, \omega_i=1}}
    (\mathbf{x}_i - [\ell (R^{*})^\top]_i)^2,
\end{equation}
an approach that is commonly used \citep[e.g.,][]{mf_koren2009};
alternative methods 
such as 
incorporating side information \citep[e.g.,][]{jain2013provableinductivematrixcompletion} 
could also be considered.
\subsection{Robustness guarantees}\label{sec:robustness-mc} 
We consider the setting in which the source domains are fully observed and missingness arises only in the target domain. In this setting,
we prove that
a subspace that is worst-case optimal for reconstruction remains $\epsilon$-worst-case optimal for inductive matrix completion over the convex hull of the source covariances.
We 
impose incoherence of the right factor---a standard assumption in matrix completion---and require a minimum number of observed entries;
later, this allows us to control the reconstruction error under missingness.
\begin{definition}[$\mu$-incoherence]
Let $R \in \Ok$. We say that $R$ is \emph{$\mu$-incoherent} if, for all $i \in [p]$, $\|R^{(i)}\|_2 \le \mu \sqrt{\tfrac{k}{p}}$,
where $R^{(i)}$ denotes the $i$-th row of $R$.
\end{definition}
\begin{assumption}[Sufficient observations]\label{ass:sufficient-obs}
    Let $\epsilon \in (0,1)$ and $\mu \ge 1$. 
    For $\omega \sim Q$, the number of unobserved entries $s$ satisfies
    $s \le \frac{p \epsilon}{k \mu^2 (2\epsilon + 1)}$
    almost surely.
\end{assumption}
As in Theorem~\ref{thm:maxrcs-convex-hull}, let $\mathcal{C}$ be the convex hull of the domain covariance matrices and let $\mathcal{P}$ denote the set of distributions with covariance in $\mathcal C$. 
Additionally, for all distributions $P, Q$ over $\mathbf{x}$ and $\omega$, for all $R\in\Ok$, let 
$\mathcal{L}_{P,Q}(R) := \Ex_{\mathbf{x}\sim P, \omega \sim Q}\| \mathbf{x} - \ell_\mathrm{OLS}(\mathbf{x}, \omega, R) R^\top \|^2_2$ be the expected reconstruction error of $\mathbf{x}\sim P$ under sampling distribution $Q$. 
\begin{theorem}[Worst-case robustness for inductive completion]\label{thm:mc}
    Consider Setting~\ref{setting:mc}
    and let $\epsilon\in(0,1)$
    Assume the source domains are fully observed, i.e., the source mask distribution $Q_{\mathrm{source}}$ satisfies $Q_{\mathrm{source}}(\omega=\mathbf{1})=1$. Let
    $R^* \in \Ok$ 
    solve the population worst-case matrix completion problem, defined as
    \begin{equation}\label{eq:pop-maxmc}
    R^* \in \argmin_{R\in\Ok}
    \max_{e\in\E}
    \mathcal{L}_{P_e,Q_\mathrm{source}}(R).
    \end{equation}
    ($R^*$ thus also solves \maxRCS.) 
    Let $Q_\mathrm{target}$ be the distribution of the masks in the target domains.
    If
    $R^*$ is $\mu$-incoherent,
    and 
    Assumption~\ref{ass:sufficient-obs} holds for $Q_\mathrm{target}$
    with this $\mu$,
    then $R^*$ is $\epsilon$-worst-case optimal for inductive matrix completion over $\mathcal{P}$, that is,
    \begin{equation*}
        \sup_{P \in \mathcal{P}}
        \mathcal{L}_{P,Q_\mathrm{target}}(R^*)
        \le
        (1+\epsilon)\,
        \min_{R \in \Ok}
        \sup_{P \in \mathcal{P}}
        \mathcal{L}_{P,Q_\mathrm{target}}(R).
        \end{equation*}
\end{theorem}
Thus, a subspace that is worst-case optimal for low-rank approximation remains (approximately) worst-case optimal for inductive matrix completion over all distributions in $\mathcal{P}$. 
Theorem~\ref{thm:mc} analyzes the regime in which the source domains are fully observed (so that \maxMC reduces to \maxRCS). 
Our simulations indicate that, even 
when the source domains are only partially observed,
\maxMC improves worst-case reconstruction error relative to \poolMC (see Section~\ref{sec:sims-maxmc}).
Furthermore, 
the bound on the number of unobserved entries imposed by Assumption~\ref{ass:sufficient-obs} is 
restrictive
(e.g., for $(k,\epsilon,\mu)=(2,0.1,1)$, 
the bound yields $s \le 0.0417\,p$, i.e., approximately $4\%$ missing entries).
Nevertheless, empirically, robustness degrades only mildly as missingness increases, and \maxMC continues to outperform pooling well beyond the guaranteed regime (see Section~\ref{sec:sims-maxmc}). 
We thus believe that stronger versions of Theorem~\ref{thm:mc} may hold.

\section{Numerical experiments}\label{sec:simulations}
We investigate the worst-case estimators through synthetic simulations. 
In each experiment, we consider five source domains, obtained by sampling five population covariances matrices 
$(\Sigma_1, \ldots, \Sigma_e)$
in $\mathbb{R}^{p\times p}$
($p=20$ unless otherwise stated) 
of rank $10$ 
from a measure $\mathbb{Q}_{\alpha, \beta}$ that 
depends on parameters $\alpha, \beta$ 
(the role and standard choices of the parameters are explained in Appendix~\ref{app:datagen}); in each experiment we resample the covariance matrices. 
Under the measures $\mathbb{Q}_{\alpha,\beta}$, 
the sampled covariances consist of
a rank-$5$ shared component and a rank-$5$ domain-specific component
with shared eigenvalues across domains, and are trace-normalized.
These properties ensure that the solutions to \minPCA, \maxRCS, \maxregret, and their normalized variants all coincide.
We therefore report only the results of \maxRCS and evaluate performance using reconstruction error. 
The only exception is Section~\ref{sec:sims-noise}, where we modify the data-generating process to study the effect of heterogeneous noise.
Appendix~\ref{app:datagen} provides more details on the data generating process.
The \wcPCA variants are solved using projected gradient descent (Appendix~\ref{app:algo} provides further details including a comparison of our implementation to the approximations of \citet{wang2025stablepca} and \citet{tantipongpipat2019multi}) and \poolMC and \maxMC are solved with alternating minimization (also see Appendix~\ref{app:algo}).
Code to reproduce all simulations and figures is publicly available at \url{https://github.com/anyafries/wcPCA}.

\subsection{Worst-case low-rank approximation}\label{sec:minpca-sim}
\subsubsection{Convex-hull robustness in the population case}\label{sec:sims-illustrate}
We illustrate Theorem~\ref{thm:maxrcs-convex-hull}(i)
stating
that, for \maxRCS, the maximum reconstruction error
on the source domains provides a minimal
upper bound for the reconstruction error of all covariances in their convex hull.
To do so,
we sample five source covariances 
$\mathcal{C}_\mathrm{source} := (\Sigma_1, \dots, \Sigma_5) \sim \mathbb{Q}_{\alpha, \beta}$
and compute the rank-5 population solutions to \poolPCA and \maxRCS, $\Vpool$ and $\Vmaxrcs$. 
We sample 50 target covariances, 
collected in the set
$\mathcal{C}_\mathrm{target}$,
from the convex hull of the source covariances
(as described in Appendix~\ref{app:datagen}).
For every $\Sigma \in \mathcal{C}_\mathrm{target} \cup \mathcal{C}_\mathrm{source}$, 
we compute the reconstruction errors using \poolPCA and \maxRCS, that is,
$\mathcal{L}_\mathrm{RCS}(\Vpool;\Sigma)$ 
and 
$\mathcal{L}_\mathrm{RCS}(\Vmaxrcs;\Sigma)$.
Figure~\ref{fig:sim_maxrcs_bound} shows that all test errors of \maxRCS lie below the bound $\mstar_\mathrm{maxRCS} := \max_{1\le e\le5} \mathcal{L}_{\mathrm{RCS}}(\Vmaxrcs; \Sigma_e)$, as guaranteed by Theorem~\ref{thm:maxrcs-convex-hull}(i). In contrast, \poolPCA exceeds this bound in several cases. 
\begin{figure}[t]
  \centering
  \begin{minipage}[t]{0.48\linewidth}
    \centering
    \includegraphics[scale=1]{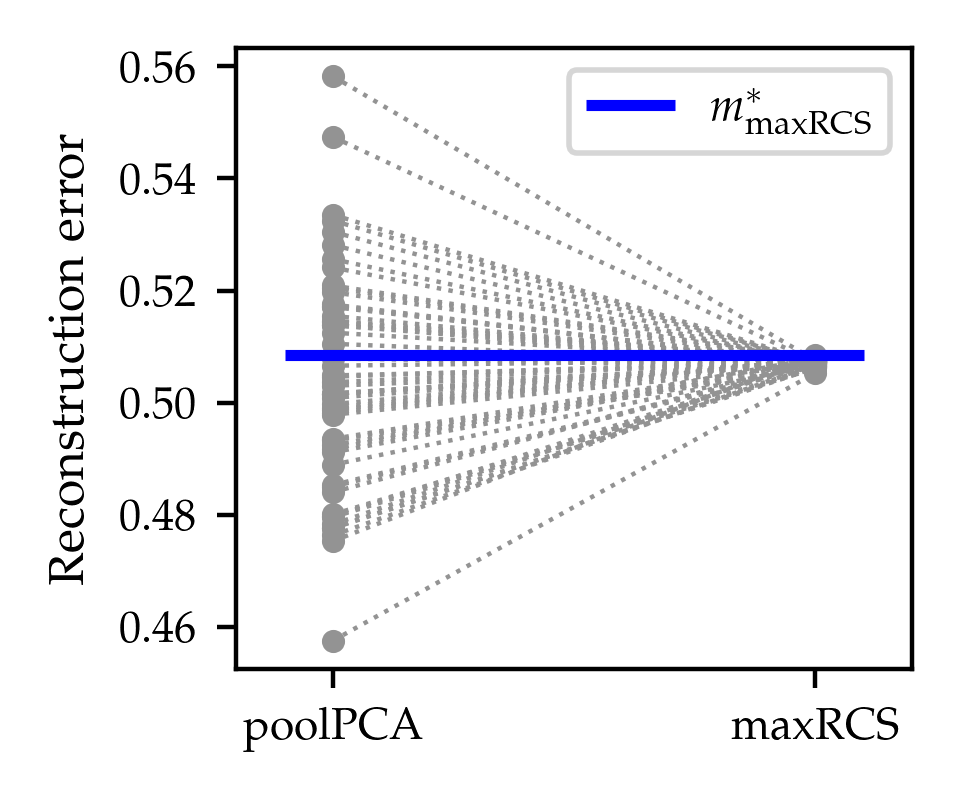}
    \caption{
        \textit{Illustration of Theorem~\ref{thm:maxrcs-convex-hull}(i); 
        see Section~\ref{sec:sims-illustrate}}.
        Reconstruction errors for the source domains and 50 target domains are shown for \poolPCA 
        and \maxRCS 
        in a population setting.
        The blue line marks $\mstar_\mathrm{maxRCS}$,
        the maximum reconstruction error over the source domains. 
        As expected from Theorem~\ref{thm:maxrcs-convex-hull}(i), all \maxRCS errors lie below this bound, whereas \poolPCA exceeds it in several cases.
        } 
    \label{fig:sim_maxrcs_bound}
  \end{minipage}\hfill
  \begin{minipage}[t]{0.48\linewidth}
    \centering
    \includegraphics{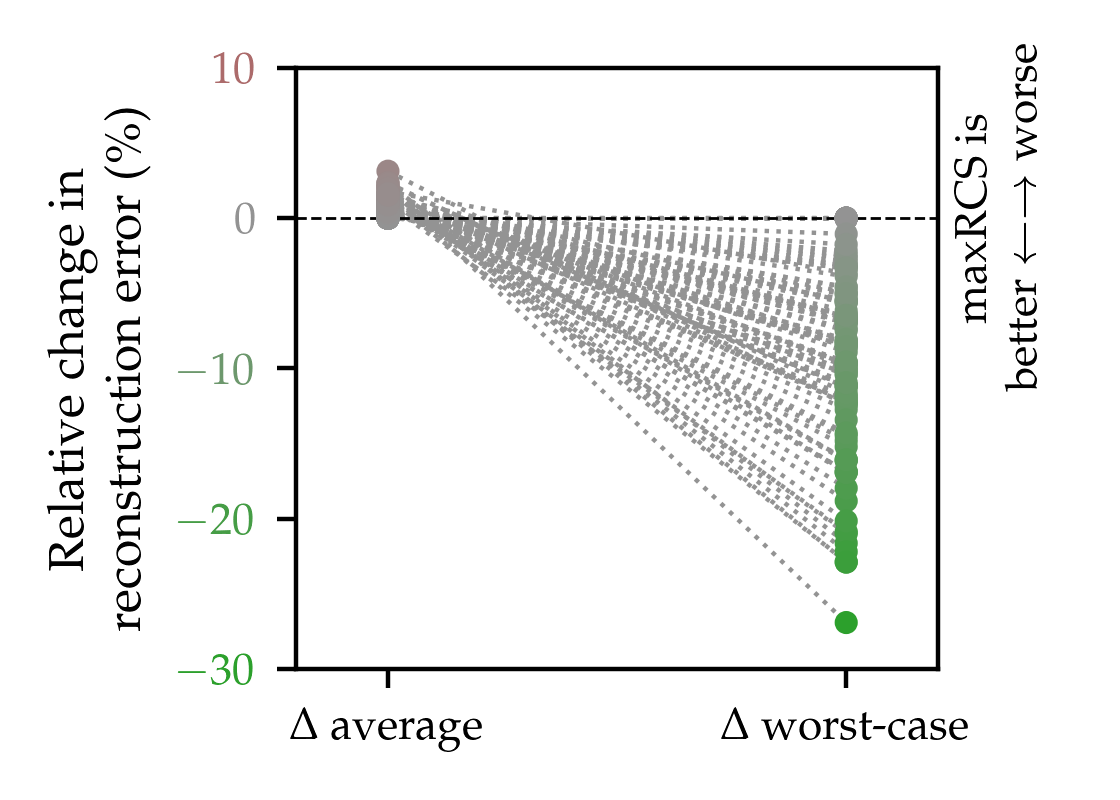}
    \caption{
        \textit{Average vs.\ worst-case reconstruction error of \maxRCS and \poolPCA; 
        see Section~\ref{sec:sims-avg-vs-wc}}.
        The difference in errors are shown relative to \poolPCA's average performance. 
        Values below zero indicate that \maxRCS outperforms \poolPCA. 
        \maxRCS 
        improves worst-case performance, while incurring only small losses on average.
        }
    \label{fig:sim_wc_vs_avg}
  \end{minipage}
\end{figure}

\subsubsection{Average versus worst-case performance}\label{sec:sims-avg-vs-wc}
We next study how much average performance \maxRCS loses in exchange for improved worst-case robustness compared to \poolPCA.
We vary the heterogeneity between domains 
by varying the parameters $\alpha, \beta$ in the distribution of the source covariances $\mathbb{Q}_{\alpha, \beta}$ (see Appendix~\ref{app:datagen} for more details);
we consider $(\alpha,\beta) \in \{(0.1, 0.5),\allowbreak(0.5, 1),\allowbreak(1, 2),\allowbreak(2, 5)\}$.
Higher values of $(\alpha,\beta)$ increase the contribution of domain-specific components, creating greater variation between domains. 
For each $(\alpha,\beta)$, we repeat the following 
steps
25 times. We sample $(\Sigma_1, \dots, \Sigma_5) \sim \mathbb{Q}_{\alpha,\beta}$, 
compute rank-5 solutions $\Vpool$ and $\Vmaxrcs$, and report the difference in average reconstruction error (over the source domains) and worst-case reconstruction error (over the convex hull of the source covariances),
relative to the average performance of \poolPCA (for details, see equations~\eqref{eqn:rel-error-avg} and~\eqref{eqn:rel-error-wc} in Appendix~\ref{app:eval-metrics}).
The results are shown in Figure~\ref{fig:sim_wc_vs_avg}. 
Solutions to \maxRCS consistently improve worst-case error while incurring only small losses in average error, across varying domain heterogeneity. 
These findings complement Theorem~\ref{thm:maxrcs-convex-hull}(iii): while the theorem guarantees that there exist settings where \maxRCS strictly outperforms \poolPCA in the worst case, the simulations suggest that such improvements occur generically.

\subsubsection{Convergence and finite-sample robustness}\label{sec:sims-finite-sample}
Proposition~\ref{prop:consistency-of-guarantees}(i)
states that the worst-case population loss of the empirical \maxRCS solution converges in probability to the optimal worst-case population loss. 
We illustrate this convergence empirically and additionally assess the finite-sample robustness of \maxRCS relative to pooled PCA.
To do so, we vary the sample size $n \in \{100, 250, 500, 1000, 2000, 5000\}$ and heterogeneity $(\alpha,\beta)$ as in Section~\ref{sec:sims-avg-vs-wc}. 
For each $(n,\alpha,\beta)$ we repeat the following 
steps
25 times. We sample  source covariances $\mathcal{C}_\mathrm{source} := (\Sigma_1,\dots,\Sigma_5) \sim\mathbb{Q}_{\alpha,\beta}$ and, for all domains $1\le e\le 5$, we draw $n$ observations from $\mathcal{N}(0,\Sigma_e)$, and compute the rank-5 empirical and population solutions of \poolPCA and \maxRCS, denoted as $\VpoolHAT$, $\VmaxrcsHAT$, $\Vpool$, and $\Vmaxrcs$, respectively.

We assess convergence by 
reporting the difference in worst-case population reconstruction error (evaluated over the convex hull $\mathrm{conv}(\mathcal{C}_\mathrm{source})$) between empirical and population \maxRCS 
(for details, see equation~\eqref{eqn:diff-in-rcs} in Appendix~\ref{app:eval-metrics}).
Values near zero indicate that the empirical estimator 
attains nearly optimal worst-case performance
over the convex hull of the population matrices.
Figure~\ref{fig:sim_finite_sample} (left) illustrates behavior 
in line
with Proposition~\ref{prop:consistency-of-guarantees}(i), 
with the distribution of the 
difference
increasingly concentrated around zero as $n$ increases.
\begin{figure}[t]
    \centering
    \includegraphics{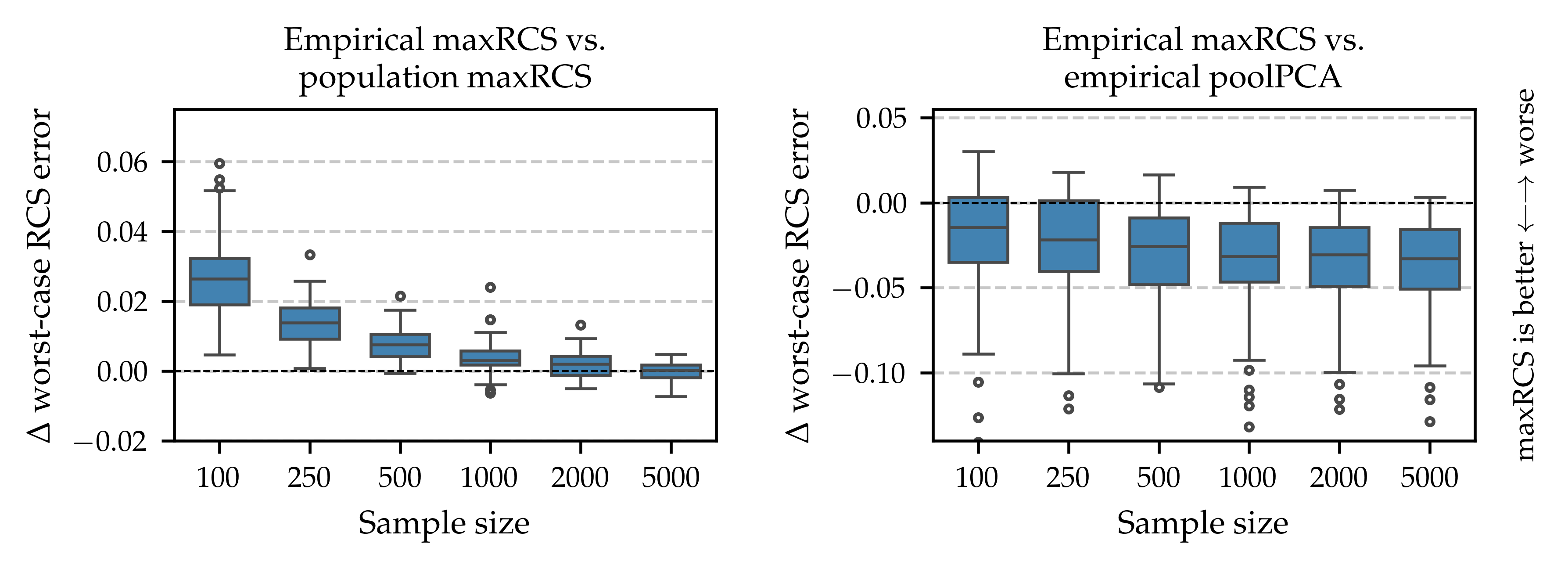}
    \caption{
    \textit{Finite-sample performance of \maxRCS; 
    see Section~\ref{sec:sims-finite-sample}}.
    \textit{Left}: Difference in worst-case 
    loss (reconstruction error) 
    over the convex hull of the population covariance matrices
    between empirical and population \maxRCS.
    \textit{Right}: 
    Analogous difference between 
    between empirical \maxRCS and empirical \poolPCA.
    Negative values indicate that \maxRCS attains lower worst-case error than \poolPCA.
    Across sample sizes, \maxRCS 
    outperforms \poolPCA.
    }
    \label{fig:sim_finite_sample}
\end{figure}

In the same way, we also compare 
empirical \maxRCS and empirical \poolPCA (for details, see equation~\eqref{eqn:rel-error-fs} in Appendix~\ref{app:eval-metrics}).
Figure~\ref{fig:sim_finite_sample} (right) shows that \maxRCS 
outperforms \poolPCA across sample sizes, showing that the practical benefits of \maxRCS 
may
emerge even at modest sample sizes.

\subsubsection{Using regret as a proxy for reconstruction error under heterogeneous noise}\label{sec:sims-noise}
This experiment shows that, under heterogeneous noise across domains, worst-case regret may outperform worst-case reconstruction- or variance-based objectives, even when evaluated using reconstruction error, consistent with the discussion in Section~\ref{sec:discussion-of-objectives}.
We 
again
sample five source covariances $\Sigma_1,\ldots,\Sigma_5$, as described in Appendix~\ref{app:datagen}, 
using the variant in which the
domain-specific 
eigenvalues
differ across domains.
Then,
for all domains $e\in\E$, we draw a noise level $\sigma_e \sim U([0,0.1])$ and sample $n=2000$ 
i.i.d.\ training observations 
according to
\begin{equation*}
    \mathbf{z}_e = \mathbf{x}_e + \sigma_e \boldsymbol{\varepsilon}_e, 
    \quad 
    \mathbf{x}_e \sim \mathcal{N}(0,\Sigma_e),\ 
    \boldsymbol{\varepsilon}_e \sim \mathcal{N}(0,I_p);
\end{equation*}
thus,
the population covariance is $\Sigma_e + \sigma_e^2 I_p$.
We solve rank-$10$ and rank-$5$ \maxRCS and \maxregret\footnote{
    As the domain-specific eigenvalues vary between domains, the solutions to \maxRCS and \maxregret no longer necessarily coincide. 
    However, all covariances are trace-normalized, so the solutions to \minPCA, \maxRCS, and their normalized variants all coincide, and we report results only for \maxRCS. 
    Likewise, the solutions to \maxregret and \normmaxregret coincide, and we report results only for \maxregret.
} 
using the empirical covariances, and evaluate both methods  in terms of worst-case reconstruction error on $n$ independent observations from $\mathcal{N}(0,\Sigma_e)$ without additional heterogeneous noise. 
The experiment is repeated 25 times and results are shown in Figure~\ref{fig:sim_het_noise}.

When the approximation rank $k$ matches the rank of the covariances (rank $10$), the solutions of \maxregret and \maxRCS coincide in the noiseless population case 
and the population solution of \maxregret is insensitive to heterogeneous noise 
(see Section~\ref{sec:discussion-of-objectives} for both statements)
For finitely many observations, these equalities do not hold exactly. 
Nevertheless, Figure~\ref{fig:sim_het_noise} (left) 
\begin{figure}[t]
    \centering
    \includegraphics[scale=1]{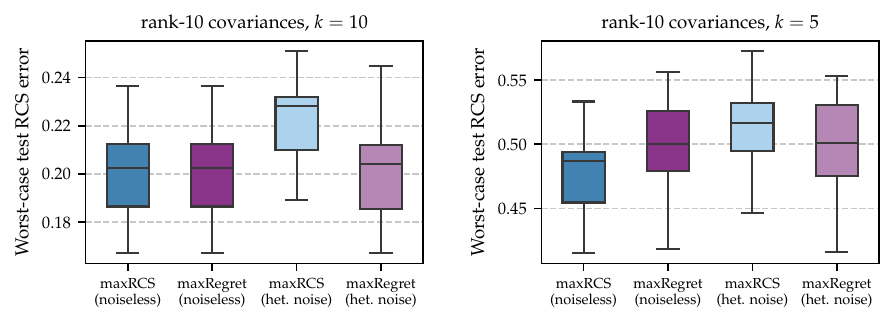}
    \caption{
        \textit{Regret 
        may be
        preferable under heterogeneous noise; 
        see Section~\ref{sec:sims-noise}.}
        \textit{Left}: As suggested by the population theory, solutions to \maxRCS and \maxregret are almost identical in the noiseless case. With heterogeneous noise, \maxregret behaves similarly as if the data were noiseless. 
        \textit{Right}: When $k$ is strictly smaller than the rank of the covariance matrices, 
        solutions to \maxRCS and \maxregret differ in the noiseless case.
        Nevertheless,
        under heterogeneous noise, \maxregret can still be preferable over \maxRCS even if the evaluation criteria is reconstruction error.
    }
    \label{fig:sim_het_noise}
\end{figure}
illustrates empirical behavior consistent with the
population theory: in the noiseless case, \maxRCS and \maxregret yield similar worst-case reconstruction errors, and under heterogeneous noise, \maxregret 
approximately recovers the solution obtained in the noiseless case,
whereas \maxRCS incurs increased worst-case reconstruction error.
When using a lower-rank approximation ($k=5$), 
the population solution of \maxregret is still unaffected by heterogeneous noise (see Section~\ref{sec:discussion-of-objectives}), but those of
\maxregret and \maxRCS no longer coincide
(see, e.g., Example~\ref{example1}).
Figure~\ref{fig:sim_het_noise} (right) shows similar empirical results: 
in the noiseless case, \maxRCS and \maxregret differ, and under heterogeneous noise, \maxregret exhibits similar behavior while 
\maxRCS incurs higher worst-case reconstruction error.
These results indicate that regret-based objectives may be preferable under heterogeneous noise, even when the goal is robust reconstruction.

\subsection{Worst-case inductive matrix completion}\label{sec:sims-maxmc}
We study worst-case matrix completion in two regimes: 
(i) fully observed source domains (the setting of Theorem~\ref{thm:mc}), and 
(ii) partially observed source domains. 
Our simulations show that \maxMC generally reduces worst-case reconstruction error relative to \poolMC, including when source domains are partially observed and missingness exceeds the theoretical bound.

As in Section~\ref{sec:sims-avg-vs-wc}, 
we vary domain heterogeneity, which is controlled by $(\alpha,\beta)$.
For each configuration, we repeat the following steps 25 times. 
We sample five rank-$10$ source covariances $(\Sigma_1,\ldots,\Sigma_5)\sim\mathbb{Q}_{\alpha,\beta}$ with dimension $p=500$ and, for all domains $1\le e \le 5$, we draw $n=1000$ observations from $\mathcal N (0, \Sigma_e)$. 
In regime (i), source data are fully observed; in regime (ii), 90\% of entries in each source row are selected uniformly at random
and masked. 
In each regime, we learn shared right factors via \poolMC and \maxMC.
For evaluation, we draw 1000 independent test observations per domain and mask 90\% of entries uniformly at random.
Reconstruction is performed using the least-squares rule of Section~\ref{sec:inductive-mc}. 
Performance is measured by per-entry mean squared reconstruction error, 
either
averaged across domains 
or
in the worst-case domain (see Appendix~\ref{app:eval-metrics} 
for details).

Figure~\ref{fig:maxmc_all} shows the results of the more difficult regime~(ii);
\begin{figure}
    \centering
    \includegraphics[scale=1]{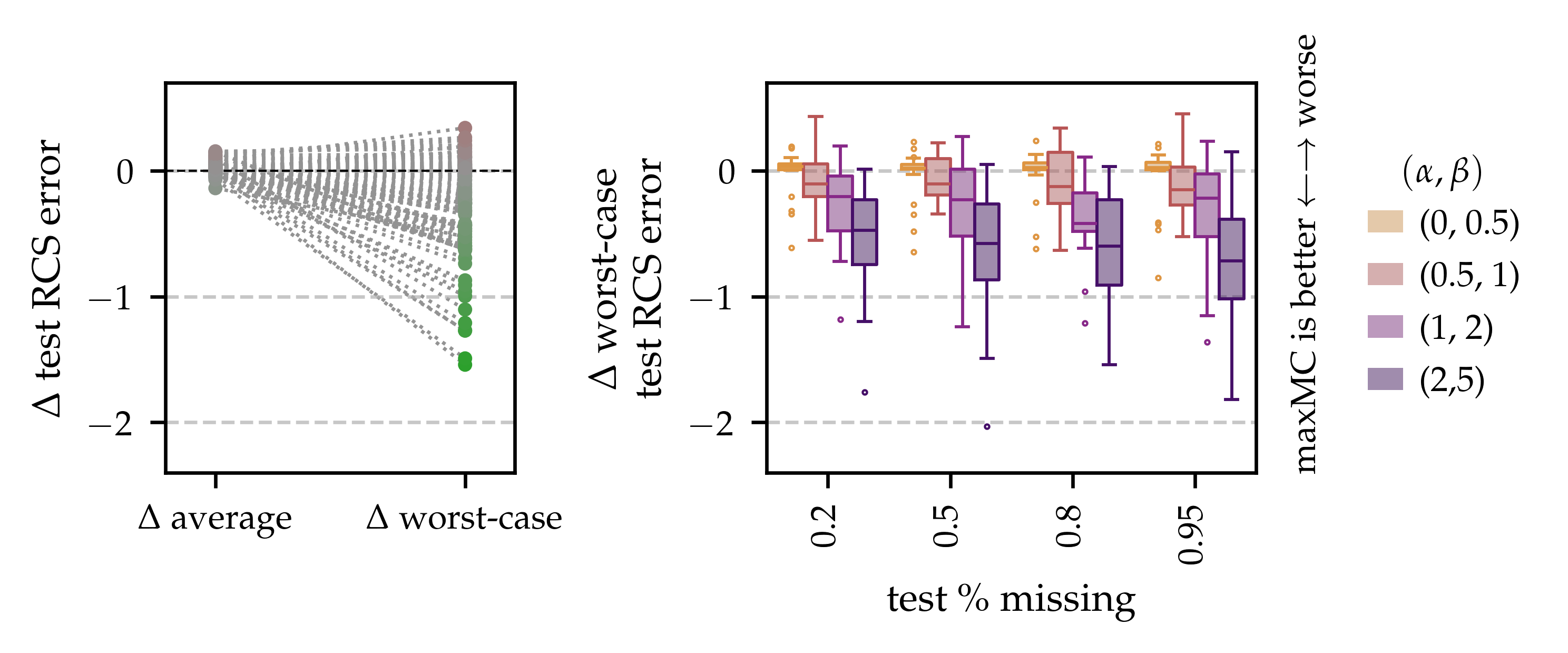}
    \caption{
        \textit{Worst-case inductive matrix completion with partially observed source domains (90\% missing entries; 
        regime (ii)).}
        \textit{Left:} Difference in average and worst-case test reconstruction error between \maxMC and \poolMC. 
        In line with Theorem~\ref{thm:mc} (which covers the population case and allows for some $\epsilon$ deviation),
        \maxMC 
        generally
        improves worst-case performance; 
        in terms of average error, it incurs only minor changes.
        \textit{Right:} Worst-case difference as the proportion of missing target entries varies, for different heterogeneity levels $(\alpha,\beta)$.  \maxMC achieves lower median worst-case error than \poolMC, with larger gains under stronger domain heterogeneity.  
    }
    \label{fig:maxmc_all}
\end{figure}
the left panel
shows that \maxMC generally reduces worst-case target reconstruction error relative to pooling while incurring only small changes in average performance. 
The right panel varies the proportion of missing target entries and shows that, under high domain heterogeneity, the worst-case advantage of \maxMC persists even when substantially more entries are missing than 
specified
by Assumption~\ref{ass:sufficient-obs}. Under low heterogeneity, performance is comparable to \poolMC, reflecting the reduced variation across domains.
Results for the simpler regime (i) are qualitatively similar and are shown in Figure~\ref{fig:maxmc_all_observed_sources} in Appendix~\ref{app:sims-mc}.

\section{Applications}\label{sec:applications}
\subsection{Explaining variance in held-out FLUXNET regions}\label{sec:appl:fluxnet}
We illustrate the benefits of worst-case PCA objectives on FLUXNET data \citep{pastorello2020fluxnet2015, baldocchi2008breathing}. 
These data contain measurements of variables related to biosphere and atmosphere interactions obtained at different places on the Earth (see Figure~\ref{fig:appl-fluxnet} left). More precisely
the variables comprise gross primary productivity (GPP), meteorological variables (air temperature, vapor pressure deficit, and day and night land surface temperature), and remote-sensing indices (EVI, NDVI, NDWI, LAI, NIRv, and fPAR). We group the data into TransCom regions 
\citep{gurney2004transcom}
as domains (see Figure~\ref{fig:appl-fluxnet}, left), which spans heterogeneous climates and measurement conditions.
The data pre-processing is described in Appendix~\ref{app:appl:fluxnet}. 
Based on scree plots and the Kaiser criterion (Appendix~\ref{app:appl:fluxnet}, Figure~\ref{fig:scree_plots}), we consider rank-$2$ approximations to the data. 

We perform 20 random splits of the regions into five source regions and eight target regions. On each split, we solve rank-$2$ \poolPCA, \maxregret, \normmaxregret, and \normminPCA on the source regions. 
This setting extends the illustrative example in Section~\ref{sec:intro}, where we compared \poolPCA and \normmaxregret on a single split of the TransCom regions.
We use the normalized and regret-based variants because the total variance in each domain varies 
(see Section~\ref{sec:discussion-of-objectives}).
For comparison, we solve rank-$2$ PCA on the average empirical covariance ($\frac{1}{|\E|} \sum_{e\in\E}\hat\Sigma_e $), denoted \emph{\avgcovPCA}, which removes the effect of unequal sample sizes across regions; 
to the best of our knowledge, it does not come with any extrapolation guarantee.

Figure~\ref{fig:appl-fluxnet} (right) reports the relative difference in worst-case proportion of explained variance between \poolPCA and the other methods over the target regions (results for additional components are reported in Appendix~\ref{app:appl:fluxnet}, Figure~\ref{fig:fn_reps}.). 
\begin{figure}
    \hfill
    \begin{minipage}[t]{0.5\linewidth}
        \vspace{0pt}
        \centering
        \includegraphics[scale=1]{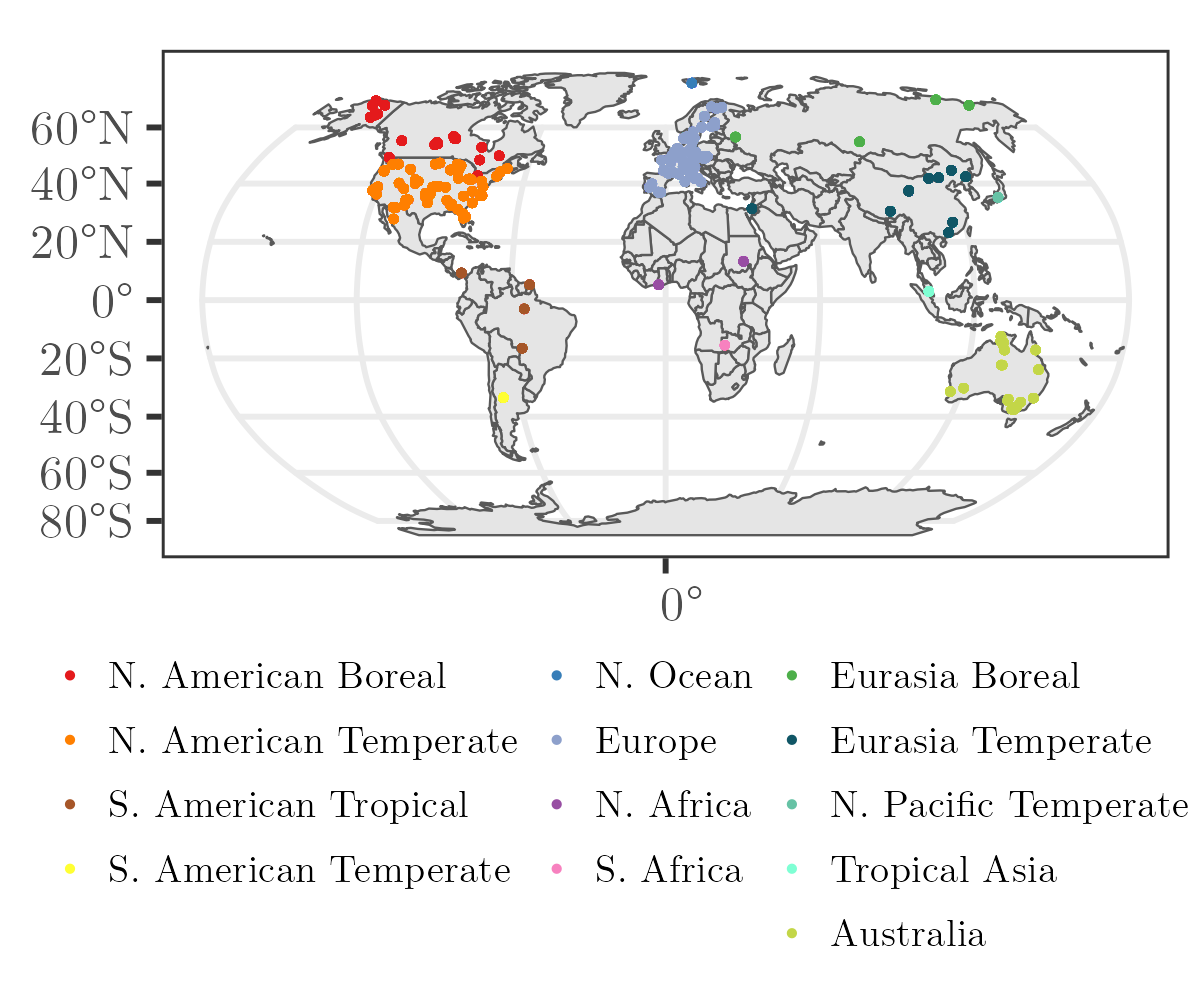}
    \end{minipage}
    \hfill
    \begin{minipage}[t]{0.45\linewidth}
        \vspace{0pt} 
        \centering
        \includegraphics[scale=1]{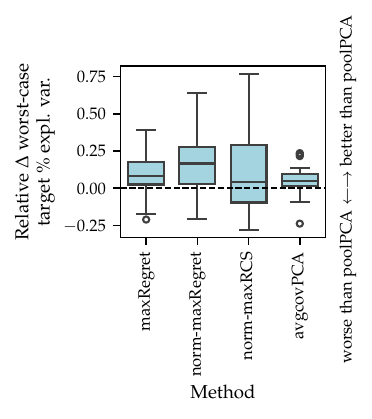}
    \end{minipage}
    \hfill
  \caption{
    \textit{Left}: 
        Map of FLUXNET sites, colored by TransCom region. 
    \textit{Right}: 
        Relative difference in worst-case proportion of explained variance between \poolPCA and \maxregret, \normmaxregret, \normmaxRCS, and \avgcovPCA, evaluated on held-out target regions over 20 random train–test splits. 
        Positive values indicate improved worst-case performance relative to pooled PCA.
        The worst-case objectives generally improve robustness, with \normmaxregret achieving the strongest
        gains in 
        median
        worst-case proportion of explained variance.
    }
  \label{fig:appl-fluxnet}
\end{figure}
Positive values indicate improved worst-case performance relative to \poolPCA. 
All alternatives to \poolPCA improve worst-case performance on the target regions, including PCA on the average covariance. However, the worst-case objectives typically perform best: \normmaxregret achieves the strongest gains and improves worst-case explained variance in 16 out of 20 splits (p-value of a binomial test equals 
ca.\ 0.01) 
and a median improvement of 0.07 in absolute worst-case proportion of explained variance.

\subsection{Reanalysis of the three major axes of terrestrial ecosystem function}\label{sec:appl:ecosys}
We 
now revisit part of the
analysis 
conducted by
\citet{migliavacca2021three}.
The authors aggregate
FLUXNET data 
over time 
to derive site-level functional ecosystem properties
related
to carbon uptake and release, energy and water exchange, and growing-season water-use efficiency (see Table~\ref{tab:ecosys_traits} in Appendix~\ref{app:appl:ecosys}). 
Motivated,
for example,
by the 
task of predicting
ecosystem responses to climatic changes, the authors sought a low-dimensional representation of ecosystem functioning, defining three 
axes of variation of terrestrial ecosystem function
via PCA on the pooled dataset. 
Based on the loadings, these three axes are then interpreted to represent 
``maximum ecosystem productivity'', ``ecosystem water-use strategies'', and
``ecosystem carbon-use efficiency'' \citep{migliavacca2021three}.

Arguably, such interpretations should be consistent over different contintents.
We
thus
evaluate the robustness of the 
three
axes, 
when treating
continents as distinct domains. We compare 
\poolPCA (which corresponds\footnote{Here, we pre-process the data similarly to Section~\ref{sec:appl:fluxnet}, which 
deviates slightly from the procedure used by \citet{migliavacca2021three}.
The differences are minor: for example, 
the loadings from \poolPCA shown in Figure~\ref{fig:appl_ecosys} are almost identical to those presented by \citet{migliavacca2021three}.
All details on pre-processing are provided in Appendix~\ref{app:appl:ecosys}.}
to the analysis by
\citet{migliavacca2021three})
against \normmaxRCS and \normmaxregret, which explicitly optimize for worst-case performance across domains. 
(We use the normalized variants because total variance differs across continents, see Section~\ref{sec:discussion-of-objectives}; 
for completeness, results for PCA based on the average empirical covariance and for \maxregret are reported in Appendix~\ref{app:appl:ecosys}.)

Figure~\ref{fig:appl_ecosys} compares the pooled and worst-case solutions. 
\begin{figure}[t]
    \centering
    \includegraphics[scale=1]{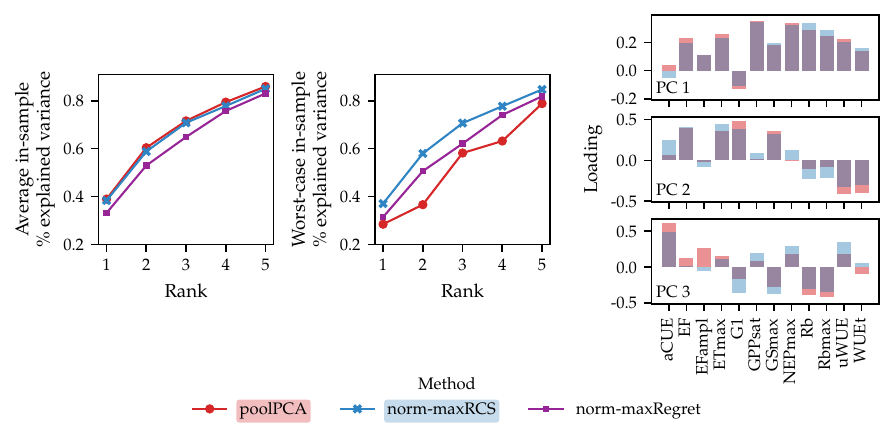}
    \caption{
        Comparison of dimensionality reduction methods on ecosystem function data across continents. 
        \textit{Left}: Proportion of variance explained on the pooled dataset.
        \textit{Centre}: Worst-case proportion of variance explained across continents. 
        \textit{Right}: Loadings of the top three principal components for \poolPCA and \normmaxRCS.
        \normmaxRCS improves worst-case proportion of explained variance with negligible loss in average performance, while largely retaining the physical interpretability of the principal axes.
    }
    \label{fig:appl_ecosys}
\end{figure}  
For \poolPCA, we observe a substantial gap between the average in-sample proportion of explained variance and the worst-case performance across continents.
In contrast, \normmaxRCS achieves nearly identical average proportion of explained variance while substantially improving worst-case performance. \normmaxregret shows a similar, though smaller, improvement.

We further compare the loadings of \poolPCA and \normmaxRCS (Figure~\ref{fig:appl_ecosys}, right), providing a diagnostic for whether the identified axes reflect domain-invariant ecological structure
(Appendix~\ref{app:ordering-basis} details the construction of the ordered worst-case basis).
The leading axis remains largely stable, reflecting maximum ecosystem productivity and its coupling with ecosystem respiration. 
The second axis 
may still be interpreted
as a water-use strategy gradient dominated by the same six
metrics,
however, \normmaxRCS reduces the influence of the water use efficiency (WUE) loadings in favor of the remaining properties. 
The third axis, 
originally interpreted as
carbon-use efficiency 
(aCUE, Rb, Rb$_{\max}$, EF$_\mathrm{ampl}$), 
shifts most prominently:
EF$_\mathrm{ampl}$ is effectively removed while G1 and GS$_{\max}$ are upweighted, indicating a stronger contribution of water-regulation traits.
In our view, the stability of the first two axes under worst-case optimization strengthens the interpretation of these axes as fundamental ecological gradients; 
our results also suggest that the interpretation of the third axis may warrant further investigation.

\section{Discussion}\label{sec:discussion}
This paper studies worst-case formulations of low-rank approximation across multiple domains.
Considering the worst-case instead of average-case can substantially improve performance in new target domains.
The key
theoretical guarantee 
is out-of-sample worst-case optimality over the convex hull of the source covariances (or of their normalized versions). 
Empirically, these objectives improve worst-case performance with only modest losses on average, 
as shown both on simulated data and case studies using real-world FLUXNET data.

The different worst-case objectives are not interchangeable.
Normalized objectives mitigate sensitivity to scale differences by measuring proportional explained variance (or normalized reconstruction error), but unit comparability across domains is lost.
Because worst-case objectives depend on `extreme' domains, noisy or small domains can affect their solutions; the extent of this effect depends on the chosen objective. 
Normalization and the use of the regret can mitigate these effects to some extent.
In particular, regret-based objectives 
evaluate performance relative to each domains'
optimal subspace and are robust to heterogeneous noise. 
On real data,
we found that regret performs best even when the objective is explained variance or reconstruction error.

As in classical PCA, pre-processing choices (such as centering, scaling, standardization) matter. These choices can change which domain is worst-case and therefore change the solution. 
An additional choice is whether to form the pooled covariance from pooled data or to average the domain covariances (these coincide when domain sample sizes are equal). While averaging can reduce the influence of large domains, as far as we know, it does not come with 
comparable
out-of-sample guarantees.

When data are only partially observed, low-rank approximation 
can be used to tackle
matrix completion. 
In the fully observed training regime, our convex-hull robustness guarantee transfers to inductive completion (up to an approximation factor). When source domains are also partially observed, \maxMC remains empirically effective but our theory does not yet cover this setting; extending the out-of-sample guarantees to jointly handle domain heterogeneity and missingness in the source domains is left to future work. An additional direction is to allow domain-dependent observation mechanisms 
and to improve the bounds we establish

We regard the following additional extensions to be
particularly promising.
First, worst-case objectives can be used for nonlinear representation learning (e.g., autoencoders) by replacing the reconstruction loss with an appropriate domain-wise worst-case criterion.
Second, robustness to distribution shift can be combined with robustness to corruption and outliers by integrating worst-case domain objectives with robust 
(in the sense of robustness against outliers)
PCA-type modeling.
Third, it is natural to consider interpolations between the pooled and worst-case objectives that trade average and worst-case performance, for instance via convex combinations or regularized formulations.
Solutions to these objectives may provide smoother behavior when the worst-case domain is small or noisy, without discarding the robustness perspective.
Fourth, when taking the diagnostics perspective, see Section~\ref{sec:discussion-of-objectives}, in settings where domains are ordered (e.g., through time), tracking domain-wise losses can be used to detect distribution shifts; for worst-case methods, a sharp degradation may additionally indicate that the target domain has moved outside the uncertainty set implied by the sources (e.g., beyond the convex-hull condition).
Fifth, 
our guarantees formalize robustness of reconstruction criteria, not invariance or stability: for example, a direction shared across domains need not be selected if it explains little variance in each domain. Investigating notions closer to invariance
seems a promising route with potential connections to causality. 

\acks{We thank
    Jacob Nelson for facilitating access to the FLUXNET data and preparing it for our use, and
    Maybritt Schillinger for many helpful 
    discussions. 
    We also thank Francesco Freni for 
    helpful discussions
    on optimization and Linus K\"uhne for ideas related to Section~\ref{sec:sims-noise}. 
    We have used LLMs for debugging code and improving language formulations.
} 

\bibliography{bib}

\appendix
\numberwithin{figure}{section}
\renewcommand{\thefigure}{\thesection.\roman{figure}}
\clearpage

\startcontents[appendices]
\section*{Contents of the Appendix}
\printcontents[appendices] 
  {l}                      
  {1}                      
  {\setcounter{tocdepth}{1}} 
\vspace{5mm}
\section{Additional details on \wcPCA}\label{appendix:add-results}
\subsection{Sequential versus joint optimization for \minPCA}\label{app:seq-joint}
This section provides a formal version of Proposition~\ref{prop:seq-ne-joint-informal}. 
In particular, we show that the sequential extension of PCA to the \minPCA objective can fail to recover the solution of the corresponding joint optimization problem.
\begin{proposition}\label{prop:seq-ne-joint}
    Consider Setting \ref{setting:set-up}. Define $\Vseqminpcak:= [v_1, \ldots, v_k] \in \Rpk$ as
    $k$ orthonormal vectors such that 
    \begin{align*}
        v_1 &\in \argmax_{v\in\R^p, \|v\| = 1} \min_{e\in\E} \mathcal{L}_\mathrm{var}(v;P_e) \\
        v_j &\in \argmax_{\substack{v\in\R^p, \|v\| = 1, \\ \forall i < j, v \perp v_i }} 
             \min_{e\in\E} \mathcal{L}_\mathrm{var}([v_1,\ldots,v_{j-1},v]; P_e) \qquad \textrm{for}\ 2\le j \le k.
    \end{align*}
    Then, in general, $\Vseqminpcak$ does not solve \minPCA.
\end{proposition}
\begin{proposition}[Normalized case]
The statement of
Proposition~\ref{prop:seq-ne-joint} holds for the normalized objective, that is, when $\mathcal{L}_\mathrm{var}$ is replaced by $\mathcal{L}_\mathrm{normVar}$.
\end{proposition}
\begin{proof}
The same counter-example can be used to prove both statements; we present it in terms of \minPCA.
Consider three domains $\E = \{1,2,3\}$ with $p=5$ and the following diagonal population covariances:
\begin{align*}
\Sigma_1 := \tfrac{1}{4}\,\mathrm{diag}(2,2,0,1,1),\quad
\Sigma_2 := \tfrac{1}{4}\,\mathrm{diag}(2,0,2,1,1),\quad
\Sigma_3 := \tfrac{1}{4}\,\mathrm{diag}(0,2,2,1,1).
\end{align*}
We find 
a
rank-2 solution of the sequential procedure and then show that it is not optimal for the joint objective. 

Step 1:
We determine the first sequential component,
i.e.,
$$v_1^* \in \argmax_{v\in\R^5;\,\|v\|=1} \min_{e\in\E} v^\top \Sigma_e v.$$
Let 
$v_1^*$
$=(a,b,c,d,f) \in \R^5$, 
$\|v_1^*\|=1$.
Then,
$a^2 = b^2 = c^2$. We show this by contradiction.
Since $a, b,$ and $c$ appear symmetrically in the objective, 
i.e., $\min_{e\in\E} (v_1^*)^\top \Sigma_e v_1^* = \tfrac{1}{2}\min(a^2 + b^2, b^2 + c^2, a^2 + c^2) + \tfrac{1}{4}(d^2 + f^2)$, we assume w.l.o.g.\ 
that $a^2\le b^2\le c^2$. 
Then, $\min_{e\in\E} (v_1^*)^\top \Sigma_e v_1^* = \tfrac{1}{2}(a^2 + b^2) + \tfrac{1}{4}(d^2 + f^2) =: m_1^*$.
We consider the two cases that lead to strict inequalities.
First,
if $a^2<b^2$, let 
$\delta := \tfrac{1}{4}(b^2-a^2) > 0$
and let
$\tilde v := (\sqrt{a^2 + \delta}, \sqrt{b^2 - \delta/2}, \sqrt{c^2-\delta/2}, d, f)$. Then, 
$
\min_{e\in\E} \tilde{v}^\top \Sigma_e \tilde{v} =
\tfrac{1}{2}(a^2 + b^2) + \tfrac{1}{4}(d^2 + f^2) + \tfrac{1}{4}\delta 
> m_1^*
$.
Second, if $a^2=b^2$ and $b^2<c^2$, let 
$\delta:=\tfrac{1}{4}(c^2-b^2)$
and let
$ 
\tilde v = (\sqrt{b^2+\delta/2}, \sqrt{b^2+\delta/2}, \allowbreak \sqrt{c^2 - \delta}, d, f)$.
Then,
$
\min_{e\in\E} \tilde{v}^\top \Sigma_e \tilde{v} = 
a^2 + \tfrac{1}{4}(d^2 + f^2) + \tfrac{1}{2}\delta 
> m_1^*
$.
Hence, 
this proves that
$a^2 = b^2 = c^2$.
Then, $\|v_1^*\|=1$ implies $3a^2+d^2+f^2=1$. Hence,
\begin{equation*}
\min_{e\in\E} (v_1^*)^\top\Sigma_e v_1^*
= 0.25 + 0.25a^2.
\end{equation*}
This is maximized 
for
$a^2=1/3$, 
yielding
$v_1^*:=\frac{1}{\sqrt{3}}(1,1,1,0,0)$.

Step 2:
We determine the second sequential component, i.e., 
\begin{equation*}
v_2^* \in \argmax_{\substack{v\in\R^5;\,\|v\|=1;\\ v\perp v_1^*}} \min_{e\in\E} \left\{v^\top\Sigma_e v + (v_1^*)^\top\Sigma_e v_1^* \right\}= \argmax_{\substack{v\in\R^5;\,\|v\|=1;\\ v\perp v_1^*}} \min_{e\in\E} v^\top\Sigma_e v + \frac{1}{3}.
\end{equation*}
Let $v_2^*=(a,b,c,d,f)$. 
$v_2^*\perp v_1^*$ implies $a=-(b+c)$ 
and $\|v_2^*\|=1$ implies $d^2 + f^2 = 1 - a^2 - b^2 - c^2$. Simplifying gives
\begin{equation*}
\min_{e\in\E} \left\{ (v_2^*)^\top\Sigma_e v_2^* + (v_1^*)^\top \Sigma_e v_1^* \right\} = \frac{1}{3} + \frac{1}{4}  + \frac{1}{2} \min_{e\in\E} \{b^2 + bc, -bc, c^2 +bc\}.
\end{equation*}
For all $b,c\in\R$, $\min\{b^2+bc,-bc,c^2+bc\} \leq 0$ with equality 
if and 
only if $bc=0$ or $b=-c$. 
Thus, the sequential solution has explained variance at most $\frac{1}{3}+\frac{1}{4}=\frac{7}{12}$.

Step 3:
We show that the sequential solution is suboptimal for \minPCA.
Consider the rank-2 orthonormal matrix  $\tilde V := [\tilde v_1, \tilde v_2]$ where $\tilde v_1 := \tfrac{1}{\sqrt2}(1,1,0,0,0)$ and $\tilde v_2 := \tfrac{1}{\sqrt2}(0,0,1,1,0)$. For each domain, the total variance equals $\tfrac{5}{8}$, which is strictly larger than$\frac{7}{12}$.
Thus, the sequential solution does not solve \minPCA.
\end{proof}

\subsection{Ordering the basis of \normmaxRCS}\label{app:ordering-basis}
Sequential constructions are, in general, not optimal for the joint worst-case objective (see Proposition~\ref{prop:seq-ne-joint}). 
Nevertheless, in some applications an ordered basis is required. 
We therefore describe a post-hoc procedure that imposes an ordering within the optimal worst-case subspace.
We present the construction for \normminPCA; the same idea can be applied to the other objectives.

Consider Setting~\ref{setting:set-up} and let $V_k^* \in \mathbb{R}^{p \times k}$ solve \normminPCA.
Starting with the full subspace $S_1 := \mathrm{span}(V_k^*)$, we iteratively identify and remove the direction whose removal maximizes the worst-case explained variance of the remaining subspace.
For $i \in \{1,\ldots, k-1\}$, let
\begin{align*}
    &v_i
    \in
    \argmax_{\substack{v \in \mathrm{span}(V_i); \\ \|v\|=1}}
    \,
    \min_{e \in \mathcal{E}}\frac{1}{\Tr(\Sigma_e)}
    \,
    \Tr\left(
    U_{S_i \cap v^\perp}^\top \Sigma_e U_{S_i \cap v^\perp}
    \right),
\end{align*}
where $S_{i+1} := \mathrm{span}(V_k^*) \cap \mathrm{span}(\{v_1, \ldots, v_{i}\})^\perp$ and for any subspace $S \subset \R^p$, $U_S$ denotes an orthonormal basis of $S$.
Reversing the resulting sequence $(v_1,\ldots,v_{k-1})$ and appending the final remaining direction yields $(v_1',\ldots,v_k')$ such that, for any $j \le k$, the span of $(v_1',\ldots,v_j')$ maximizes worst-case explained variance among all $j$-dimensional subspaces contained in $\mathrm{span}(V_k^*)$.

\subsection{Comparison to Group DRO}\label{app:groupdro}
Group DRO \citep{sagawa2019distributionally, hu2018does} is a widely studied framework for robustness to distributional shifts across a finite set of domains. Our guarantees share a similar structure but apply to a broader class of distributions and rely on weaker assumptions.
\begin{remark}
In Group DRO, the goal is to minimize the worst-case loss over a finite set of distributions $ \{P_e\}_{e \in \E} $. 
Losses of the form $\mathcal{L}(V;P) = \Ex_{\mathbf{x}\sim P}[\ell(V;\mathbf{x})]$ are linear with respect to convex combinations of distributions.
The worst-case loss over all mixtures of the source distributions (i.e., the convex hull of the source distributions) is thus attained at one of the source distributions, that is, 
\begin{equation}\label{eqn:group-dro}
    \max_{e \in \E} \mathcal{L}(V;P_e) = \sup_{P \in \mathcal{Q}} \mathcal{L}(V;P) 
    \quad \mathrm{where} \quad
    \mathcal{Q} := \mathrm{conv}(P_1, \ldots, P_E).
\end{equation}
Our results for \maxRCS, and \minPCA apply in this setting and are strictly stronger, establishing worst-case optimality over a strictly larger uncertainty set $\mathcal{P}$ 
(based on covariance matrices rather than distributions, see~\eqref{eqn:mathcalP}).
even a weaker form of the result does not follow directly from Group DRO.
For \normmaxRCS, \normminPCA, and \normmaxregret the loss is a ratio of expectations and thus not linear 
with respect to convex combinations of distributions.
In this case, the standard Group DRO guarantees over the mixture of source distributions no longer hold, and our guarantees instead apply to an uncertainty set defined via constraints on the normalized covariance matrices.
\end{remark}

\subsection{Empirical objectives and finite-sample guarantees}\label{app:finite-sample}
This appendix collects the finite-sample counterparts of the worst-case objectives introduced in the main text that are not presented in Section~\ref{sec:finite-sample}.  
It also states the finite-sample optimality guarantees of the unnormalized objectives, which mirrors the population results (Theorem~\ref{thm:maxrcs-convex-hull}) with population covariances replaced by empirical covariances.
Consider Setting~\ref{setting:set-up-finite-sample}. For $\hat V\in\Ok$, we say 
\begin{flalign*}
    &\hat V \textrm{ solves \emph{rank-$k$ empirical \normminPCA} if }\hat V \in \argmax_{V\in\Ok} \max_{e\in\E} \left\{ \frac{\Tr(V^\top \hat \Sigma_e V)}{\Tr(\hat \Sigma_e)} \right\},&& \\
    &\hat V \textrm{ solves \emph{rank-$k$ empirical \normmaxRCS} if } \hat V\in \argmin_{V\in\Ok} \max_{e\in\E} \left\{ \frac{\|X_e - X_e VV^\top\|^2_F}{\|X_e\|^2_F} \right\},&&
\end{flalign*}
$\hat V$  solves \emph{rank-$k$ empirical \maxregret} if 
$$
\hat V \in \argmin_{V\in\Ok} \max_{e\in\E} \left\{ \|X_e - X_eVV^\top\|^2_F - \left(\min_{W\in\Ok} \|X_e - X_eWW^\top\|^2_F \right)\right\},
$$
and $\hat V$  solves \emph{rank-$k$ empirical \normmaxregret} if 
$$
\hat V \in \argmin_{V\in\Ok} \max_{e\in\E} \left\{ \frac{\|X_e - X_eVV^\top\|^2_F}{\|X_e\|^2_F} - \left(\min_{W\in\Ok} \frac{\|X_e - X_eWW^\top\|^2_F}{\|X_e\|^2_F} \right)\right\}.
$$

\begin{proposition}[Empirical robustness of \wcPCA]\label{prop:maxrcs-convex-hull-emp}  
    Consider Setting~\ref{setting:set-up-finite-sample} and
    let 
        $
        \hat{\mathcal{C}} := \mathrm{conv}(\{\hat\Sigma_e\}_{e \in \E})
    $
        be the convex hull of the empirical source covariances.
    Let $\hat V_k \in \Ok$ solve empirical rank-$k$ \maxRCS,
    \minPCA, or \maxregret.
    Let $\mathcal{L}$ denote the corresponding loss function, i.e., $\mathcal{L}=\mathcal{L}_\mathrm{RCS}$ for \maxRCS, $\mathcal{L}=-\mathcal{L}_\mathrm{var}$ for \minPCA, and $\mathcal{L}=\mathcal{L}_\mathrm{reg}$ for \maxregret.
    Then, the following statements hold.
    \begin{enumerate}[label=\roman*)]
        \item For all $V_k\in\Ok$, the worst-case loss over 
                        $\hat{\mathcal{C}}$
        equals the worst-case loss over the source domains, i.e.,
            \begin{equation*}
                \sup_{\Sigma \in \hat{\mathcal{C}}}
                    \mathcal{L}(V_k;\Sigma)
                = \max_{e \in \E} \mathcal{L}(V_k;\hat\Sigma_e).
            \end{equation*}
                        \item Consequently, the ordering induced by the worst-case loss over the source domains is preserved over $\hat{\mathcal{C}}$: for all $V_k, W_k \in \Ok$, if
            $
            \max_{e \in \E} \mathcal{L}(V_k;\hat\Sigma_e)
            <
            \max_{e \in \E} \mathcal{L}(W_k;\hat\Sigma_e),
            $
            then
            $
            \sup_{\Sigma \in \hat{\mathcal{C}}} \mathcal{L}(V_k;\Sigma)
            <
            \sup_{\Sigma \in \hat{\mathcal{C}}} \mathcal{L}(W_k;\Sigma).
            $
        \item Hence, $\hat V_k$ is worst-case optimal over $\hat{\mathcal{C}}$, i.e.,
            $\hat V_k \in \displaystyle 
                \argmin_{V \in \Ok} 
                    \sup_{\Sigma \in \hat{\mathcal{C}}} 
                        \mathcal{L}(V;\Sigma).$
        \label{bullet-empirical-optimality}
        \item 
        Statement \ref{bullet-empirical-optimality} does not generally hold for the solutions to empirical \poolPCA and \sepPCA.
    \end{enumerate}
\end{proposition}

\section{Additional experimental details and results}
\subsection{Data generating process for source and target covariances}\label{app:datagen}
In each experiment, we consider $E$ source domains, obtained by sampling $E$ population covariances matrices $(\Sigma_1,\ldots,\Sigma_E)$ in $\R^{p\times p}$ of rank $10$ from a measure $\mathbb{Q}_{\alpha,\beta}$ that depends on parameters $\alpha, \beta$ (explained below). 
Unless otherwise specified, $E=5$ and $p=20$.
Under the measures $\mathbb{Q}_{\alpha,\beta}$, the sampled covariances consist of a rank-5 shared component and a rank-5 domain-specific component, and 
each of them is
trace-normalized. 
For the simulations in Section~\ref{sec:sims-illustrate}--\ref{sec:sims-finite-sample} and Appendix~\ref{app:algo}, the domain-specific components have 
an equal set of
eigenvalues across domains, and for Section~\ref{sec:sims-noise}, these 
sets of
eigenvalues vary across domains.
We now describe the generation of the source covariances in more detail.
\paragraph{Source covariances.}
To sample $E$ 
source covariances, denoted as 
$(\Sigma_1, \ldots, \Sigma_E)\sim\mathbb{Q}_{\alpha,\beta}$,
we proceed as follows. 
\begin{enumerate}
    \item 
        Sample the eigenvalues and eigenvectors of the shared component as $\lambda_1,\ldots,\lambda_5 \iid U([0.1,1])$ and $Q \sim \mathbb P_\mathbb{O}$, where $\mathbb P_\mathbb{O}$ denotes the Haar distribution 
        over 
        orthonormal
        matrices $Q\in\R^{p\times p}$ with $Q^\top Q = QQ^\top = I_p$. 
        Let $V \in \R^{p\times 5}$ denote the first five columns of $Q$, forming an orthonormal basis of the shared rank-$5$ subspace, and let 
        let $\boldsymbol{\lambda} := (\lambda_1,\ldots,\lambda_5) \in \R^5$.
    \item 
        Sample the 
        eigenvalues of the domain-specific component as $\gamma_1,\ldots,\gamma_5 \iid U([\alpha, \beta])$ (unless otherwise specified, $\alpha=0.1$ and $\beta=1$).
        Let $\boldsymbol{\gamma} := (\gamma_1,\ldots,\gamma_5) \in \R^5$ 
        (we use the same $\boldsymbol{\gamma}$ for all source domains).
    \item 
        Sample the (domain-specific) eigenvectors of the domain-specific component as 
        follows. For $1 \le e \le E$, sample a Gaussian matrix $G_e \in \R^{p \times 5}$ with i.i.d.\ standard normal entries.  
        Let $P_\perp := I - VV^T$ be the projector onto the orthogonal complement of $\mathrm{span}(V)$.  
        We define $V_e \in \R^{p \times 5}$ as the $Q$ factor in the QR decomposition of $P_\perp G_e$.  
        Then the columns of $V_e$ form an orthonormal basis of a rank-$5$ subspace orthogonal to $V$.
    \item For all $1\le e\le E$, define
        \begin{equation*}
            \Sigma_e := \frac{1}{\sum_{i=1}^5(\lambda_i + \gamma_i)} \Big(
                V \diag(\boldsymbol{\lambda})V^\top 
                + V_e \diag(\boldsymbol{\gamma}) {V_e^\top}
            \Big).
        \end{equation*}
\end{enumerate}
\paragraph{Modifications for Section~\ref{sec:sims-noise}.}
In Section~\ref{sec:sims-noise}, we modify the procedure as follows. In step 2, instead of sampling shared domain-specific eigenvalues, we sample different eigenvalues for each domain. That is, for all domains $1\le e\le E$, we sample $\gamma_1^e,\ldots,\gamma_5^e \iid U([\alpha, \beta])$ and let $\boldsymbol{\gamma}_e := (\gamma_1^e,\ldots,\gamma_5^e)\in \R^5$. 
In step 4, we then replace the eigenvectors with the domain specific ones, that is, for all $1\le e\le E$, for $1\le i\le 5$, we replace~$\gamma_i$ with~$\gamma_i^e$ and we replace $\boldsymbol{\gamma}$ with $\boldsymbol{\gamma}_e$.

\paragraph{Target covariances.} The set of target covariances, $\mathcal{C}_\mathrm{target}^t$, consists of $t$ covariances sampled 
i.i.d.\
from the convex hull of the source covariances, $\mathcal{C} := \mathrm{conv}(\{\Sigma_e\}_{e=1}^E)$. 
Each target covariance $\Sigma$ is obtained by sampling $w$ uniformly from the $E$-dimensional simplex and setting $\Sigma = \sum_{e=1}^E w_e \Sigma_e$. 
The superscript $t$ is omitted when the number of target covariances is clear from context.

\subsection{Optimization algorithms}\label{app:algo}
\subsubsection{wcPCA}
The optimization problems  of the \wcPCA variants are non-convex due to the orthogonality constraint set~$\Ok$. 
In principle, any optimization scheme designed for orthogonality constraints may be employed. Several relaxations have been proposed for
some of the variants.
    \citet{tantipongpipat2019multi} (FairPCA) provides a semidefinite programming (SDP) relaxation and multiplicative weights (MW) algorithm for the \minPCA and \maxregret objectives.
   \citet{wang2025stablepca} (StablePCA) provide an optimization scheme for \minPCA, \maxRCS, and \maxregret with convergence guarantees.
    Neither of them consider the other objectives. 
Other approaches, such as extra-gradient methods for saddle-point problems, may also apply in this setting but have not yet been explored for these objectives.

Our focus in this paper lies on the generalization guarantees rather than the optimization. For the implementation, we therefore employ a straightforward projected gradient descent (PGD) scheme implemented in PyTorch \citep{pytorch}, with updates computed via the Adam optimizer. Because the objectives involve pointwise maxima or minima, gradients are propagated only through the active component at each step. Concretely, with current estimate $V_t$, one iteration update is
\begin{align*}
    \nabla \mathcal{L}(V_t) &\leftarrow \text{autograd}(\mathcal{L}, V), \\
    V' &\leftarrow \text{AdamStep}(V_t, \nabla \mathcal{L}(V_t)), \\
    V_{t+1} &\leftarrow \Pi_{\Ok}(V'),
\end{align*}
where $\Pi_{\Ok}$ denotes projection onto $\Ok$, implemented via an SVD-based orthogonalization. To mitigate poor local minima, we run five random restarts and retain the best solution.

To assess the 
effectiveness
of our implementation, we compare the proposed PGD scheme with
    StablePCA and FairPCA.
 We consider three configurations for the number of features $p$ and domains $E$: $(p, E) \in \{(10,5), (10,50), (50,5)\}$, corresponding, respectively, to few features–few domains, few features–many domains, and many features–few domains. For each configuration, 
we repeat the following steps 25 times.
We generate domain-specific covariance matrices as described in Appendix~\ref{app:datagen}. 
For $k = 1, \ldots, \min(p-1,30)$, we solve the rank-$k$ population \minPCA problem using our PGD implementation, StablePCA \citep{wang2025stablepca} and the SDP and MW variants of \cite{tantipongpipat2019multi}. 
When the solution matrix is not rank-$k$ in the 
output of the FairPCA
optimization, we consider only the top~$k$ eigenvectors.

\begin{figure}[t]
    \centering
    \includegraphics[scale=1]{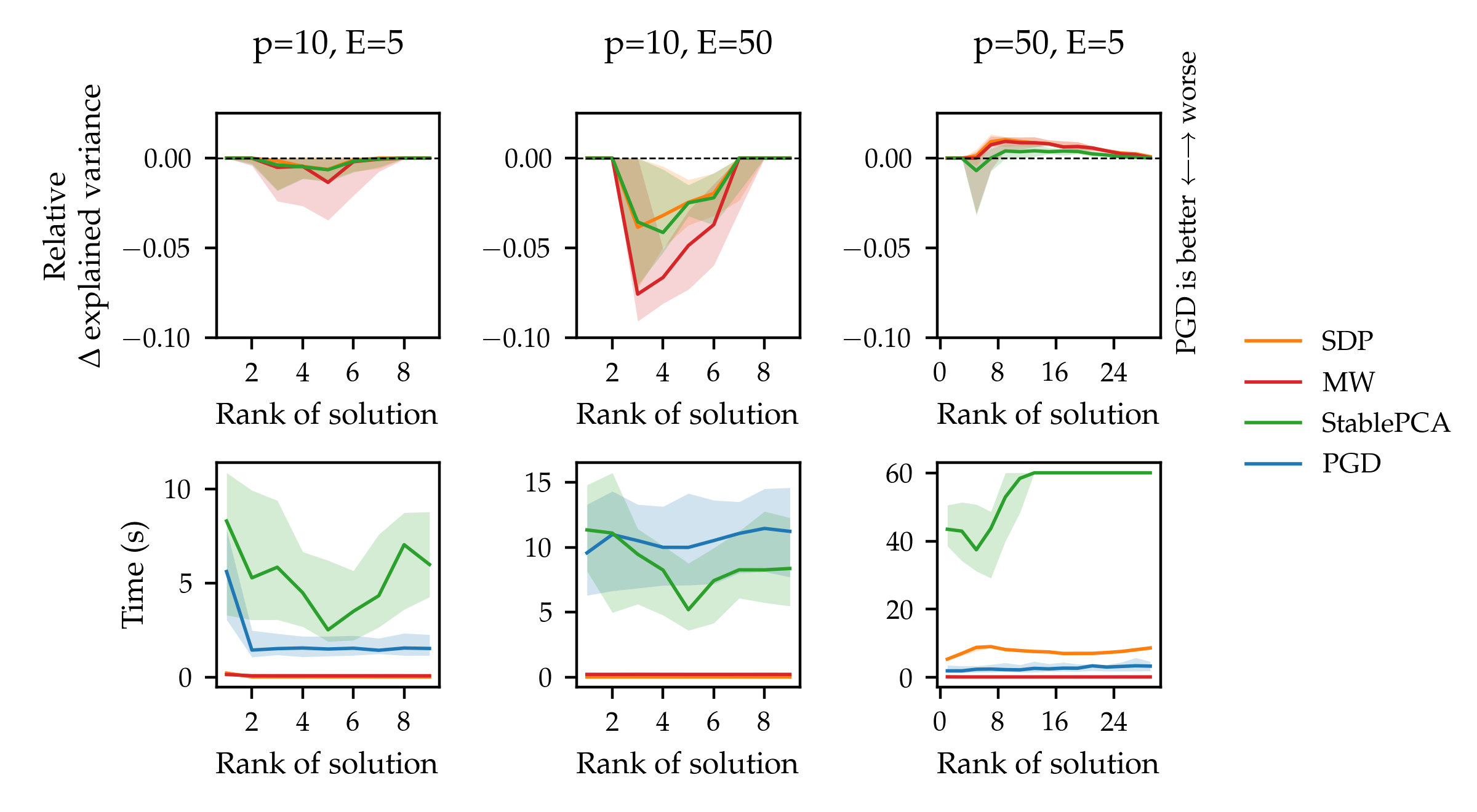}
    \caption{
        Comparison of projected gradient descent (PGD; 
        used in our experiments) 
        with 
        the optimization schemes of \cite{wang2025stablepca} (StablePCA) and \citet{tantipongpipat2019multi} (MW and SDP)
        for solving \minPCA.
        Results are shown for $(p, E) \in \{(10,5), (10,50), (50,5)\}$ where $p$ is the number of features and $E$ is the number of domains.
        \textit{Top row}: median difference in minimum explained variance relative to that of PGD (        negative values indicate that PGD explains more variance in the worst case)
        as a function of the number of components~$k$. 
        \textit{Bottom row}: runtime in seconds for each algorithm.
        Shaded regions represent the 25\%-75\% quantiles.
        PGD attains similar objective values to both baselines, with acceptable runtime.
    }
    \label{fig:comp1}
\end{figure}
Figure~\ref{fig:comp1} reports the 
median
value of
 $(\textrm{minvar}_\textrm{other} - \textrm{minvar}_\textrm{PGD}) \,/ \, | \;\textrm{minvar}_\textrm{PGD}\;|$
where $\textrm{minvar}_\textrm{other}$ and $\textrm{minvar}_\textrm{PGD}$ are, respectively, the minimum explained variance of 
the other
optimization scheme and of PGD and the median running time.
Across all configurations, the objective values obtained by PGD closely match those by StablePCA, SDP, and MW and for smaller problems, PGD explains more variance. 
For smaller problems, the SDP and MW optimization schemes are faster, whereas for higher-dimensional settings only MW is faster than PGD.
As computational cost is not a limiting factor for our analysis, we use the PGD implementation throughout the paper. 
For completeness, we also consider solutions to population \maxregret for $(p,E)=(10,5)$ (shown in Figure~\ref{fig:comp2}), where PGD attains objective values on par with the other optimization schemes
with runtimes that are acceptable for the purpose of our paper.
Here, the other optimization schemes occasionally returns solutions with considerably worse (i.e., higher) objective value. 
\begin{figure}[t]
    \centering
    \includegraphics[scale=1]{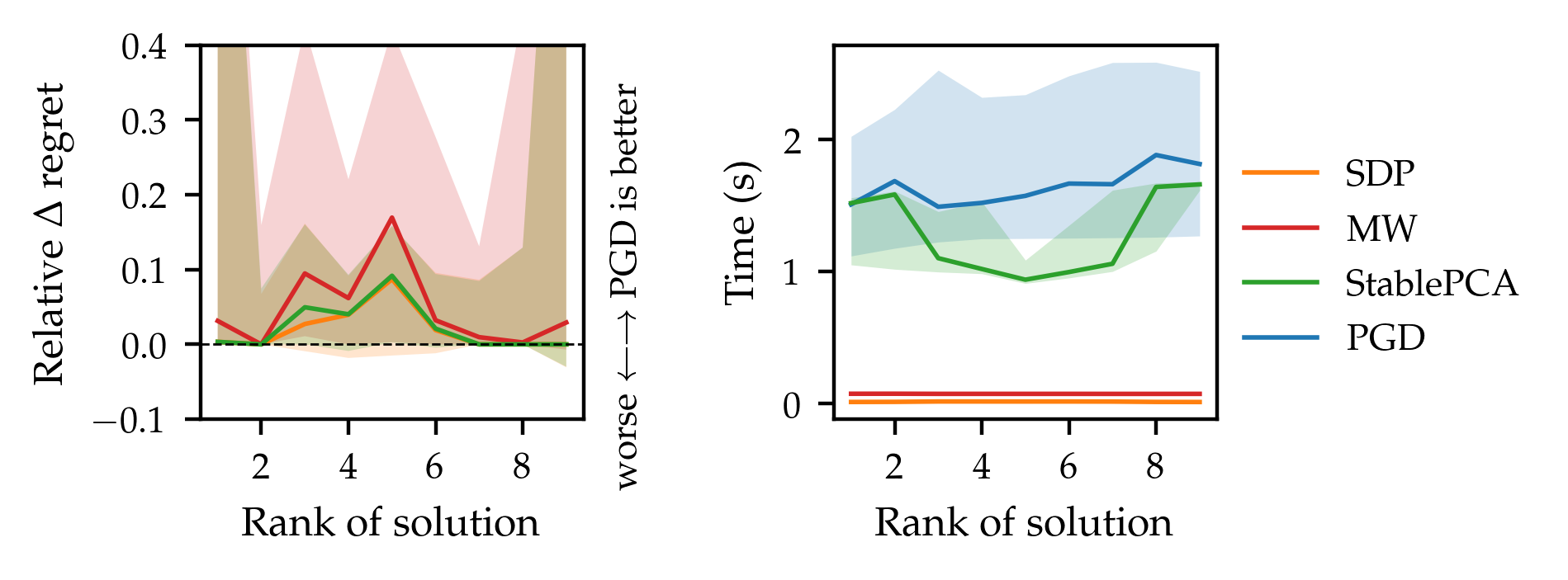}
    \caption{
        Comparison of PGD, StablePCA, MW, and SDP for solving \maxregret for $(p,E)=(10,5)$.
        PGD achieves comparable 
        (or better)
        maximum
        regret values (positive values indicate that PGD performs better, i.e., has a lower worst-case regret).
    }
    \label{fig:comp2}
\end{figure}

\subsubsection{Matrix completion}
The matrix completion objectives \poolMC and \maxMC are non-convex in 
$(L_e, R)_{e\in\E}$,
due to the bilinear term $L_e R^\top$ (see~\eqref{eqn:maxMC}).
We therefore employ a standard alternating-minimization scheme, iteratively updating $R$ given $(L_e)_{e\in\E}$ and then $(L_e)_{e\in\E}$ given $R$.
This scheme is similar to what 
\citet{kassab2026fairernonnegativematrixfactorization} propose for a regret-based variant of fair matrix completion and to what is commonly done in the pooled case \citep[e.g.,][]{jain2013low} -- even though in the pooled case, there is only a single left factor. Our implementation is adapted from Aaron Berk's implementation of alternating minimization for single-domain matrix completion\footnote{\url{https://github.com/asberk/matrix-completion-whirlwind}}.
More precisely,
given current left factors $(L_e)_{e\in\E}$, 
we update the shared right factor $R$ by solving a convex optimization problem.
For \maxMC, this step minimizes the maximum (across domains) of the normalized reconstruction error over the observed entries; for \poolMC, it minimizes the pooled reconstruction error across all observed entries.
Both problems are formulated and solved as second-order cone programs in \texttt{cvxpy}.
Given $R$, the left-factor updates decouple across domains and rows. 
Thus,
for each domain $e$ and row $i$, we estimate ${\ell}_{e,i}$ via the inductive least-squares reconstruction rule of Section~\ref{sec:inductive-mc}, restricted to the entries observed in that row.
The alternating updates are run for at most 100 iterations and terminated early when the improvement in the objective falls below $10^{-4}$.
Initialization of $\{L_e\}$ is done
using an SVD-based heuristic obtained by filling a sparse matrix with the observed entries.

\subsection{Evaluation metrics}\label{app:eval-metrics}
This appendix defines the evaluation metrics used in Sections~\ref{sec:sims-avg-vs-wc}, \ref{sec:sims-finite-sample}, and \ref{sec:sims-maxmc}. Let $\E := \{1,\ldots,5\}$, 
$\mathcal{C}_\mathrm{source} := \{\Sigma_1,\ldots,\Sigma_5\}$ 
denote the set of source covariances and
$\mathcal{C} := \mathrm{conv}(\mathcal{C}_\mathrm{source})$ 
denote their convex hull.
Recall that for a covariance matrix $\Sigma$, for all $V\in\Ok$ the reconstruction loss is
$\mathcal{L}_{\mathrm{RCS}}(V;\Sigma) 
= 
\Tr(\Sigma) - \Tr(V^\top \Sigma V)$.
\paragraph{Average and worst-case performance (see Section~\ref{sec:sims-avg-vs-wc}).}
Let $\bar\Sigma := \frac{1}{|\E|}\sum_{e\in\E}\Sigma_e$ be the average covariance and let $\Vmaxrcs$ and $\Vpool$ solve \maxRCS and \poolPCA. 
The 
difference in average reconstruction error (over the source domains) relative to the average performance of \poolPCA is
\begin{equation}\label{eqn:rel-error-avg}
\Delta\ \mathrm{average}
:=
\frac{
\mathcal{L}_{\mathrm{RCS}}(\Vmaxrcs;\bar\Sigma)
-
\mathcal{L}_{\mathrm{RCS}}(\Vpool;\bar\Sigma)
}{
\mathcal{L}_{\mathrm{RCS}}(\Vpool;\bar\Sigma)
}.
\end{equation}
The 
difference in worst-case reconstruction error (over the convex hull of the source covariances), relative to the average performance of \poolPCA is
\begin{equation}\label{eqn:rel-error-wc}
\Delta\ \mathrm{worst\text{-}case}
:=
\frac{
\sup_{\Sigma\in\mathcal{C}}\mathcal{L}_{\mathrm{RCS}}(\Vmaxrcs;\Sigma)
-
\sup_{\Sigma\in\mathcal{C}}\mathcal{L}_{\mathrm{RCS}}(\Vpool;\Sigma)
}{
\mathcal{L}_{\mathrm{RCS}}(\Vpool;\bar\Sigma)
}.
\end{equation}
Since $\mathcal{L}_{\mathrm{RCS}}$ is linear in $\Sigma$, the supremum over $\mathcal{C}$ is attained at its extreme points; in practice, it is computed as a maximum over the source covariances.
The quantities~\eqref{eqn:rel-error-avg} and~\eqref{eqn:rel-error-wc} correspond to the values plotted in Figure~\ref{fig:sim_wc_vs_avg}.
\paragraph{Convergence and finite-sample robustness (see Section~\ref{sec:sims-finite-sample}).}
To assess convergence of the empirical estimator $\VmaxrcsHAT$ to the population solution $\Vmaxrcs$, we compute the difference in worst-case population reconstruction error (evaluated over the convex hull $\mathrm{conv}(\mathcal{C}_\mathrm{source})$): 
\begin{equation}\label{eqn:diff-in-rcs}
\sup_{\Sigma\in\mathcal{C}}\mathcal{L}_{\mathrm{RCS}}(\VmaxrcsHAT;\Sigma)
-
\sup_{\Sigma\in\mathcal{C}}\mathcal{L}_{\mathrm{RCS}}(\Vmaxrcs;\Sigma).
\end{equation}
This metric corresponds to the quantity plotted in Figure~\ref{fig:sim_finite_sample} (left).
To compare finite-sample robustness 
of empirical \maxRCS and empirical \poolPCA, we additionally report
\begin{equation}\label{eqn:rel-error-fs}
\sup_{\Sigma\in\mathcal{C}}\mathcal{L}_{\mathrm{RCS}}(\VmaxrcsHAT;\Sigma)
-
\sup_{\Sigma\in\mathcal{C}}\mathcal{L}_{\mathrm{RCS}}(\VpoolHAT;\Sigma).
\end{equation}
This metric corresponds to Figure~\ref{fig:sim_finite_sample} (right).
Negative values indicate that the empirical \maxRCS estimator achieves lower worst-case population loss than pooled PCA. Again, in practice, the supremum is computed as a maximum over the source covariances.
\paragraph{Matrix completion (see Section~\ref{sec:sims-maxmc}).}
For the matrix completion experiments, for all $e\in\E$, let $X_e \in \R^{n_{\mathrm{test}}\times p}$ denote the matrix of test observations from domain $e$, and for all $R\in\Ok$, let $\hat X_e(R)$ denote their reconstructions obtained using the inductive least-squares rule of Section~\ref{sec:inductive-mc}. For each domain $e\in\E$, we define the per-entry mean squared reconstruction error as
\begin{equation*}
    \mathcal{L}_e(R)
    :=
    \frac{1}{n_{\mathrm{test}}\,p}
    \|X_e - \hat X_e(R)\|_F^2 .
\end{equation*}
Let $R^{\mathrm{maxMC}}$ and $R^{\mathrm{poolMC}}$ solve \maxMC and \poolMC, respectively. We consider the difference in average and worst-case performance, that is,
\begin{align*}
    \Delta\, \mathrm{average}
    &:= 
    \frac{1}{|\E|}
    \sum_{e\in\E}
    \big(
    \mathcal{L}_e(R^{\mathrm{maxMC}})
    -
    \mathcal{L}_e(R^{\mathrm{poolMC}})
    \big),
    \\
    \Delta\, \mathrm{worst\text{-}case}
    &:= 
    \max_{e\in\E}
    \mathcal{L}_e(R^{\mathrm{maxMC}})
    -
    \max_{e\in\E}
    \mathcal{L}_e(R^{\mathrm{poolMC}}).
\end{align*}
Negative values indicate that \maxMC achieves lower error than \poolMC. In Figure~\ref{fig:maxmc_all}, we report $10^4 \cdot \Delta\, \mathrm{average}$ and $10^4 \cdot \Delta\, \mathrm{worst\text{-}case}$.

\subsection{Inductive matrix completion with fully observed sources}\label{app:sims-mc}
We consider the regime in which the source domains are fully observed, corresponding to the setting of Theorem~\ref{thm:mc}, while target domains remain partially observed. In this case, \maxMC reduces to \maxRCS on the source data. Consistent with the theory, \maxMC typically achieves lower worst-case target reconstruction error than \poolMC (Figure~\ref{fig:maxmc_all_observed_sources}).
For low domain heterogeneity, however, the empirical advantage is occasionally negligible or slightly reversed. This can be attributed to (i) the fact that Theorem~\ref{thm:mc} provides an approximate
worst-case guarantee rather than a strict improvement over pooling, and (ii) finite-sample variability and 
non-optimality when solving
the non-convex optimization procedure, for which improved algorithms may yield tighter empirical performance.
Overall, the results mirror those in Section~\ref{sec:sims-maxmc}: the worst-case advantage of \maxMC is most pronounced under stronger domain heterogeneity and persists across varying levels of missingness.
\begin{figure}
    \centering
    \includegraphics[scale=1]{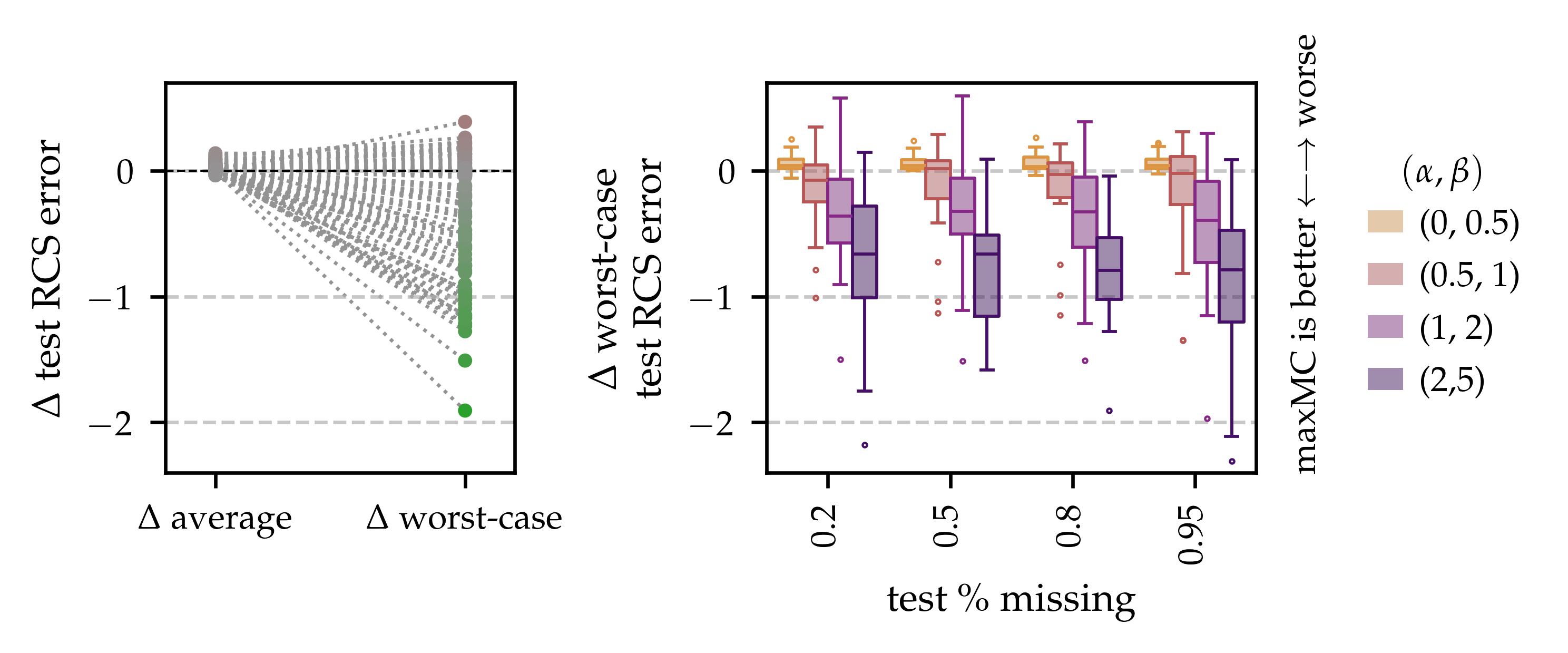}
    \caption{
        \textit{Worst-case inductive matrix completion with fully observed source domains; regime~(i).}
        \textit{Left:} Difference in average and worst-case test reconstruction error between \maxMC and \poolMC. 
        In line with Theorem~\ref{thm:mc} (which covers the population case and allows for some $\epsilon$ deviation),
        \maxMC 
        generally
        improves worst-case performance; 
        in terms of average error, it incurs only minor changes.
        \textit{Right:} Worst-case difference as the proportion of missing target entries varies, for different heterogeneity levels $(\alpha,\beta)$.  \maxMC achieves lower median worst-case error than \poolMC, with larger gains under stronger domain heterogeneity.            
    }
    \label{fig:maxmc_all_observed_sources}
\end{figure}

\section{Additional application details}\label{app:application}
\subsection{Details on the FLUXNET application (Section~\ref{sec:appl:fluxnet})}\label{app:appl:fluxnet}
\paragraph{Data preprocessing.}
We use quality-controlled FLUXNET data \citep{pastorello2020fluxnet2015}, aggregated to daily means.
Let us denote the thirteen TransCom regions as $\E$. For each TransCom region $e\in\E$ with data matrix $X_e^{\mathrm{raw}} \in \R^{n_e\times p}$, we first remove the domain mean and denote the centered matrix by $\tilde{X}_e$. 
We then stack the centered data $\tilde{X}:=[\tilde{X}_1;\dots;\tilde{X}_E]$ and compute pooled columnwise standard deviations
on $\tilde X$. Using these, we standardise $\tilde X$ to obtain $X_{\mathrm{pool}}$, and for each $e\in\E$, we define $X_e$ as the corresponding block of rows of $X_{\mathrm{pool}}$ belonging to domain $e$.
For all $e\in\E$, the empirical covariances are 
$\hat\Sigma_e=\tfrac{1}{n_e} X_e^\top X_e$, and the pooled covariance equals the weighted sum $\hat\Sigma_{\mathrm{pooled}}=\tfrac{1}{n} X_{\mathrm{pool}}^\top X_{\mathrm{pool}} =\sum_e w_e\hat\Sigma_e$ where $n=\sum_j n_j$ and $w_e=n_e/n$. 
The average covariance is given by $\tfrac{1}{|\E|}\sum_{e\in\E} \hat\Sigma_e$.
Intuitively,
this preprocessing removes between-region mean shifts and focuses the analysis on differences in covariance structure across regions.

\paragraph{Choosing the number of principal components.}
Based on scree plots of the source regions (Figure~\ref{fig:scree_plots}), the explained variance levels off after the second component. 
We 
thus use rank-2 approximations in Section~\ref{sec:appl:fluxnet}.
\begin{figure}[t]
    \centering
    \includegraphics[scale=1]{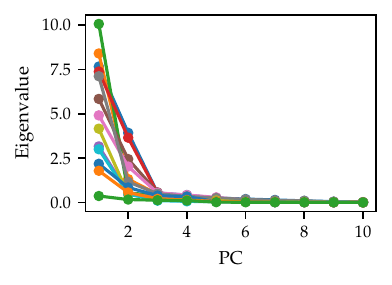}
    \caption{Scree plots for each TransCom region (shown in different colours). 
            In Section~\ref{sec:appl:fluxnet}, we use rank-2 approximations.
    }
    \label{fig:scree_plots}
\end{figure}

\paragraph{Random splits of source and target regions.}
For 20 repetitions, the TransCom regions are randomly partitioned into five source regions and eight target regions. Figure~\ref{fig:fn_reps} reports the relative difference in worst-case proportion of explained variance between \poolPCA and the competing methods (\maxregret, \normmaxregret, \normmaxRCS, \avgcovPCA) over the repetitions --- shown across ranks and on the source and target domains.
\begin{figure}[th!]
    \centering
    \includegraphics[scale=1]{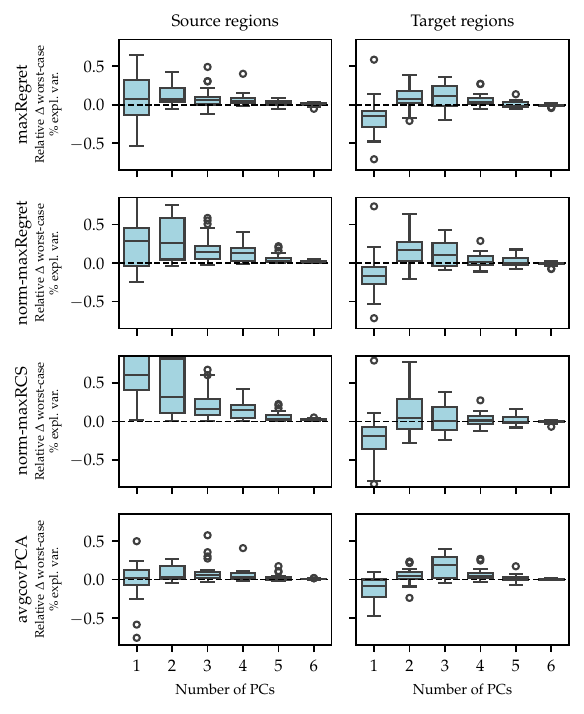}
    \caption{
        Relative difference in worst-case proportion of explained variance between \poolPCA and, from top to bottom, \maxregret, \normmaxregret, \normmaxRCS, and \avgcovPCA. 
        Positive values indicate improved worst-case performance relative to \poolPCA. 
        Results are shown for 20 random region splits, evaluated on the source and target regions across varying numbers of principal components.
        Apart from the rank-one case, the worst-case objectives exhibit improved robustness or similar behavior to \poolPCA.
    }
    \label{fig:fn_reps}
\end{figure}

\subsection{Details on the reanalysis of \citet{migliavacca2021three} (Section~\ref{sec:appl:ecosys})}\label{app:appl:ecosys}
\paragraph{Data description and preprocessing.}
We use the site-level ecosystem function dataset constructed by \citet{migliavacca2021three}, based on FLUXNET observations \citep{baldocchi2008breathing, pastorello2020fluxnet2015}. Each site contributes a single vector of aggregated functional properties (described in Table~\ref{tab:ecosys_traits}). We define continents as domains. Preprocessing follows the procedure described 
in
Appendix~\ref{app:appl:fluxnet}, with continents replacing the TransCom regions.
\begin{table}[t]
\centering
\caption{Ecosystem functional properties used in the reanalysis of \citet{migliavacca2021three}.}
\label{tab:ecosys_traits}
\begin{tabular}{|l|l|}
    \hline
    \textbf{Variable} & \textbf{Description} \\
    \hline
    GPP$_\textrm{sat}$ & Maximum gross CO$_2$ uptake at light saturation \\
    NEP$_\textrm{max}$ & Maximum net ecosystem productivity \\
    ET$_\textrm{max}$ & Maximum evapotranspiration \\
    EF & Evaporative fraction \\
    EF$_\textrm{ampl}$ & EF amplitude \\
    GS$_\textrm{max}$ & Maximum dry canopy surface conductance \\
    Rb & Mean basal ecosystem respiration \\
    RB$_\textrm{max}$ & Maximum basal ecosystem respiration \\
    aCUE & Apparent carbon-use efficiency (fraction of assimilated carbon retained) \\
    uWUE & Underlying water-use efficiency \\
    WUE$_\textrm{t}$ & Transpiration-based water-use efficiency \\
    G1 & Ecosystem-scale stomatal slope parameter \\
    \hline
\end{tabular}
\end{table}

\paragraph{Additional methods.}
As an additional baseline, we solve PCA on the unweighted average of the domain covariance matrices
($\frac{1}{|\E|} \sum_{e\in\E}\hat\Sigma_e $), denoted \emph{\avgcovPCA}.
This estimator removes the implicit weighting of domains by sample size present in \poolPCA. 
In addition, we fit \maxregret. 
Figure~\ref{fig:ecosys:extras} reports the corresponding performance and shows that \maxregret performs almost identically to \normmaxregret. Using the average covariance instead of \poolPCA substantially improves worst-case performance, but it remains slightly inferior to \normmaxRCS both in average explained variance and in the worst-performing continent. 
However, to the best of our knowledge, \avgcovPCA does not come with any extrapolation guarantee. 
\begin{figure}
    \centering
    \includegraphics[scale=1]{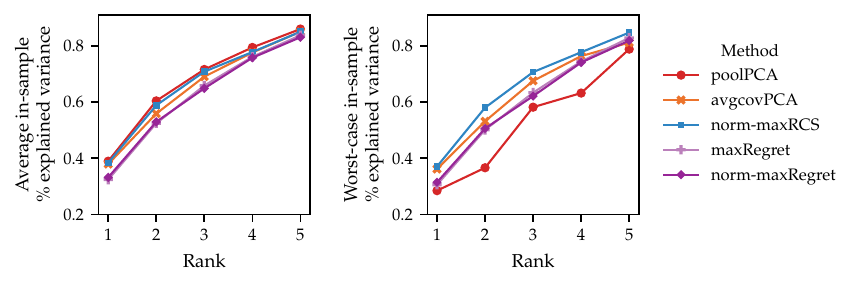}
    \caption{Comparison of dimensionality reduction methods on the ecosystem function dataset.
    \textit{Left}: Proportion of variance explained on the pooled dataset.
    \textit{Right}: Worst-case proportion of variance explained across continents.
    Methods that account for domain heterogeneity outperform \poolPCA in the worst-case continent, with \normmaxRCS most closely matching pooled average performance.
    }
    \label{fig:ecosys:extras}
\end{figure}
\section{Proofs}\label{appendix:proofs}
\subsection{Additional results}\label{app:stable-ols}
There are several variations of the following statement that are known,
    both for single-case deletion, that is, Cook's formula \citep{Cook1977}, and for subset removal \citep[e.g., Section~3.6 in][]{CookWeisberg1982}.
    Bounding this perturbation via $\mu$-incoherence \citep{candes2012exactmc} is a standard technique in randomized linear algebra \citep{mahoney2011randomized}.
We add the proof for completeness.
\begin{proposition}[Stability of OLS under subset removal]\label{prop:stable-ols}
Let $X \in \Rnp$ be a deterministic design matrix satisfying the orthogonality condition $X^\top X = I_p$ and the $\mu$-incoherence condition:
\begin{equation}\label{eqn:mu-incoherence-requirement}
    \|x^{(i)}\|_2 \le \mu \sqrt{\tfrac{p}{n}}, \quad \forall i \in \{1, \dots, n\},
\end{equation}
where $x^{(i)}$ is the $i$-th row of $X$.
Let $y \in \mathbb{R}^n$ be a response vector.
Let $\beta^*$ be the OLS estimator on the full dataset. 
Consider any subset $S \subseteq \{1, \dots, n\}$, 
define $s := |S|$,
and let $X_{-S}$ (resp.\ $y_{-S}$) denote the submatrix (resp.\ subvector) of $X$ (resp.\ $y$) consisting of the rows  (resp.\ entries) indexed by $[n] \setminus S$. Let $\beta_{-S}$ be the OLS estimator on the dataset without the entries in $S$, that is, $ \beta_{-S} \in \argmin_{b\in\R^p} \|y_{-S} - X_{-S} b \|^2_2$.
For
all
$\epsilon \in (0,1)$, if the subset size $s$ satisfies
\begin{equation}\label{eqn:sample-size-requirement}
    s \le \frac{n \epsilon}{p \mu^2 (2\epsilon+1)}
\end{equation}
then the relative increase in squared error is bounded by $\epsilon$, that is,
$$
\|y-X \beta_{-S}\|^2_2 \le (1+\epsilon) \|y-X\beta^*\|^2_2.
$$
\end{proposition}
\begin{proof}
    First, we show that $\beta_{-S}$ is unique 
    (step 1),
    then we bound the difference 
    between the
    reconstruction error of 
    $\beta^*$ 
    and
    $\beta_{-S}$ 
    (step 2).

Step 1:
    We 
    show that $X_{-S}^\top X_{-S}$ is invertible, which implies that $\beta_{-S}$ is uniquely defined.
    Let $X_{S}$ (resp.\ $y_{S}$) denote the sub-matrix (resp.\ sub-vector) of $X$ (resp.\ $y$) consisting of the rows  (resp.\ entries) indexed by $S$. 
    The spectral norm (denoted 
    by
    $\|\cdot\|_{op}$) of 
    $X_S$ is bounded 
    as follows
    \begin{equation}\label{eqn:Xs_op_bound}
        \|X_S\|_{op}^2 \le \text{Tr}(X_S^\top X_S) 
        = \sum_{i \in S} \|x^{(i)}\|_2^2 
        \le s \cdot \frac{\mu^2 p}{n} 
        < 1,
    \end{equation}
    where the second inequality follows from the incoherence condition~\eqref{eqn:mu-incoherence-requirement} and the third inequality follows from the subset size assumption~\eqref{eqn:sample-size-requirement}.
    Hence, the maximum eigenvalue $\lambda_{\max}(X_S^\top X_S)<1$ and thus $I_p - X_S^\top X_S = X_{-S}^\top X_{-S}$ is invertible and $\beta_{-S}$ is unique and given by
    \begin{equation}\label{eqn:beta-s}
        \beta_{-S} = (X_{-S}^\top X_{-S})^{-1} X_{-S}^\top y_{-S} = (I - X_S^\top X_S)^{-1} (X^\top y - X_S^\top y_S).
    \end{equation}

Step 2:
    We 
    now
    bound the reconstruction error using $\beta_{-S}$ as 
    \begin{align*}
        \|y-X\beta_{-S}\|^2_2 &= \|y-X\beta^* + X\beta^* -X\beta_{-S}\|^2_2 \\
        &= \|y-X\beta^*\|^2_2 + \|X\beta^* -X\beta_{-S}\|^2_2 \\
        &= \|y-X\beta^*\|^2_2 + \|\beta^* -\beta_{-S}\|^2_2,
    \end{align*}
    where the first equality follows as 
    $X^\top(y-X\beta^*) = 0$
    and the last equality follows as $X$ is orthogonal.
    It thus suffices to bound $\|\beta^* -\beta_{-S}\|^2_2$ which we do using~\eqref{eqn:beta-s} as follows,
    \begin{align*}
        \|\beta^* -\beta_{-S}\|^2_2 
        &= \|\beta^* - (I - X_S^\top X_S)^{-1} (X^\top y - X_S^\top y_S)\|_2^2 \\
        &= \|(I - X_S^\top X_S)^{-1} (I - X_S^\top X_S)\beta^* - (I - X_S^\top X_S)^{-1} (\beta^* - X_S^\top y_S)\|_2^2 \\
        &= \|(I - X_S^\top X_S)^{-1} X_S^\top (y_S - X_S \beta^*)\|_2^2 \\
        &\le \|(I - X_S^\top X_S)^{-1}\|_{op}^2 \|X_S^\top\|_{op}^2 \|y_S - X_S \beta^*\|_2^2\\
        &\le \|(I - X_S^\top X_S)^{-1}\|_{op}^2 \|X_S^\top\|_{op}^2 \|y - X \beta^*\|_2^2.
    \end{align*}
    In addition,
    $\|X_S^\top\|_{op}^2 =  \|X_S\|_{op}^2 \le \tfrac{s\mu^2 p}{n} 
    =:\delta
    < 1$ by~\eqref{eqn:Xs_op_bound}.
    Using the Neumann series bound (for all square matrices $A$ such that $\|A\|_{op}<1$,
    $\|(I-A)^{-1}\|_{op} \le 1/(1-\|A\|_{op})$) and~\eqref{eqn:Xs_op_bound}, we have that
    \begin{align*}
        \|(I - X_S^\top X_S)^{-1}\|_{op} \le \frac{1}{1-\|X_S^\top X_S\|_{op}} = \frac{1}{1-\|X_S\|_{op}^2}.
    \end{align*}
   It follows that
    \begin{equation*}
    \frac{
     \|y-X\beta_{-S}\|^2_2
    }{
    \|y - X \beta^*\|_2^2
    }    \leq
    1+ 
    \|(I - X_S^\top X_S)^{-1}\|_{op}^2 \|X_S^\top\|_{op}^2 \le  
    1+
    \frac{\delta}{(1-\delta)^2}.
    \end{equation*}
We further have that 
\begin{equation*}
 \delta \le \frac{\epsilon}{2\epsilon+1}
 \le 
 \frac{2\epsilon+1-\sqrt{4\epsilon+1}}{2\epsilon},
\end{equation*}
where the first inequality follows from~\eqref{eqn:sample-size-requirement}
and the second from $\epsilon > 0$. This implies
$\epsilon\delta^2 - (2\epsilon+1)\delta + \epsilon \ge 0$
(by analyzing the roots of the left-hand side)
and thus
$\frac{\delta}{(1-\delta)^2} \le \epsilon$.
This finishes the proof of 
Proposition~\ref{prop:stable-ols}.
\end{proof}

\subsection{Proof of Theorem~\ref{thm:comparing-solutions}}\label{app:pf:comparing-solutions}
\paragraph{\ref{thm:comparing-solutions-normalization} Effect of normalization.}
We prove the statement by counterexample.
Consider two domains $\E := \{1, 2\}$ with
population covariance matrices
\begin{equation*}
   \Sigma_1 := \diag(0.1, 0,9),
   \qquad 
   \Sigma_2 := \diag(9, 1).           
\end{equation*}
$\Vnormminpca$ solves rank-$1$ \normminPCA if
\begin{equation*}
\begin{aligned}
     \Vnormminpca &\in 
    \argmax_{\substack{v=(a, b) \in\R^2; \\ \|v\| = 1}}   \left\{
    \min \left(
        0.1 a^2 + 0.9 b^2 , 
        0.9 a^2 + 0.1 b^2
    \right)\right\}.
\end{aligned}
\end{equation*}
The solution equalizes the two terms, hence $\Vnormminpca = \frac{1}{\sqrt{2}} (1, 1)^\top$ solves \normminPCA. 
$\Vminpca$ solves rank-1 \minPCA if
\begin{equation*}
\begin{aligned}
     \Vminpca &\in 
    \argmax_{\substack{v=(a, b) \in\R^2; \\ \|v\| = 1}}\left\{ 
    \min  \left( 
        0.1 a^2 + 0.9 b^2 , 
        9 a^2 + 1 b^2
    \right)\right\} 
    = \argmax_{\substack{v=(a, b) \in\R^2; \\ \|v\| = 1}}  
        \left(  0.1 a^2 + 0.9 b^2 \right).
\end{aligned}
\end{equation*}
Hence $\Vminpca = (0, 1)^\top$ solves \minPCA.
Then,
\begin{align*}
    \min_{e\in\{1,2\}} \mathcal{L}_\mathrm{normVar}(\Vnormminpca;\Sigma_e) = 0.5 
    > \min_{e\in\{1,2\}} \mathcal{L}_\mathrm{normVar}(\Vminpca;\Sigma_e) = 0.1 \\
    \mathrm{and}\ 
    \min_{e\in\{1,2\}} \mathcal{L}_\mathrm{var}(\Vminpca;\Sigma_e) = 0.9 
    > \min_{e\in\{1,2\}} \mathcal{L}_\mathrm{var}(\Vnormminpca;\Sigma_e) = 0.5.
\end{align*}
Hence, $\Vminpca$ does not solve \normminPCA and $\Vnormminpca$ does not solve \minPCA.
\paragraph{\ref{thm:comparing-solutions-var-vs-rcs} Explained variance vs.\ reconstruction error.}
The first claim (that, in general, \minPCA and \maxRCS have different solution sets) follows from Example~\ref{example1}.
We now show that the solutions of \normminPCA are the same as those of \normmaxRCS.
Let us begin by simplifying the objective for \normmaxRCS:
\begin{align*}
    \argmin_{V\in\Ok} 
        \left\{ \max_{e\in \E} \
        \frac{\Ex \left[ \|\mathbf{x}_e - \mathbf{x}_e VV^\top\|^2_2 \right]}{\Ex \left[ \|\mathbf{x}_e \|^2_2\right]} \right\} 
    &= \argmin_{V\in\Ok} 
        \left\{ \max_{e\in \E} 
        \frac{\Tr(\Sigma_e) - \Tr(V^\top \Sigma_e V)}{\Tr(\Sigma_e)} \right\} \\
    &= \argmin_{V\in\Ok} 
        \left\{ 1- \min_{e\in \E} 
        \left(\frac{\Tr(V^\top \Sigma_e V)}{\Tr(\Sigma_e)} \right) \right\} \\
    &= \argmax_{V\in\Ok}  
    \left\{ \min_{e\in \E} \frac{\Tr(V^\top \Sigma_e V)}{\Tr(\Sigma_e)} \right\}. 
\end{align*}
The final line is the objective for \normminPCA. Hence, the set of solutions to \normminPCA and \normmaxRCS are identical.
\paragraph{\ref{thm:comparing-solutions-regret} Regret.}
Finally, we verify that the solutions to \maxregret are the same whether it is defined via reconstruction error or explained variance.
For $V \in\Ok$, for all domains $e\in\E$, recall that the explained variance and reconstruction error are, respectively,
\begin{align*}
    \mathcal{L}_\mathrm{var}(V;P_e) &= \Tr(V^\top \Sigma_e V), \\
    \mathcal{L}_\mathrm{RCS}(V;P_e)
    &= \Ex\!\left[\|\mathbf{x}_e - \mathbf{x}_e VV^\top\|_2^2\right]
    = \Tr(\Sigma_e) - \Tr(V^\top \Sigma_e V)
    = \Tr(\Sigma_e) - \mathcal{L}_\mathrm{var}(V;P_e).
\end{align*}
Thus, for any $e \in \E$, the regret is
\begin{align*}
    \mathcal{L}_\mathrm{reg}(V;P_e)
    &= \mathcal{L}_\mathrm{RCS}(V;P_e)
       - \min_{W \in \Ok} \mathcal{L}_\mathrm{RCS}(W;P_e) \\
    &= \big(\Tr(\Sigma_e) - \mathcal{L}_\mathrm{var}(V;P_e)\big)
       - \big(\Tr(\Sigma_e) - \max_{W \in \Ok} \mathcal{L}_\mathrm{var}(W;P_e)\big) \\
    &= \max_{W \in \Ok} \mathcal{L}_\mathrm{var}(W;P_e)
       - \mathcal{L}_\mathrm{var}(V;P_e).
\end{align*}
Hence, the regret computed by minimizing reconstruction error is identical to that obtained by maximizing explained variance, and the two formulations therefore yield the same set of optimal solutions. 
The proof for the normalized case follows from $\mathcal{L}_\mathrm{normRCS}(V;P_e) = 1 - \mathcal{L}_\mathrm{normVar}(V;P_e)$.
\hfill\BlackBox

\subsection{Proof of Theorem~\ref{thm:maxrcs-convex-hull}}\label{app:pf:maxrcs-convex-hull}
    We write the proof for \maxRCS. The guarantees for the variants \normmaxRCS, \minPCA, and \normminPCA hold with 
    minor modifications to the proof and thus are omitted. The modification for \maxregret for Statement~\ref{bullet:wc-optimality-equality-sup-max} is given at the end of the proof, from which the subsequent statements and guarantees for \normmaxregret follow.
    
    Let $V^*_k$ solve \maxRCS. Let $P \in \mathcal{P}$ be any distribution with zero mean and covariance $\Sigma\in \mathcal{C} = \mathrm{conv}(\{\Sigma_e\}_{e \in \E})$. Then, there exist weights $(\mu_e)_{e \in \E}$ such that for all $e\in\E$, $\mu_e \ge 0$, $\sum_{e \in \E} \mu_e = 1$, and
    \begin{equation}\label{eqn:convex-hull-covariance}
        \Sigma = \sum_{e \in \E} \mu_e \Sigma_e.
    \end{equation}
\paragraph{Proof of \ref{bullet:wc-optimality-equality-sup-max}.} 
    For all $P \in \mathcal{P}$ and all $V\in\Ok$, the expected reconstruction error is
    \begin{align*}
        \mathcal{L}_\mathrm{RCS}(V;P) 
        &= \Tr(\Sigma) - \Tr(V^\top \Sigma V) \\
        &= \sum_{e \in \E} \mu_e \left( \Tr(\Sigma_e) - \Tr(V^\top \Sigma_e V) \right) \qquad \textrm{(by~\eqref{eqn:convex-hull-covariance})}\\
        &= \sum_{e \in \E} \mu_e \mathcal{L}_\mathrm{RCS}(V;P_e) \\ 
        &\le \max_{e\in\E}\mathcal{L}_\mathrm{RCS}(V;P_e).
    \end{align*}
    Hence $\sup_{P\in\mathcal{P}} \mathcal{L}_\mathrm{RCS}(V;P) \le \max_{e\in\E}\mathcal{L}_\mathrm{RCS}(V;P_e)$. 
    Further, because $\{P_e\}_{e \in \E} \subseteq \mathcal{P}$, it trivially holds that
    $\max_{e\in\E}\mathcal{L}_\mathrm{RCS}(V;P_e) \le \sup_{P \in \mathcal{P}}\mathcal{L}_\mathrm{RCS}(V;P)$. 
    Thus, equality holds. 
\paragraph{Proof of \ref{bullet:wc-optimality-ordering}.}
    Statement~\ref{bullet:wc-optimality-equality-sup-max} implies that for all $V_k, W_k \in\Ok$, if 
    $\max_{e\in\E}\mathcal{L}_\mathrm{RCS}(V_k;P_e) 
    < \max_{e\in\E}\mathcal{L}_\mathrm{RCS}(W_k;P_e)$, 
    then 
    $\sup_{P \in \mathcal{P}}\mathcal{L}_\mathrm{RCS}(V_k;P_e) 
    < \sup_{P \in \mathcal{P}}\mathcal{L}_\mathrm{RCS}(W_k;P_e)$.
\paragraph{Proof of \ref{bullet:wc-optimality-optimality}.}  
    Since $\{P_e\}_{e \in \E} \subseteq \mathcal{P}$, we have for all $V \in \Ok$, 
    \begin{align*}
        \sup_{P \in \mathcal{P}} \mathcal{L}_\mathrm{RCS}(V;P)
        \ge \max_{e \in \E} \mathcal{L}_\mathrm{RCS}(V;P_e) 
        \ge \max_{e \in \E} \mathcal{L}_\mathrm{RCS}(V^*_k;P_e) 
        = \sup_{P \in \mathcal{P}} \mathcal{L}_\mathrm{RCS}(V^*_k;P),
    \end{align*}
    where the second inequality follows by the optimality of $V^*_k$, and the equality follows from Statement~\ref{bullet:wc-optimality-equality-sup-max}.
\paragraph{Proof of \ref{bullet:wc-optimality-poolpca}.}  
    A counterexample showing that pooled and separate PCA solutions can fail to satisfy Statement~\ref{bullet:wc-optimality-optimality} is provided in Example~\ref{example1}. 
\paragraph{Modification for \maxregret.} 
    We show Statement~\ref{bullet:wc-optimality-equality-sup-max} for the regret.
    For all $P\in\mathcal{P}$ and all $V\in\Ok$, 
    the expected regret is
    \begin{align*}
        \mathcal{L}_\mathrm{reg}(V;P) 
        &= \Tr(\Sigma) - \Tr(V^\top \Sigma V) - \min_{W\in\Ok} \big\{\Tr(\Sigma) - \Tr(W^\top \Sigma W)\big\} \\
        &= \sum_{e \in \E} \mu_e \left( \Tr(\Sigma_e) - \Tr(V^\top \Sigma_e V) \right) - \min_{W\in\Ok} \big\{\sum_{e \in \E} \mu_e \left( \Tr(\Sigma_e) - \Tr(W^\top \Sigma_e W) \right)\big\} \\
        &\le \sum_{e \in \E} \mu_e \left( \Tr(\Sigma_e) - \Tr(V^\top \Sigma_e V) \right) - \sum_{e \in \E}\mu_e \min_{W\in\Ok} \big\{  \left( \Tr(\Sigma_e) - \Tr(W^\top \Sigma_e W) \right)\big\} \\
        &= \sum_{e \in \E} \mu_e \mathcal{L}_\mathrm{reg}(V;P_e) 
        \le  \max_{e\in\E}\mathcal{L}_\mathrm{reg}(V;P_e).
    \end{align*}
\hfill\BlackBox

\subsection{Proof of Proposition~\ref{prop:consistency}}\label{app:pf:consistency}
    Let $V_k^*$ and $\hat V_k$ denote the population and empirical solutions, respectively, to any of the considered problems (\minPCA, \normminPCA, \maxRCS, \normmaxRCS, \maxregret, or \normmaxregret), and let $\mathcal{L}$ denote the corresponding loss.\footnote{So $\mathcal{L} \in \{-\mathcal{L}_{\mathrm{var}},\ -\mathcal{L}_{\mathrm{normVar}},\ \mathcal{L}_{\mathrm{RCS}},\ \mathcal{L}_{\mathrm{normRCS}},\ \mathcal{L}_{\mathrm{reg}},\ \mathcal{L}_{\mathrm{normReg}}\}$, as in Theorems~\ref{thm:maxrcs-convex-hull} and~\ref{thm:normmaxrcs}.}
    For all $V\in\Ok$ and $e\in\E$, let the population and empirical domain-specific objectives be
    \begin{align*}
        M_e^{\mathcal{L}}(V) := -\mathcal L (V; \Sigma_e), 
        \qquad
        \hat M_e^{\mathcal{L}}(V) := -\mathcal L (V; \hat\Sigma_e).
    \end{align*}
    Similarly, let the population and empirical worst-case objectives be
    \begin{align*}
        M^{\mathcal{L}}(V) := \min_{e\in\E} -\mathcal L (V; \Sigma_e), 
        \qquad
        \hat M^{\mathcal{L}}(V) := \min_{e\in\E} -\mathcal L (V; \hat\Sigma_e).
    \end{align*}
    All six population and empirical estimators can thus be written as $V^*_k\in\argmax_{V\in\Ok}M^{\mathcal{L}}(V)$ and $\hat V_k\in\argmax_{V\in\Ok} \hat M^{\mathcal{L}}(V)$.
    As $\Ok$ is compact 
    and $M$ and $\hat M$ are continuous with respect to the metric induced by the Frobenius norm, the maxima are attained.
    We now show that 
    \begin{enumerate}[label=\roman*)]
    \item for all $e\in\E$, $\sup_{V \in \Ok} |\hat M_e^{\mathcal{L}}(V) - M_e^{\mathcal{L}}(V)| \conv 0$ as $n_e\to\infty$, 
    \item $\hat M^{\mathcal{L}}$ uniformly converges to $M^{\mathcal{L}}$,
    \item the maximum of $M^{\mathcal{L}}$ is well-separated.
    \end{enumerate}
    Then, by Theorem 5.7.\ in \cite{Vaart_1998}, the estimators are consistent, that is,
    $d(\hat{V}_k, V^*_k) \conv 0$ 
    as $n_\mathrm{min}\to\infty$.
    \paragraph{Proof of i): Domain-specific uniform convergence.}
        We first establish 
        some general
        bounds.
        By the law of large numbers, for all $e\in\E$,
        \begin{equation}\label{eqn:sigma-conv}
            \|\hat\Sigma_e - \Sigma_e\|_F \conv 0 \quad\text{as }n_e\to\infty.
        \end{equation}
        For all $V\in\Ok$, by definition, $\|VV^\top\|_F = \sqrt{\Tr(V^\top V)} = \sqrt{\Tr(I_k)} = \sqrt k$ and $\|I_p\|_F = \sqrt p$. Hence, by Cauchy-Schwarz, for all $V\in\Ok$,
        \begin{align}\label{eqn:cs1}
            &|\Tr(V^\top[\hat\Sigma_e - \Sigma_e]V)| \le \sqrt{k}\|\hat\Sigma_e - \Sigma_e\|_F, 
        \end{align}
        \begin{align}\label{eqn:cs2}
            &|\Tr(\hat\Sigma_e - \Sigma_e)| \le \|I_p\|_F\|\hat\Sigma_e - \Sigma_e\|_F = \sqrt{p}\|\hat\Sigma_e - \Sigma_e\|_F.
        \end{align}
        We now apply these bounds to each specific objective function.

    \textbf{Case 1: \minPCA.}
            By~\eqref{eqn:cs1}, we have that for all $e\in\E$ and for all $V\in\Ok$,
            \begin{equation*}
                |\hat M_e^{\mathcal{L}}(V) - M_e^{\mathcal{L}}(V)| 
                = |\Tr(V^\top[\hat\Sigma_e - \Sigma_e]V)| 
                \le 
                \sqrt k\|\hat\Sigma_e - \Sigma_e\|_F.
            \end{equation*}
            Hence, for all $e\in\E$, $\sup_{V \in \Ok} |\hat M_e^{\mathcal{L}}(V) - M_e^{\mathcal{L}}(V)| \le \sqrt k\|\hat\Sigma_e - \Sigma_e\|_F \conv 0$ as $n_e \to \infty$ by~\eqref{eqn:sigma-conv}.
            
    \textbf{Case 2: \maxRCS.} 
            By~\eqref{eqn:cs1} and ~\eqref{eqn:cs2}, we have that for all $e\in\E$ and for all $V\in\Ok$,
            \begin{equation*}
                |\hat M_e^{\mathcal{L}}(V) - M_e^{\mathcal{L}}(V)| 
                \le |\Tr(V^\top[\hat\Sigma_e - \Sigma_e]V)| + |\Tr(\hat\Sigma_e - \Sigma_e)| \le (\sqrt k + \sqrt p) \|\hat\Sigma_e - \Sigma_e\|_F.
            \end{equation*}
            Hence, for all $e\in\E$, $\sup_{V \in \Ok} |\hat M_e^{\mathcal{L}}(V) - M_e^{\mathcal{L}}(V)| \le (\sqrt k +\sqrt p)\|\hat\Sigma_e - \Sigma_e\|_F \conv 0$ as $n_e \to \infty$.
            
    \textbf{Case 3: \normminPCA.} 
        Let $\tau := \min_{e\in\E} \Tr(\Sigma_e) >0$. 
        Then, gathering~\eqref{eqn:cs1},~\eqref{eqn:cs2}, and 
        using
        that $\Tr(V^\top \hat\Sigma_e V) \le \Tr(\hat\Sigma_e)$ 
        (indeed, as $I-VV^\top$ and $\hat\Sigma_e$ are positive semi-definite, we have 
        $\Tr(\hat\Sigma_e - \hat\Sigma_e V V^\top) \geq 0$), we have that for all $e\in\E$ and for all $V\in\Ok$,
        \begin{align}\label{eqn:upperbound-norm-mincpca}
        |\hat M_e^{\mathcal{L}}(V) - M_e^{\mathcal{L}}(V)| 
        &=
            \left| \frac{\Tr(V^\top \hat\Sigma_e V)}{\Tr(\hat\Sigma_e)}  - \frac{\Tr(V^\top \Sigma_e V)}{\Tr(\Sigma_e)}  \right|  \notag \\
            &\le 
                \frac{|\Tr(V^\top[\hat\Sigma_e - \Sigma_e]V)|}{\Tr(\Sigma_e)} 
                + \Tr(V^\top \hat\Sigma_e V) \left|\frac{1}{\Tr(\hat\Sigma_e)} - \frac{1}{\Tr(\Sigma_e)}\right| \notag\\
            &\le 
                \frac{|\Tr(V^\top[\hat\Sigma_e - \Sigma_e]V)|}{\Tr(\Sigma_e)} 
                + \Tr(V^\top \hat\Sigma_e V) \frac{|\Tr(\hat\Sigma_e-\Sigma_e)|}{\Tr(\hat\Sigma_e)\Tr(\Sigma_e)} \notag\\
            &\le 
                \frac{|\Tr(V^\top[\hat\Sigma_e - \Sigma_e]V)|}{\Tr(\Sigma_e)} 
                + \Tr(\hat\Sigma_e) \frac{|\Tr(\hat\Sigma_e-\Sigma_e)|}{\Tr(\hat\Sigma_e)\Tr(\Sigma_e)} \notag\\
            &\le 
                 \tfrac{\sqrt k + \sqrt p}{\tau}  \|\hat\Sigma_e -\Sigma_e\|_F .   
        \end{align}
        Thus, 
        $
            \sup_{V\in\Ok} |\hat M_e^{\mathcal{L}}(V) - M_e^{\mathcal{L}}(V)|  
            \le \tfrac{\sqrt k + \sqrt p}{\tau}  \|\hat\Sigma_e -\Sigma_e\|_F 
            \conv 0
        $
        as
        $n_e \to\infty$.

    \textbf{Case 4: \normmaxRCS.} 
        By Theorem~\ref{thm:comparing-solutions}\ref{thm:comparing-solutions-normalization}, the solutions of population \normminPCA and population \normmaxRCS coincide; the same algebra shows the solutions of the empirical problems also coincide.
        Thus, the consistency of \normminPCA implies the consistency of \normmaxRCS.
        
    \textbf{Case 5: \maxregret.} 
        For all $e\in\E$ and for all $V\in\Ok$, consider
        \begin{align*}
            |\hat M_e^{\mathcal{L}}(V) - M_e^{\mathcal{L}}(V)| 
            &= \left| \Tr(V^\top \hat\Sigma_e V) - \max_{W\in\Ok}\Tr(W^\top \hat\Sigma_e W)
                - \left\{ \Tr(V^\top \Sigma_e V) - \max_{W\in\Ok}\Tr(W^\top \Sigma_e W) \right\}\right| \\
            &\le 
            \underbrace{\big| \Tr(V^\top \hat\Sigma_e V) -  \Tr(V^\top \Sigma_e V) \big| }_{=:A}
                + \underbrace{\Big|\sum_{i=1}^k \lambda_i(\hat\Sigma_e) - \sum_{i=1}^k \lambda_i(\Sigma_e) \Big|,}_{=:B} \\
        \end{align*}
        where $\lambda_i(\cdot)$ gives the $i$-th largest eigenvalue.
        By~\eqref{eqn:cs1}, $A \le \sqrt k \|\hat\Sigma_e-\Sigma_e\|_F$ and by Weyl's inequality for symmetric matrices \citep[e.g., Theorem 4.3.1 in][]{Horn_Johnson_1985}
        \begin{equation}\label{eqn:weyl}
            B   \le k\|\hat\Sigma_e - \Sigma_e\|_\mathrm{op} 
            \le k\|\hat\Sigma_e - \Sigma_e\|_F. 
        \end{equation}
        Hence,
        $
            \sup_{V \in \Ok} |\hat M_e^{\mathcal{L}}(V) - M_e^{\mathcal{L}}(V)| \le (\sqrt k + k)\|\hat\Sigma_e - \Sigma_e\|_F \conv 0
        $
        as $n_\mathrm{min}\to\infty$.
    
    \textbf{Case 6: \normmaxregret.}
        By the triangle inequality, for all $e\in\E$ and for all $V\in\Ok$, $|\hat M_e^{\mathcal{L}}(V) - M_e^{\mathcal{L}}(V)|$ is bounded by
        \begin{align*}
            \left| \frac{\Tr(V^\top \hat\Sigma_e V)}{\Tr(\hat\Sigma_e)} - \frac{\Tr(V^\top \Sigma_e V)}{\Tr(\Sigma_e)} \right|
                + \left| \max_{W\in\Ok}\frac{\Tr(W^\top \hat\Sigma_e W)}{\Tr(\hat\Sigma_e)} 
                    - \max_{W\in\Ok}\frac{\Tr(W^\top \Sigma_e W)}{\Tr(\Sigma_e)} \right|.
        \end{align*}
        The first term is bounded by $\frac{\sqrt k + \sqrt p}{\tau}\|\hat\Sigma_e - \Sigma_e\|_F$ exactly as in~\eqref{eqn:upperbound-norm-mincpca}. Following 
        a similar
        algebraic sequence 
        as in~\eqref{eqn:upperbound-norm-mincpca}, the second term is bounded by
        \begin{align}\label{eqn:eig-value-diff-bound}
            &\tfrac{1}{\Tr(\Sigma_e)}\big|\max_{W\in\Ok} \Tr(W^\top \hat\Sigma_e W) - \max_{W\in\Ok} \Tr(W^\top \Sigma_e W)\big| + \max_{W\in\Ok} \Tr(W^\top \hat\Sigma_e W) \left| \frac{1}{\Tr(\hat\Sigma_e)} - \frac{1}{\Tr(\Sigma_e)}\right|\notag\\
            &\le \tfrac{k}{\tau} \|\hat\Sigma_e - \Sigma_e\|_\mathrm{op} + \Tr(\hat\Sigma_e) \left|\frac{\Tr(\hat\Sigma_e - \Sigma_e)}{\Tr(\hat\Sigma_e)\Tr(\Sigma_e)}\right|\notag\\
            &\le \tfrac{k +\sqrt p}{\tau} \|\hat\Sigma_e - \Sigma_e\|_F,
        \end{align}
        where the first inequality follows from Weyl's inequality (as in case 5) as $\tau$ is still defined as $\min_{e\in\E} \Tr(\Sigma_e)$.
        Hence, $\sup_{V \in \Ok}  |\hat M_e^{\mathcal{L}}(V) - M_e^{\mathcal{L}}(V)| \le \left(\frac{k + \sqrt k + 2\sqrt p }{\tau}\right) \|\hat\Sigma_e - \Sigma_e\|_F \conv 0$ as $n_e\to\infty$.
    \paragraph{Proof of ii).}
            We show that $\hat M^{\mathcal{L}}$ uniformly converges to $M^{\mathcal{L}}$ by bounding 
            \begin{align*}\label{eqn:sup-rearranged}
                \sup_{V \in \Ok} |\hat M^{\mathcal{L}}(V) - M^{\mathcal{L}}(V)|
                &= \sup_{V \in \Ok} |\min_{e\in\E} \hat M_e^{\mathcal{L}}(V) - \min_{e\in\E} M_e^{\mathcal{L}}(V)| \notag\\
                &\le \sup_{V \in \Ok} \max_{e\in\E} |\hat M_e^{\mathcal{L}}(V) - M_e^{\mathcal{L}}(V)|\\
                &= \max_{e\in\E} \sup_{V \in \Ok} |\hat M_e^{\mathcal{L}}(V) - M_e^{\mathcal{L}}(V)|. 
            \end{align*}
            As $\E$ is finite, Statement~i) implies that $\sup_{V \in \Ok} |\hat M^{\mathcal{L}}(V) - M^{\mathcal{L}}(V)| \conv 0$ as $n_\mathrm{min}\to\infty$.
    \paragraph{Proof of iii).}
            Consider Assumption \ref{ass:uniqueness} and assume by contradiction that the maximum is not well-separated, i.e., there exists $\epsilon > 0$ with 
            $
                \sup_{V\in\Ok; \ d(V, V^*_k) > \epsilon} M^{\mathcal{L}}(V) = M^{\mathcal{L}}(V^*_k).
            $
            This implies that there exists a sequence $(V_n)_{n\in\mathbb{N}}$ 
            in $\Ok$ with $\lim_{n\to\infty}M^{\mathcal{L}}(V_n) = M^{\mathcal{L}}(V^*_k)$ and
            $
                \forall n\in\mathbb{N}, \ d(V_n, V^*_k) > \epsilon.
            $
            Since $\Ok$ is compact 
            under the standard metric induced by the Frobenius norm,
            the Bolzano–Weierstrass theorem 
            guarantees the existence of
             a convergent 
             (w.r.t.\ the standard metric)
             subsequence 
             of $(V_n)_{n \in \mathbb{N}}$
            with limit $\tilde{V}\in\Ok$. 
            Because $d$ is continuous with respect to this standard metric, taking the limit yields 
            \begin{equation}\label{eqn:separation}
                d(\tilde V, V_k^*) \ge \epsilon.
            \end{equation} 
            Continuity of $M^{\mathcal{L}}$ then implies
            $
            M^{\mathcal{L}}(\tilde V) 
            = M^{\mathcal{L}}(V^*_k).
            $
            By Assumption \ref{ass:uniqueness}, there exists $Q\in \R^{k\times k}$ with $Q^\top Q = QQ^\top = I_k$ such that $\tilde V = V^*_k Q$. Hence, $d(\tilde V, V^*_k) =0$ contradicting~\eqref{eqn:separation}. Thus, the maximum is a well-separated point.
\hfill\BlackBox

\subsection{Proof of Proposition~\ref{prop:consistency-of-guarantees}}\label{app:pf:consistency-of-guarantees}
Fix any of the considered problems (\minPCA, \normminPCA, \maxRCS, \normmaxRCS, \maxregret, or \normmaxregret), let $\mathcal{L}$ denote the corresponding loss, and
    let $V_k^*$ and $\hat V_k$ denote 
    population and empirical solutions, respectively. 
    \paragraph{Proof of~\ref{prop:consistency-of-guarantees-i}.}
    By Theorems~\ref{thm:maxrcs-convex-hull} and~\ref{thm:normmaxrcs},
    \begin{align}\label{eqn:sup=max}
        \sup_{P\in\mathcal{Q}} \mathcal{L}(\hat V_k;P)
        = \max_{e\in\E} \mathcal{L}(\hat V_k;
        \Sigma_e) , 
            \qquad
        \sup_{P\in\mathcal{Q}} \mathcal{L}(V_k^*;P)
        = \max_{e\in\E} \mathcal{L}(V^*_k;\Sigma_e).
    \end{align}
    For all $V,W\in\Ok$ and for all $e\in\E$, the map $V\mapsto \mathcal{L}(V;\Sigma_e)$ is $c_e$-Lipschitz with respect to the projection distance $d(V, W) = \|VV^\top - WW^\top\|_F$, where
    \begin{align*}
        c_e := \begin{cases}
            \|\Sigma_e\|_F 
            &\textrm{for \minPCA, \maxRCS,  \maxregret} \\
            \frac{\|\Sigma_e\|_F}{\Tr(\Sigma_e)} 
            &\textrm{for \normminPCA, \normmaxRCS, \normmaxregret}.
        \end{cases}
    \end{align*}
    To see this, consider \minPCA. For all $e\in\E$ and for all $V,W \in \Ok$
    \begin{align*}
        \big| \mathcal{L}_\mathrm{var}(V;\Sigma_e) - 
        \mathcal{L}_\mathrm{var}(W;\Sigma_e) \big| 
        &=
        \big|\Tr(V^\top \Sigma_e V) - \Tr(W^\top\Sigma_e W)\big| \\
        &= \big|\Tr(\Sigma_e(VV^\top - WW^\top))\big| \\
        & \le \|\Sigma_e\|_F \|VV^\top - WW^\top\|_F = \|\Sigma_e\|_F\ d(V,W),
    \end{align*}
    by Cauchy–Schwarz and cyclicity of the trace.
    For \maxRCS and \maxregret, the objectives differ from $\mathcal{L}_{\mathrm{var}}(V;\Sigma_e)$ only by additive terms independent of $V$, so the same bound applies.
    The normalized variants introduce only the multiplicative factor $1/\Tr(\Sigma_e)$. 
    This concludes the Lipschitz argument.
    The maximum of Lipschitz functions is Lipschitz, 
    in this case
    with constant equal to $\max_{e\in\E} c_e$. Therefore, for all $V,W\in \Ok$,
    \begin{align}\label{eqn:lipschitz-bound}
    \Big|\max_{e\in\E}\mathcal{L}(V;\Sigma_e)
          - \max_{e\in\E}\mathcal{L}(W;\Sigma_e)\Big| 
    \le \big(\max_{e\in\E}c_e\big) \, d(V,W).
    \end{align}
    Combining~\eqref{eqn:sup=max} with the Lipschitz bound~\eqref{eqn:lipschitz-bound} at $V=\hat V_k$, $W=V_k^*$ yields
    \begin{align*}
        \Big|\sup_{P\in\mathcal{Q}} \mathcal{L}(\hat V_k;P) - \sup_{P\in\mathcal{Q}} \mathcal{L}(V_k^*;P)\Big| 
        \le \big(\max_{e\in\E}c_e\big) \, d(\hat V_k, V_k^*).
    \end{align*}
    Since $d(\hat V_k,V_k^*) \conv 0$ as $n_\mathrm{min}\to\infty$ by Proposition~\ref{prop:consistency}, 
    the right-hand side converges in probability to zero, proving Statement~\ref{prop:consistency-of-guarantees-i}.
    \paragraph{Proof of~\ref{prop:consistency-of-guarantees-ii}.}
     It is sufficient to show that 
    $$
     \Big|\sup_{P\in\mathcal{Q}} \mathcal{L}(\hat V_k;P) - \hat m_k \Big| \conv 0 \quad\textrm{as }n_\mathrm{min}\to\infty,
    $$
    where $\hat m_k = \max_{e\in\E} \mathcal{L}(\hat V_k;\hat\Sigma_e)$.
    By~\eqref{eqn:sup=max}, we have that
    \begin{align*}
        \Big|\sup_{P\in\mathcal{Q}} \mathcal{L}(\hat V_k;P) - \hat m_k \Big| 
        &= \Big|
            \max_{e\in\E} \mathcal{L}(\hat V_k;\Sigma_e) 
            - \max_{e\in\E} \mathcal{L}(\hat V_k;\hat \Sigma_e)
            \Big| 
        \le \max_{e\in\E}\Big|
             \mathcal{L}(\hat V_k;\Sigma_e) 
            - \mathcal{L}(\hat V_k;\hat \Sigma_e)
            \Big|.
    \end{align*}
    We bound the final term for each of the objectives, 
    using arguments that are analogous to the proof of i) in the proof of Proposition~\ref{prop:consistency}.
    For \minPCA, by~\eqref{eqn:cs1},
    \begin{align}\label{eqn:var-bound}
        \max_{e\in\E}\Big|
             \mathcal{L}_\mathrm{var}(\hat V_k;\Sigma_e) 
            - \mathcal{L}_\mathrm{var}(\hat V_k;\hat \Sigma_e)
            \Big|
        &= \max_{e\in\E}\Big|
            \Tr(\hat V_k^\top\Sigma_e\hat V_k)
            -\Tr(\hat V_k^\top\hat\Sigma_e\hat V_k)
            \Big| \notag \\
        &\le \sqrt{k} \, \max_{e\in\E} \|\Sigma_e - \hat\Sigma_e\|_F.
    \end{align}
    Similarly, for \maxRCS, by the cyclic property of the trace,
    \begin{align*}
        \max_{e\in\E}\Big|
             \mathcal{L}_\mathrm{RCS}(\hat V_k;\Sigma_e) 
            - \mathcal{L}_\mathrm{RCS}(\hat V_k;\hat \Sigma_e)
            \Big| 
    &= \max_{e\in\E}\Big| 
        \Tr\big((\Sigma_e - \hat\Sigma_e) (I - \hat V_k \hat{V}_k^\top)\big)\Big| \\
    &\le \sqrt{p-k} \, \max_{e\in\E} \|\Sigma_e - \hat\Sigma_e\|_F,
    \end{align*}
        where $\|I - \hat V_k \hat V_k^\top\|_F = \sqrt{p-k}$ as it is an orthogonal projector of rank $p-k$.
    Let $\tau := \min_{e\in\E}\Tr(\Sigma_e)>0$. Then for \normminPCA and \normmaxRCS,
    \begin{align*}
        \max_{e\in\E}\Big|
             \mathcal{L}_\mathrm{normVar}(\hat V_k;\Sigma_e) 
            - \mathcal{L}_\mathrm{normVar}(\hat V_k;\hat \Sigma_e)
            \Big|
        & = \max_{e\in\E}\Big|
             \mathcal{L}_\mathrm{normRCS}(\hat V_k;\Sigma_e) 
            - \mathcal{L}_\mathrm{normRCS}(\hat V_k;\hat \Sigma_e)
            \Big| \\
        &= \max_{e\in\E} \Big| 
            \frac{\Tr(\hat V_k^\top \hat\Sigma_e \hat V_k)}{\Tr(\hat\Sigma_e)}
            - 
            \frac{\Tr(\hat V_k^\top \Sigma_e \hat V_k)}{\Tr(\Sigma_e)} 
        \Big| \\
        &\le \frac{\sqrt k + \sqrt p}{\tau} \max_{e\in\E}  \|\hat\Sigma_e -\Sigma_e\|_F, 
    \end{align*}
    where the last inequality follows from~\eqref{eqn:upperbound-norm-mincpca}.
    For \maxregret,
    \begin{align*}
        &\max_{e\in\E}\Big|
             \mathcal{L}_\mathrm{reg}(\hat V_k;\Sigma_e) 
            - \mathcal{L}_\mathrm{reg}(\hat V_k;\hat \Sigma_e)
            \Big|\\
        &= \max_{e\in\E} \Big| 
                \max_{W\in\Ok}\Tr(W^\top\Sigma_e W) 
                - \Tr(\hat V_k^\top \Sigma_e \hat V_k) 
                - \max_{W\in\Ok}\Tr(W^\top\hat\Sigma_e W) 
                + \Tr(\hat V_k^\top \hat\Sigma_e \hat V_k) 
            \Big|\\
        &\le \max_{e\in\E} \Big| 
                \max_{W\in\Ok}\Tr(W^\top\Sigma_e W)
                - \max_{W\in\Ok}\Tr(W^\top\hat\Sigma_e W) \Big|  
            + \max_{e\in\E} \Big| 
                \Tr(\hat V_k^\top \Sigma_e \hat V_k) 
                - \Tr(\hat V_k^\top \hat\Sigma_e \hat V_k)  \Big| \\
        &= \max_{e\in\E} \Big| 
                \sum_{i=1}^k\lambda_i(\Sigma_e)
                - \sum_{i=1}^k\lambda_i(\hat\Sigma_e) \Big|  
            + \max_{e\in\E} \Big| 
                \Tr(\hat V_k^\top \Sigma_e \hat V_k) 
                - \Tr(\hat V_k^\top \hat\Sigma_e \hat V_k)  \Big| \\
        &\le (k + \sqrt{k}) \, \max_{e\in\E} \|\Sigma_e - \hat\Sigma_e\|_F,
    \end{align*}
    where the last inequality 
    follows as in~\eqref{eqn:weyl}.
    Finally, for \normmaxregret, as in case~6 of the proof of Proposition~\ref{prop:consistency},
    \begin{align*}
        &\max_{e\in\E}\Big|
             \mathcal{L}_\mathrm{normReg}(\hat V_k;\Sigma_e) 
            - \mathcal{L}_\mathrm{normReg}(\hat V_k;\hat \Sigma_e)
            \Big|
          \le \tfrac{k+\sqrt k + 2\sqrt p}{\tau} \, \max_{e\in\E}  \|\hat\Sigma_e - \Sigma_e\|_F. 
    \end{align*}

    Since $\E$ is finite and for all $e\in\E$ we have
    $\|\hat\Sigma_e-\Sigma_e\|_F \conv 0$ as $n_e\to\infty$ (by~\eqref{eqn:sigma-conv}),
    it follows that
    \begin{equation*}
        \max_{e\in\E}\|\hat\Sigma_e-\Sigma_e\|_F \conv 0
    \qquad \text{as } n_{\min}\to\infty,
    \end{equation*}
    which proves Statement~\ref{prop:consistency-of-guarantees-ii}.
\hfill\BlackBox

\subsection{Proof of Theorem~\ref{thm:mc}}
    Let $V^*_k \in \Ok$ be a rank-$k$ solution to \maxRCS and assume $V^*_k$ is $\mu$-incoherent.
    Fix $\mathbf{x}\in\R^{1\times p}$ and a mask $\omega \in\{0,1\}^p$ such that the number 
    $s$
    of unobserved entries 
    satisfies
    $s \le \tfrac{ p \epsilon}{k \mu^2 (2\epsilon + 1)}$ and
    let $S(\omega):=\{i\in[p]:\omega_i=0\}$ denote the set of unobserved coordinates.
    Inductive matrix completion via least squares is equivalent to ordinary least squares of the response $y:=\mathbf{x}^\top$ on the design matrix $X:=V^*_k$, with the rows indexed by $S(\omega)$ removed.
    Hence, Proposition~\ref{prop:stable-ols} (Appendix~\ref{app:stable-ols}) applies with 
    \begin{equation*}
        X := V^*_k \in \R^{p\times k},
        \quad y := \mathbf{x}^\top \in \R^{p},
        \quad S := S(\omega),
    \end{equation*}
    which yields\footnote{This is the transpose-version of Proposition~\ref{prop:stable-ols}; the Euclidean norm is invariant under transposition.}
    \begin{equation*}
        \|\mathbf{x} - \ell(\mathbf{x},\omega,V^*_k)(V^*_k)^\top\|_2^2
        \le
        (1+\epsilon)
        \|\mathbf{x} - \mathbf{x} V^*_k(V^*_k)^\top\|_2^2.
    \end{equation*}
    Since 
    $s \le \tfrac{ p \epsilon}{k \mu^2 (2\epsilon + 1)}$
    holds $Q$-almost surely 
    by Assumption~\ref{ass:sufficient-obs},
    the above inequality holds $Q$-almost surely for $\omega\sim Q$.
    For any $P\in\mathcal P$, we then take expectations over $\mathbf{x}\sim P$ and $\omega\sim Q$, yielding
    \begin{align*}
    \mathcal{L}_{P,Q}(V^*_k)
    &=
    \Ex_{\mathbf{x}\sim P,\omega\sim Q}
    \|\mathbf{x} - \ell(\mathbf{x},\omega,V^*_k)(V^*_k)^\top\|_2^2 \\
    &\le
    (1+\epsilon) 
    \Ex_{\mathbf{x}\sim P,\omega\sim Q}
    \|\mathbf{x} - \mathbf{x} V^*_k (V^*_k)^\top\|_2^2
    =
    (1+\epsilon) 
    \Ex_{\mathbf{x}\sim P}
    \|\mathbf{x} - \mathbf{x}V^*_k (V^*_k)^\top\|_2^2.
    \end{align*}
    Taking the supremum over $P\in\mathcal{P}$ and applying statement~(i) of Theorem~\ref{thm:maxrcs-convex-hull}  gives
    \begin{align}\label{eqn:bound1_clean}
        \sup_{P\in\mathcal P} \mathcal{L}_{P,Q}(V^*_k) 
        &\le (1+\epsilon) \sup_{P\in\mathcal P} \Ex_{\mathbf{x}\sim P}\| \mathbf{x} - \mathbf{x} V^*_k (V^*_k)^\top \|^2_2 \notag\\
        &= (1+\epsilon) \max_{e\in\E} \Ex_{\mathbf{x}\sim P_e}\| \mathbf{x} - \mathbf{x} V^*_k (V^*_k)^\top \|^2_2 \notag\\
        &=
        (1+\epsilon)\min_{V\in\Ok} \max_{e\in\E} 
        \Ex_{\mathbf{x}\sim P_e}\| \mathbf{x} - \mathbf{x} V V^\top \|^2_2, 
    \end{align}
    where the last equality follows from the definition of $V_k^*$ as a rank-$k$ solution to \maxRCS.
    Finally,
    for any fixed $\mathbf{x}$ and $V\in\Ok$, the choice $\ell=\mathbf{x}V$ minimizes $\|\mathbf{x}-\ell V^\top\|_2^2$ over $\ell\in\R^k$. 
    Hence, for all $\mathbf{x}$ and for all $\omega$,
    \begin{equation*}
        \| \mathbf{x} - \mathbf{x} V V^\top \|^2_2
        \le 
        \| \mathbf{x} - \ell(\mathbf{x},\omega,V) V^\top \|^2_2,
    \end{equation*}
    and therefore for each domain $e$,
    \begin{equation*}
        \Ex_{\mathbf{x}\sim P_e}
            \| \mathbf{x} - \mathbf{x} V V^\top \|^2_2
        \le \Ex_{\mathbf{x}\sim P_e,\omega\sim Q}
            \| \mathbf{x} - \ell(\mathbf{x},\omega,V) V^\top \|^2_2
        \le \sup_{P\in\mathcal{P}} \mathcal{L}_{P,Q}(V).
    \end{equation*}
    Taking the maximum over $e\in\E$ and then the minimum over $V\in\Ok$ gives
    \begin{equation*}
    \min_{V\in\Ok} \max_{e\in\E} 
        \Ex_{\mathbf{x}\sim P_e}\| \mathbf{x} - \mathbf{x} V V^\top \|^2_2 
        \le 
        \min_{V\in\Ok} \sup_{P \in \mathcal{P}}
        \mathcal{L}_{P,Q}(V).
    \end{equation*}
    Combining this identity with~\eqref{eqn:bound1_clean} 
    yields
    \begin{equation*}
        \sup_{P\in\mathcal P} \mathcal{L}_{P,Q}(V^*_k) 
        \le (1+\epsilon) 
            \min_{V\in\Ok} 
            \sup_{P \in \mathcal{P}}
            \mathcal{L}_{P,Q}(V).
    \end{equation*}
    as claimed.
\hfill\BlackBox
\end{document}